\documentclass[11pt]{article}

\usepackage{authblk}
\usepackage[linesnumbered, ruled, vlined]{algorithm2e}
\usepackage{adjustbox} 
\usepackage{amsfonts,amsthm,amssymb}
\usepackage{amsmath}  
\usepackage{bm,bbm} 
\usepackage{booktabs}
\usepackage{etoolbox}  
\usepackage{graphicx}
\usepackage{lipsum}
\usepackage[numbers]{natbib}
\usepackage{siunitx}
\usepackage{threeparttable} 

\makeatletter
\newcommand{\algorithmfootnote}[2][\footnotesize]{%
  \let\old@algocf@finish\@algocf@finish
  \def\@algocf@finish{\old@algocf@finish
    \leavevmode\rlap{\begin{minipage}{\linewidth}
    #1#2
    \end{minipage}}%
  }%
}
\makeatother

\providecommand{\keywords}[1]
{
  \small	
  \textbf{\textit{Keywords---}} #1
}

\SetKwComment{command}{right mark}{left mark} 
\appto\TPTnoteSettings{\footnotesize}  
\graphicspath{{figures/}}  

\newtheorem{mytheorem}{Theorem}[section]
\newtheorem{mylemma}[mytheorem]{Lemma}

\newtheorem{mycondition}[mytheorem]{Condition}
\newtheorem{myremark}[mytheorem]{Remark}
\newtheorem{myprop}[mytheorem]{Proposition}
\newtheorem{mycriterion}[mytheorem]{Criterion}

\usepackage{amssymb}

\begin{document}

\title{Accelerate the Warm-up Stage in the Lasso Computation via a Homotopic Approach}

\author[1]{Yujie Zhao}
\author[2]{Xiaoming Huo}
\affil[1]{Biostatistics and Research Decision Sciences Department, Merck \& Co., Inc, PA, USA}
\affil[2]{School of Industrial and Systems Engineering, Georgia Institute of Technology, Atlanta, GA, USA}

\maketitle

\begin{abstract}
In optimization of the least absolute shrinkage and selection operator (Lasso) problem, the fastest algorithm has a convergence rate of $O(1/\sqrt{\epsilon})$.
This polynomial order of $1/\epsilon$ is caused by the undesirable behavior of the absolute function at the origin.
In this paper, we propose an algorithm called \textit{homotopy shrinkage yielding} (HOSKY), which helps expedite the warm-up stage of the existing algorithms.
With the acceleration by HOSKY in the warm-up stage, one can get a provable convergence rate lower than $O(1/\sqrt{\epsilon})$.
The main idea of the proposed HOSKY algorithm is to use a sequence of surrogate functions to approximate the $\ell_1$ penalty that is used in Lasso.
This sequence of surrogate functions, on the one hand, gets closer and closer to the $\ell_1$ penalty; on the other hand, they are strictly convex and well-conditioned, which enables a provable exponential rate of convergence by gradient-based approaches.
As we will prove in this paper, the convergence rate of the HOSKY algorithm is $O([\log(1/\epsilon_w)]^2)$, where $\epsilon_w$ is the precision used in the warm-up stage ($\epsilon_w \nrightarrow 0$).
Our numerical simulations also show that HOSKY empirically performs better in the warm-up stage and accelerates the overall convergence rate.
\end{abstract}


\section{Introduction}
\label{sec: introduction}

In the framework of regression methods, the least absolute shrinkage and selection operator (Lasso) is a tool for both variable selection and model estimation.
It was originally introduced in geophysics \cite{santosa1986linear} and later by Robert Tibshirani \cite{tibshirani1996regression} who coined the term.
Its major objective is to select a reduced set of known covariates for use in a predictive model.
In this paper, we focus on the strategy to assign initial points to the optimization problem in Lasso.
And we propose an algorithm called \textit{homotopy shrinkage yielding} (HOSKY).
The initial points generated by HOSKY, on the one hand, are more computationally efficient; on the other hand, accelerate the overall convergence rate. 

In the rest of this section, we first introduce the problem formulation in Section \ref{sec: problem formulation}.
Then, we list criteria to compare different algorithms in Section \ref{sec: crierion of computational efficiancy}. 
Next, we summarize the existing literature in Section \ref{sec: literature review} and compare their computational efficiency under the criteria in Section \ref{sec: crierion of computational efficiancy}.
Finally, we discuss the motivation and contributions of HOSKY in Section \ref{sec: contribution}.

\subsection{Problem Formulation}
\label{sec: problem formulation}

In linear regression, the available dataset is $\mathcal D = \{y \in \mathbb R^n, X \in \mathbb R^{n\times p}\}$, where $y$ is the response vector and $X$ is the model matrix (of predictors).
Here $n,p > 0$ refers to the number of observations and covariates, respectively.
Given the above dataset $\mathcal D$, the linear regression model is
$$
  y = X\beta^* + w,
$$
where $\beta^* \in \mathbb R^p$ is the ground truth of the regression coefficients desired to be estimated.
And $w \in \mathbb R^n$ is the white-noise residual, i.e., $w_i \overset{i.i.d.}{\sim} N(0,\sigma^2)$ for any $i = 1, \ldots, n$.
Accordingly, the Lasso estimator $\widehat \beta$ is commonly written as
\begin{equation}
	\label{equ: lasso estimator}
	\widehat\beta
	=
	\arg \min_\beta
	\left\{
	  F(\beta) :=
	  \frac{1}{2n}\|y-X\beta\|_2^2 + \lambda \|\beta\|_1
	\right\},
\end{equation}
where parameter $\lambda > 0$ controls the trade-off between the sparsity and the model's goodness of fit.
Here, we exclude $\lambda$ in the notation $F(\beta)$, since we don't consider the selection of $\lambda$ in this paper (which by itself has a large literature).

Under the above Lasso model, many iterative algorithms are proposed to minimize its objective function $F(\beta)$.
Technically, these iterative algorithms are involved in two stages.
The first stage is called the \textit{warm-up stage}.
In this stage, one decides the strategy to assign initial points.
The simplest strategy is to use pre-specified vectors as initial points, say, a vector of all zeros.
An alternative strategy is to use the solution from the ridge regression \cite{hoerl1970ridge} as the initial points \citep{melkumova2017comparing, sun2000lasso}.
In this paper, we propose one more option called HOSKY.
The initial points generated by HOSKY, on the one hand, are more computationally efficient; on the other hand, accelerate the overall convergence rate. 
We skip its detailed description here and articulate its implementation later in Section \ref{sec: our algorithm}.
The second stage, which is right after the warm-up stage, is called the \textit{after-warm-up stage}. 
In this stage, one runs a selected iterative algorithm with the initial point from the warm-up stage until convergence.
The visualization of the correlation between these two stages is available in Fig. \ref{fig: define lasso warmup}.

\begin{figure}[htbp]
    \centering
    \includegraphics[width = 0.7\textwidth]{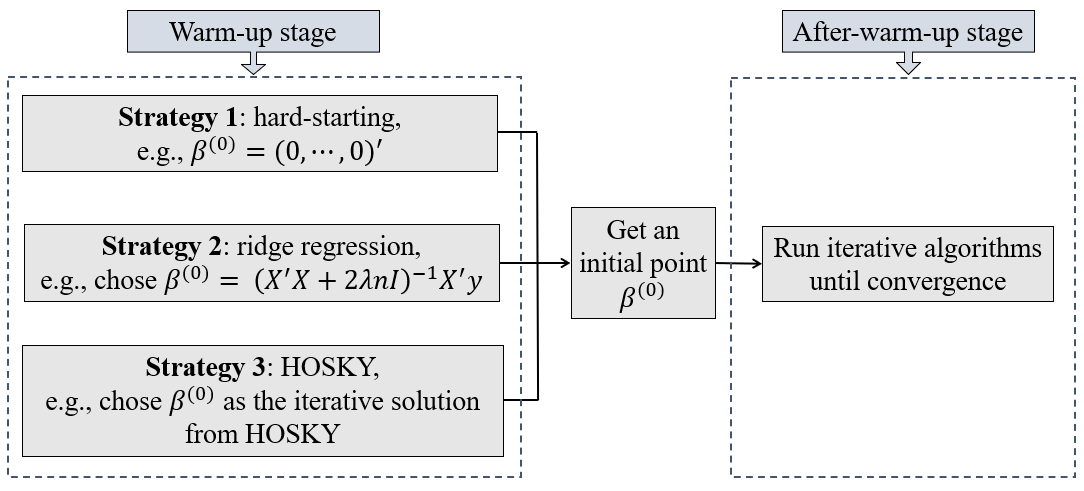}
    \caption{The visualization of the warm-up stage and the after-warm-up stage. Here $\beta^{(0)}$ is the initial point used in the after-warm-up stage.}
    \label{fig: define lasso warmup}
\end{figure}

The objective of this paper is to develop a new strategy in the warm-up stage.
The proposed strategy, on the one hand, is more computationally efficient than the existing strategies (see our comparison  criteria in Section \ref{sec: crierion of computational efficiancy} and literature review in Section \ref{sec: literature review}); one the other hand, helps to accelerate the calculation in the after-warm-up stage (see various numerical simulations in Section \ref{sec: simulation}).

\subsection{Criteria  to Measure Computational Efficiency}
\label{sec: crierion of computational efficiancy}

In this section, we present the criteria to measure the computational efficiency in the warm-up stage, as well as the contribution of the warm-up stage to convergences. 

First, we introduce the criterion to measure the computational efficiency in the warm-up stage. 
It is widely acknowledged that, different initial points might leads to different closeness to the optima $\widehat \beta$. 
If one takes lots of computations to get an initial point, this initial point will tentatively converge to the optima $\widehat \beta$ shortly.
Otherwise, if one takes limited computations to derive an initial point, one might end with slow convergence to the optima $\widehat \beta$ in the after-warm-up stage.
To compare the computational efficiency of different warm-up strategies, we measure their total number of numerical operations to achieve a common warm-up precision $\epsilon_w$.
And this warm-up precision $\epsilon_w$ defines the closeness between $F(\beta^{(k)})$ and $F(\widehat\beta)$.
The mathematical definition of this criterion is articulated in the following statement.

\begin{mycriterion}[Criterion to measure the computational efficiency in the warm-up stage]
\label{def: epsilon precision}
  Suppose there are two strategies A and B in the warm-up stage. 
  Assume both of them are iterative algorithms and give iterative solutions $\beta_A^{(k)}, \beta_B^{(k)}$ after $k$ iterations in the warm-up stage.  
  We declare strategy A is more computationally efficient than B if A's total number of numerical operations (such as plus, minus, multiplications and divisions) to achieve 
  \begin{equation}
	\label{equ: F precision}
	F(\beta_A^{(k)}) - F(\widehat\beta) \leq \epsilon_w,
  \end{equation}
  is less than B's total number of numerical operations to achieve 
  $
	F(\beta_B^{(k)}) - F(\widehat\beta) \leq \epsilon_w.
  $
  Here $\epsilon_w > 0$ is a pre-specified warm-up precision and commonly not set as tiny number close to 0, i.e., $\epsilon_w \nrightarrow 0$.
\end{mycriterion}

The above criterion indicates that the computational complexity is in terms of $\epsilon_w$.
And this correlation is usually adopted in the big $O$ notation.
For example, if the computational complexity of a warm-up algorithm is $O\left(np / \epsilon_w \right)$, then it means that to achieve the $\epsilon_w$ warm-up precision, the number of numeric operations can be upper bounded by a constant multiplies $np / \epsilon_w$.
In theory, an $O\left(np / \sqrt{\epsilon_w}\right)$ algorithm is more computationally efficient than an $O\left(np / \epsilon_w\right)$ algorithm.
Moreover, an $O\left(np \log(1/\epsilon_w)\right)$ algorithm has an even lower order of complexity.
Although the order of computational complexity gives an upper bound of the number of numerical operations to achieve the warm-up precision $\epsilon_w$, it does not say anything about the average performance of the algorithm.
It is possible that an algorithm with larger upper bounds performs better in some cases than an algorithm with lower upper bounds.

The aforementioned Criterion \ref{def: epsilon precision} compares two iterative algorithms in the warm-up stage. 
And it is not recommended to use Criterion \ref{def: epsilon precision} to compare an iterative algorithm with a closed-form strategy (like ridge regression).
This is because, a  closed-form strategy gives a fixed warm-up precision, which is independent of the warm-up precision $\epsilon_w$: no matter how much $\epsilon_w$ changes, the total number of numerical operations to get the closed-form initial points is fixed.
On the contrary, the iterative strategy depends on the warm-up precision $\epsilon_w$: a smaller $\epsilon_w$ leads to a larger total number of numerical operations and vice versa. 

In this paper, we call the total number of numerical operations to achieve the warm-up precision $\epsilon_w$ as \textit{order of computational complexity} in the warm-up stage.
In the remainder of this paper, we use it as our primary measure to compare different iterative algorithms.
Besides, we also use running time to achieve the warm-up precision $\epsilon_w$ as our secondary measure in the numerical simulations in Section \ref{sec: simulation}. 
The reason why we adopt the order of computational complexity as our primary measure is that it records the number of numerical operations (like plus and minus), which is independent of different computer platforms.
While running time, though widely used, depends on different platforms.
Consequently, the order of computational complexity provides a more reliable way for us to compare different algorithms.

Second, we introduce the criterion to measure the contribution of a warm-up stage to the overall convergence. 
Recall in Fig. \ref{fig: define lasso warmup} that, an iterative algorithm to solve Lasso can be cut into two stages: warm-up stage and after-warm-up stage. 
In the warm-up stage, one stops when the warm-up precision $\epsilon_w$ is achieved. 
And usually $\epsilon_w \nrightarrow 0$ (e.g., $\epsilon_w = 0.05$).
In the after-warm-up stage, one stops when the after-warm-up precision $\epsilon_{w+}$ is achieved.  
And this $\epsilon_{w+}$ controls the overall convergence. 
Commonly we have $\epsilon_{w+} \to 0$ (e.g., $\epsilon_w = 10^{-8}$).

To measure the contribution of different warm-up strategies to the overall convergence, one can use the following procedure. 
\begin{itemize}
    \item Step 1: run different warm-up strategies until a common warm-up precision $\epsilon_w$ is achieved. 
    This makes the initial points from different warm-up strategies share the same closeness to the optima $\widehat\beta$.
    And the only difference lies in the number of numerical operations to arrive at $\epsilon_w$.
    \item Step 2: select an algorithm in the after-warm-up stage. The selected algorithm can minimize $F(\beta)$ until convergence.
    Options of the algorithms in the after-warm-up stage are reviewed in Section \ref{sec: literature review}.
    \item Step 3: run the algorithm selected in Step 2 until a common after-warm-up precision $\epsilon_{w+}$ is achieved. 
    \item Step 4: compare the total number of numerical operations (warm-up stage + after-warm-up stage).
\end{itemize}
The above procedure is summarized in the following proposition.

\begin{mycriterion}[Criterion to measure the contribution of the warm-up stage to overall convergence]
\label{criterion: warm-up contribution}
  Suppose there are two strategies A and B in the warm-up stage. 
  Assume both of them are iterative algorithms and they both achieve a common warm-up precision $\epsilon_w$ in \eqref{equ: F precision}.
  With the initial points by strategies A and B available, one can run a selected algorithm in the after-warm-up stage, and minimize $F(\beta)$ until a common after-warm-up precision $\epsilon_{w+}$ is achieved.
  We declare strategy A contributes more to convergence than B if A's total number of numerical operations (warm-up stage + after-warm-up-stage) is smaller than B. 
\end{mycriterion}

As readers will see in the rest of the paper, both Criterion \ref{def: epsilon precision} and Criterion \ref{criterion: warm-up contribution} are used in Section \ref{sec: order of complexity of HS} and Section \ref{sec: simulation} to theoretically/numerically verify that the proposed HOSKY algorithm is more computationally efficient in the warm-up stage and also accelerates the convergence in the after-warm-up stage.

\subsection{Literature Review}
\label{sec: literature review}

In this section, we present representative strategies in the warm-up stage, and also briefly review the representative algorithms in the after-warm-up stage.

In the warm-up stage, people can use different strategies to assign initial points. 
The first strategy is to set the initial points as a pre-specified vector. 
For example, one can set $\beta^{(0)} = (0, 0, \ldots, 0)'$, where $\beta^{(0)}$ denotes the initial point.
Then one can run a selected iterative algorithm to minimize $F(\beta)$ until convergence.
This strategy is adopted in many papers, like \cite{ISTA, FISTA}, given its simplicity. 
The second strategy is to use the solution from ridge regression \cite{hoerl1970ridge} as the initial points.
Specifically, one can set $\beta^{(0)}$ as
\begin{equation}
\label{equ: ridge regression close-form solution}
  \beta^{(0)} 
  = 
  \arg\min_{\beta} 
  \left\{
    \frac{1}{2n} \left\| y - X \beta \right\|_2^2 
    + 
    \lambda \left\|\beta\right\|_2^2
  \right\}
  =
  \left(X'X + 2 \lambda n I \right)^{-1} X'y.
\end{equation}
This strategy is also preferred by many researchers like \cite{melkumova2017comparing, sun2000lasso} due to its closed-form propriety.
To get the above closed-form solution, the major computation lies in the inverse of the matrix $X'X + 2 \lambda n I$.
As indicated by \cite{tveit2003complexity, zhao2022survey, zhao2021new}, the computational complexity to inverse a matrix is at least $O(p^2 \log(p))$.  
In additional to solving \eqref{equ: ridge regression close-form solution} directly, one can also use the coordinate descent algorithm to get an iterative solution \cite{glmnet}. Since we will introduce this algorithm in the next paragraph, we skip its detailed description here. And more details can also be found in Appendix \ref{sec: CD}.

In the after-warm-up stage, with the initial points from the warm-up stage, one can run a selected iterative algorithm until convergence. The representatives of these iterative algorithms are reviewed as follows. 
The first representative algorithm is the iterative shrinkage threshold algorithm (ISTA) proposed by \cite{ISTA}.
It approximates the first term of $F(\beta)$, i.e., $\frac{1}{2n}\|y-X\beta\|_2^2$, by its second-order Taylor expansion.
Then, they use its gradient, Hessian matrix, and soft-thresholding function to iteratively update the solution.
The second representative algorithm is fast iterative shrinkage-thresholding algorithms (FISTA) proposed by \cite{FISTA}, which is an accelerated version of ISTA.
Compared with ISTA, FISTA takes advantage of the accelerated gradient descent (AGD) algorithm and uses the gradients at the previous two solutions to learn from the ``history.''
The third representative algorithm is the coordinate descent (CD) algorithm in \cite{glmnet}.
Different from ISTA and FISTA, which updates their solution globally, CD utilizes the coordinate descent to update the solution.
A R package named \textit{glmnet} has fueled its adoption.
The fourth representative algorithm is the smooth L1 algorithm (SL) in \cite{smoothlassoClass01EM}.
Compared with ISTA. FISTA, CD, which targets directly at the minimization of $F(\beta)$, SL aims to find a surrogate of $F(\beta)$.
The surrogate function is 
$$
  F_{\alpha}(\beta)
  = 
  \frac{1}{2n} \left\| y - X \beta \right\|_2^2 
  + 
  \lambda \sum_{i = 1}^p \phi_{\alpha}(\beta_i),
$$ 
where 
$
  \phi_{\alpha}(x) 
  = 
  \frac{1}{\alpha}
  \left[
    \log\left( 1 + \exp(-\alpha x)\right)
    +
    \log\left(1 + \exp(\alpha x)\right)
  \right]
$.
This surrogate function $F_{\alpha}(\beta)$ is twice differentiable by taking advantage of the non-negative projection operator of $|x|$ (seeing equations (2) and (3) in \cite{smoothlassoClass01} for more details).
Consequently, the EM algorithm \citep{smoothlassoClass01EM} is used for the optimization.
The fifth representative algorithm is the path-following (PF) algorithm in \cite{tibshirani2011solution, rosset2007piecewise, park2007l1}.
It begins with a large penalty parameter $\lambda$, which leads all the estimated coefficients to $0$.
Then it tries to identify a sequence of decreasing penalty parameter $\lambda$, such that when $\lambda$ is between two kink points, the support set (the set of non-zero entries of estimated $\beta$) remains unchanged.
Moreover, the estimated $\beta$ elementwisely is a linear function of $\lambda$.
However, when one is over the kink point, the support is changed.

It is worth mentioning that, part of the algorithms in the after-warm-up stage can also be used in the warm-up stage.
These algorithms are ISTA, FISTA, CD, and SL.
For example, one can run 20 iterations in ISTA and input the ISTA's solution as the initial point in the after-warm-up stage, where FISTA will run 1,000 iterations until convergences. 
Under this scenario, ISTA can be regarded as a warm-up strategy.
In the remainder of this paper, we will not only compare HOSKY with representative warm-up strategies (like ridge regression), but we will also compare it with some after-warm-up algorithms (like ISTA, FISTA, CD, and SL) since they can be both applied in warm-up stage and after-warm-up stage.


\subsection{Our Motivation and Contribution}
\label{sec: contribution}

Our motivation to develop the HOSKY algorithm includes two. 
First, if one uses ridge regression to assign the initial point -- which is frequently used -- one will end up with at least $O(p^2 \log(p))$ numerical operations. 
Yet, in Lasso, usually, the number of covariates $p$ is much larger than the sample size $n$, i.e., $p \gg n$.
So it is computationally expensive if one uses ridge regression to assign initial points.
And this makes it desirable to develop a strategy with lower computational complexity.
Second, if one uses a pre-specified vector like $(0,\ldots, 0)'$ as the initial point-- which is also frequently used --  and runs a selected algorithm (say, FISTA) until convergence, the convergence rate is at best $O(1/\sqrt{\epsilon_{w+}})$.
This strategy, in the other words, can be regarded as using FISTA in both the warm-up stage and after-warm-up stage: in the warm-up stage, FISTA stops until $\epsilon_w$ is achieved; in the after-warm-up stage, FISTA stops when $\epsilon_{w+}$ is achieved. 
Thus, the overall convergence rate is the exactly the convergence rate of FISTA, i.e., $O(1/\sqrt{\epsilon_{w+}})$.
However, if one can expedite the warm-up stage with convergence rate of $O(\log(1/\epsilon_w))$, then the overall convergence is likely lower than $O(1/\sqrt{\epsilon_{w+}})$.
(See a visualization of our motivations in Fig. \ref{fig: motivation}).
Given the above two motivations, we develop the HOSKY algorithm, which accelerates the computation in the warm-up stage, and thus improves the overall convergence rate.

\begin{figure}[htbp]
    \centering
    \includegraphics[width = 0.7\textwidth]{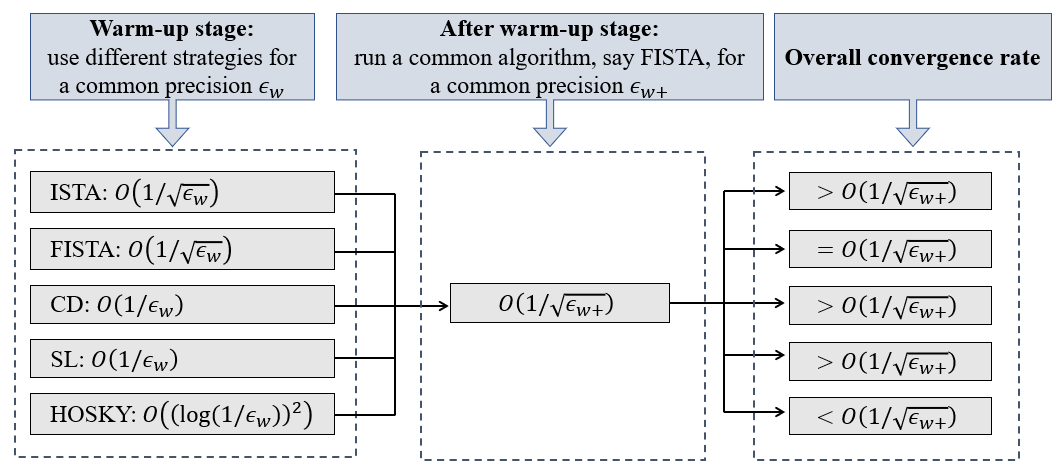}
    \caption{The advantage to apply HOSKY in the warm-up stage}
    \label{fig: motivation}
\end{figure}

The main contribution of the proposed HOSKY algorithm is its {\bf provable lower order of computational complexity} in the warm-up stage, compared with the existing algorithms introduced in Section \ref{sec: literature review}.
Specifically, to achieve a common warm-up precision $\epsilon_w$ in Criterion \ref{def: epsilon precision}, HOSKY achieves a \textit{log-polynomial} order of $1/\epsilon_w$, while ISTA, FISTA, CD, SL have a \textit{polynomial} order of $1/\epsilon_w$ (see Table \ref{table: prediction error and iteration number of existing lasso algorithm}).
So, HOSKY has lower computational complexity than ISTA, FISTA, CD, and SL.
For ridge regression, since it gives a closed-form solution, its computational complexity is independent of $\epsilon_w$.
For PF, we exclude it from the comparison, since it only works for a subset of the Lasso problem and there is no guarantee that its computational complexity is bounded.
Another contribution of HOSKY is that, as reflected by various simulations in Section \ref{sec: simulation}, with the speed-up on the warm-up stage, the computation in the after-warm-up stage can also be expedited.
This demonstrates that our proposed HOSKY algorithm is beneficial to the final convergences, and thus indicates the importance of our work.

\begin{table}[htbp]
  \caption{The the provable upper bounds in convergence rate of  HOSKY and its benchmarks for achieving a common warm-up precision $\epsilon_w$.
  \label{table: prediction error and iteration number of existing lasso algorithm}}
  \centering
  \begin{adjustbox}{max width=0.95\textwidth}
  \begin{threeparttable}
  \begin{tabular}{c|cccccc}
    \hline
    method\textsuperscript{1} &
    ISTA\textsuperscript{2} & FISTA\textsuperscript{3} & CD\textsuperscript{4} &
    SL\textsuperscript{5} & HOSKY \\
	\hline
	Order of complexity
	&
	$O( p^2/\epsilon_w)$
	&
	$O( p^2/\sqrt\epsilon_w)$
	&
	$O( p^2/\epsilon_w)$
	&
	$O( p^2/\epsilon_w)$
	&
	$
    O\left(
    \left[
    p^2
    \log(1/\epsilon_w)
    \right]^2 \right)
    $\\
    \hline
  \end{tabular}
  \begin{tablenotes}
  \footnotesize
  \item[1] Ridge regression is excluded since its complexity is not in terms of $\epsilon_w$.
  \item[2,3,4,5] These methods are reviewed in Section \ref{sec: literature review}.
  \end{tablenotes}
  \end{threeparttable}
  \end{adjustbox}
\end{table}

The organization of the rest of the paper is articulated as follows.
We develop our proposed HOSKY algorithm in Section \ref{sec: our algorithm}.
The related main theory is established in Section \ref{sec: order of complexity of HS}.
Numerical examples are shown in Section \ref{sec: simulation}.
Some discussions are presented in Section \ref{sec:discuss}.
In  \ref{appendix: Review of some State-of-the-art Algorithms}, we summarize some necessary technical details of these benchmark algorithms.
A useful theorem on accelerated gradient descents is restated in  \ref{appendix: george lan theorem on AGD}.
All the technical proofs are relegated to  \ref{app:proofs}.



\section{The Proposed Algorithm}
\label{sec: our algorithm}

The main idea of the HOSKY algorithm is articulated as follows.
Instead of minimizing $F(\beta)$ in \eqref{equ: lasso estimator} directly, we minimize a sequence of surrogate functions
$
  F_{t_0}(\beta),  F_{t_1}(\beta), \ldots, F_{t_K}(\beta)
$
in a sequential manner.
Specifically, we minimize $F_{t_0}(\beta)$ first and then minimize $F_{t_1}(\beta)$, until we arrive at $F_{t_K}(\beta)$.
And the length of the surrogate functions $K$ is decided by the pre-specified warm-up precision $\epsilon_w$: a small $\epsilon_w$ is likely to lead a large value of $K$ and vice versa.
A nice propriety of the sequence of the surrogate functions is that, it gets closer and closer to $F(\beta)$ when $k \to K$, and we call it as a \textit{homotopy path}.
Additionally, when minimizing a given surrogate function $F_{t_k}(\beta)$ for any $k = 0, 1, \ldots, K$, we applied the accelerated gradient descent (AGD) algorithm.

Technically, the above HOSKY algorithm involves two loops: in the outer-loop, we update $F_{t_{k-1}}(\beta)$ to $F_{t_{k}}(\beta)$; and in the inner-loop, we minimize $F_{t_k}(\beta)$ by the AGD algorithm.
By optimizing this sequence of surrogate functions, one can get an iterative estimator $\beta^{(k)}$, which can be served as an initial point to minimize $F(\beta)$.
The visualization of the two types of loops in HOSKY is available in Fig. \ref{fig: vis HOSKY main idea}.

To enable the above homotopy path idea, there are three technical blocks. 
The first block is the design of the surrogate function $F_{t}(\beta)$ for any $t \in \{t_0, \ldots, t_K\}$.
The second block is the design of the hyper-parameters $\{t_k\}_{k=0, \ldots, K}$, which forms the outer-loops in HOSKY.
The third block is the optimization strategies to minimize $F_{t_k}(\beta)$ for any $k=0, \ldots, K$, which forms the inner-loops in HOSKY.
In the remainder of this section, we will discuss these three blocks separately in Section \ref{sec: design of ft beta}, Section \ref{sec:initial-t}, and Section \ref{sec:inner-stop}.

\begin{figure}[htbp]
    \centering
    \includegraphics[width = 0.8\textwidth]{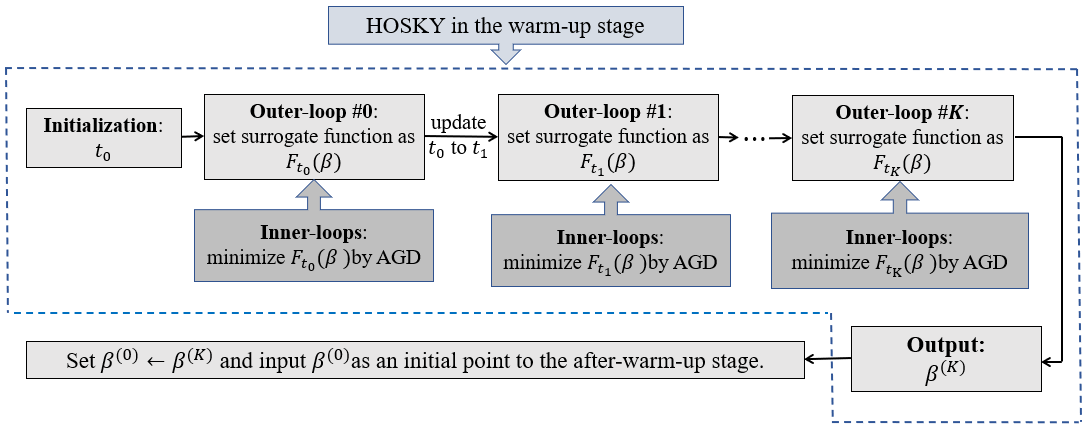}
    \caption{The main idea of the proposed HOSKY algorithm.}
    \label{fig: vis HOSKY main idea}
\end{figure}

\subsection{Design of Surrogate Functions}
\label{sec: design of ft beta}

In this section, we discuss the design of the surrogate function of $F_t(\beta)$ for a general hyper-parameter $t > 0$. 
In this paper, we design $F_t(\beta)$ in the form of 
\begin{equation}
    \label{equ: F_tk define}
    F_t(\beta)
    =
    \frac{1}{2n}\Vert y - X\beta\Vert_2^2
    +
    \lambda f^*_t(\beta),
\end{equation}
where the surrogate function replaces the $\ell_1$ penalty $\|\beta\|_1$ into $f^*_t(\beta)$.
Here $f^*_t(\beta) = \sum_{i=1}^{p}f_t(\beta_i)$ where $\beta_i$ is the $i$-th entry of $\beta$ and the function $f_t(\cdot): \mathbb R \to \mathbb R$ is 
\begin{eqnarray}
	\label{equ: ft}
	f_t(x)	&=& \left\{
	\begin{array}{ll}
		\frac{1}{3 t^3}   \left[ \log(1+t) \right]^2 x^2, & \mbox{ if } |x|\leq t, \\
		\left[  \frac{\log(1+t)}{t} \right]^2 |x| +  \frac{1}{3 |x|} \left[  \log(1+t) \right]^2  -\frac{1}{t}\left[ \log(1+t) \right]^2,
		&  \mbox{ otherwise. }
	\end{array}
	\right.
\end{eqnarray}
Here $t>0$ is a hyper-parameter controlling the closeness between $|x|$ and $f_t(x)$.
And we will discuss the value of $t$ in Section \ref{sec:initial-t}.
In Fig. \ref{fig: ft(x) and |x| closeness}, we display the curve of $f_t(x)$ under different values of $t$.
It can be seen that, when $t$ gets smaller, $f_t(x)$ become closer to the counterparts of the function $|x|$.

For the above surrogate function $f_t(x)$ and $F_t(\beta)$, they have three nice proprieties. 
The first nice property of $f_t(x)$ is that it is quadratic near $0$ and almost linear outside.
This property make the overall surrogate function $F_t(\beta)$ both strongly convex and well conditioned.
Consequently, a lower order of complexity becomes achievable when we minimize $F_t(\beta)$ by the AGD algorithm.
Specifically, one can prove a log-polynomial computational complexity for this algorithm at the warm-up stage. 
The second nice property of $f_t(x)$ is that, if we decrease the value of $t$, the surrogate function $f_t(\beta)$ gets closer and closer to $|x|$.
This is exactly what happens in HOSKY: when the outer-loop index $k \to K$, we get a decreasing hyper-parameter sequence $t_0 > t_1 > \ldots t_K > 0$.
Accordingly, the surrogate function sequence $\{f_{t_k}(x)\}_{k = 0, \ldots, K}$ gets closer and closer to $|x|$.
The third nice propriety of $f_t(x)$ is that, the difference between $f_t(x)$ and $|x|$ can be bounded, as shown in Lemma \ref{lemma:closeness-between-ft-and-x}.

\begin{figure}[htbp]
  \centering
  \begin{tabular}{ccc}
    \includegraphics[width=0.3\textwidth]{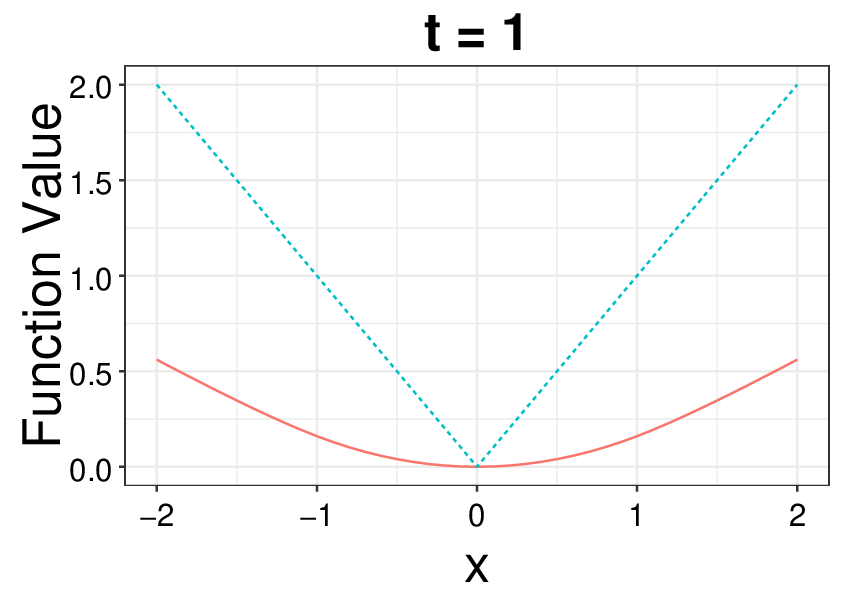}   &
	\includegraphics[width=0.3\textwidth]{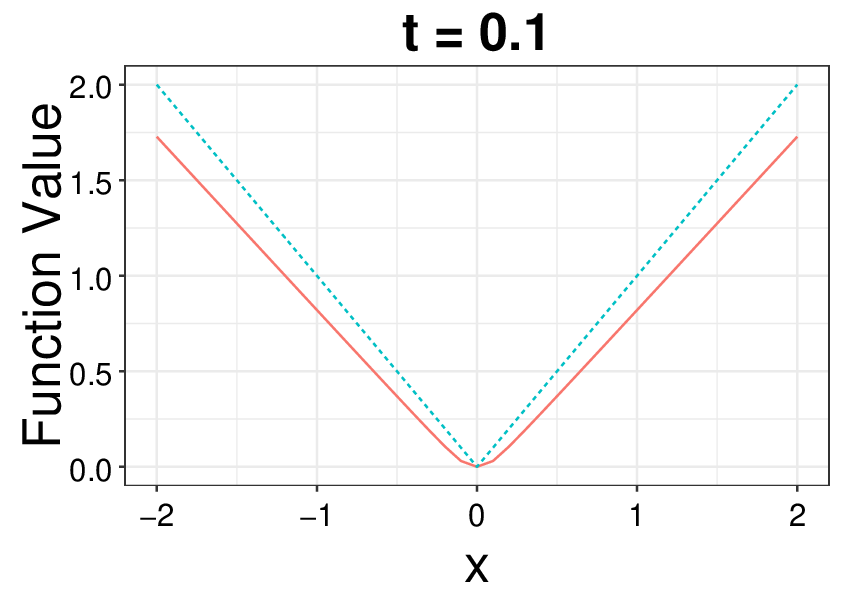}  &
	\includegraphics[width=0.3\textwidth]{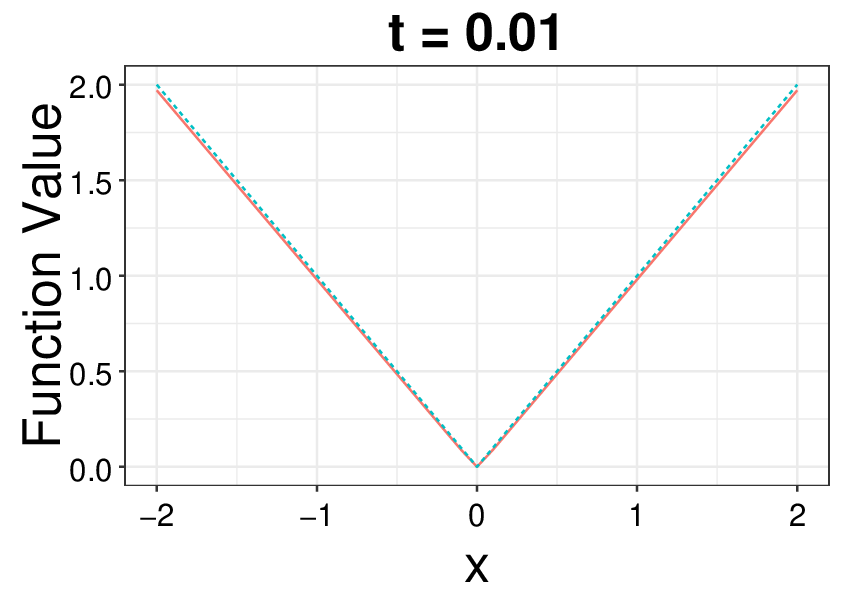}  \\
  \end{tabular}
  \caption{
  \label{fig: ft(x) and |x| closeness}
	The red and blue solid line represent $f_t(x)$ and  $|x|$, respectively. 
  }
\end{figure}

\begin{mylemma}
\label{lemma:closeness-between-ft-and-x}
    Suppose that from the beginning of our algorithm to the end of our algorithm, we have that $\beta^{(k)}_i \leq B$ for any $i = 1 \ldots p$ and $k = 0, 1, \ldots, K$.
    Then for any $k$, the surrogate function $f_{t_k}(x)$ in equation \eqref{equ: ft} has the following property:
    \begin{eqnarray}
	\label{equ: fx-x}
	  f_{t_k}(B)- B  \leq  f_{t_k}(x)-|x|  \leq   0.
    \end{eqnarray}
\end{mylemma}

In the above lemma, viewing from the right-hand side of inequality \eqref{equ: fx-x}, we find the surrogate function $f(t)$ is always below $|x|$.
Besides, viewing from the left-hand side of inequality \eqref{equ: fx-x}, we find shows that $f(t)$ is not too below $|x|$.
This inequality guarantees the estimator of HOSKY is close to $\widehat\beta$ because their objective function is close, which can be used as initial points in the after-warm-up stage.

The motivations to design the surrogate function $f_t(\beta)$ include two.
First, we hope the surrogates $F_{t}(\beta)$ are strongly convex and well-conditioned.
If so, the gradient descent method (like the AGD algorithm) can achieve a very fast convergence rate.
However, for the original objective function $F(\beta)$ in Lasso, it is not strongly convex given the $\ell_1$ norm ($\|\beta\|_1$ in our paper).
Second, it is widely acknowledged that the quadratic function (such as $\|\beta\|_2^2$) can be easily proved to be strongly convex.
Motivated by the aforementioned two facts, we try to replace $\|\beta\|_1$ by $f_t(\beta)$, which is quadratic near $0$ and almost linear outside.
By making this replacement, the surrogates $F_t(\beta)$ can be strongly convex.
Yet, it is nontrivial to find a good surrogate function $f_t(\beta)$.
We list the requirements of $f_t(\beta)$ in Condition  \ref{cond: requirement of ft(x)}.
\begin{mycondition}
	\label{cond: requirement of ft(x)}
Desirable conditions for function $f_t(x)$   are in the following.
	\begin{enumerate}
		\item When $t$ gets closer to 0, we have $f_t(x)$ close to the absolute value function $|x|$.
		\item Function $f_t(x)$ has the second derivative with respective to $x$.
		\item For fixed $t > 0$, function $x \mapsto f_t(x)$ is quadratic on $[-t, t]$, here $\mapsto$ indicates that the left hand side (i.e., $x$) is the variable in the function in the right hand side (i.e., $f_t(x)$).
		We follow this convention in the rest of this paper.
		\item Function $x\mapsto f_t(x)$ is $C^1$.
		Here $C^1$ is  the set of all continuously differentiable functions.
	\end{enumerate}
\end{mycondition}

\begin{proof}
  See \ref{proof: lemma closeness between ft and x}.
\end{proof}

We acknowledge that, the design of $f_t(x)$ in equation \eqref{equ: ft} is not unique but needs to satisfy some special requirements.
	Generally speaking, we can assume that $f_t(x)$ has the following format:
	\begin{eqnarray}
		\label{equ: ft general}
		f_t(x)	&=& \left\{
		\begin{array}{ll}
			d(t)  x^2, & \mbox{ if } |x|\leq t, \\
			a(t) |x| + b(t) g(x) +c(t), &  \mbox{ otherwise. }
		\end{array}
		\right.
	\end{eqnarray}
	Then requirement in Condition \ref{cond: requirement of ft(x)} is equivalently transformed into:
	\begin{enumerate}
		\item both $x \mapsto  f_t(x)$ and $t \mapsto  f_t(x)$  are $C^1$.
		\item $a(0)=1$, $b(0)=0$, $c(0)=0$, so that $f_0(x) = |x|$.
	\end{enumerate}
	Besides, we  wish the second derivative of $f_t(x)$ has the format of $f''_t(x) = h(t)\max\{t, |x|\}^{\upsilon} $, where $h(t)$ is a function of $t$ and $\upsilon$ is a constant.
	Accordingly, it is reasonable to suppose that $g(x) = \frac{1}{ (1-\upsilon) (2-\upsilon) } x^{2-\upsilon}$.
	Combining all the requests above, one has
	$$
	a(t)=\frac{\upsilon}{1+\upsilon} t^{1+\upsilon} b(t).
	$$
	Since $a(0)=1$, we choose $b(t)=\frac{1+\upsilon}{\upsilon} \left[ \log(1+t) \right]^{1+\upsilon}$.
Other choices of $b(t)$ can be $\sin(\cdot)$ function or other functions, which makes $t^{1-\upsilon} b(t)$ as a constant when $t=0$.

\begin{myremark}
Our idea is similar to \cite{lee} in appearance, however, the differences are as follows.
\begin{enumerate}
\item \cite{lee} aimed at $\ell_p$ penalty, where $p\notin\{1,2, +\infty\}$, while we focus on the $p=1$, which is not discussed in \cite{lee} and the theory in \cite{lee} is not easily-extendable to the situation when $p=1$.

\item \cite{lee} minimizes a linear function instead of the quadratic residual $\frac{1}{2n}\Vert y-X\beta\Vert_2^2$, where the Hessian matrix of the objective function needs different treatments.
\end{enumerate}
\end{myremark}

\subsection{Design of Outer-loops}
\label{sec:initial-t}

In this Section, we discuss the outer-loop of HOSKY. 
As indicated by Fig. \ref{fig: vis HOSKY main idea}, the objective in the $k$-th outer-loop is to set the surrogate function as $F_{t_k}(\beta)$ and then update the hyper-parameter from $t_k$ to $t_{k+1}$.
Since we already discuss the design of the surrogate function $F_{t_k}(\beta)$ in Section \ref{sec: design of ft beta}, we will focus on the hyper-parameter $\{t_k\}_{k = 0, \ldots, K}$ in this section.
Particularly, we will cover three parts: (1) the length of the hyper-parameters $K$, (2) the initial value of the hyper-parameter $t_0$, and (3) the updating rule from $t_k$ to $t_{k+1}$ for any $k = 0, \ldots, K$.

Theoretically, the total number of outer-loops $K$ is decided by the pre-fixed warm-up precision $\epsilon_w$, i.e.,
$$
  K = \min_{k} \left\{k:\; F({\beta^{(k)}}) - F(\widehat\beta) \leq \epsilon_w \right\}.
$$
With the above $K$, it is guaranteed that the difference between $F({\beta^{(k)}})$ and $F(\widehat\beta)$ is bounded by $\epsilon_w$.
And a small warm-up precision $\epsilon_w$ usually leads to a large number of outer-loops $K$.
However, in practice, one can always set $K$ as a pre-fixed positive number, say $K = 20$, without caring about the value of  $\epsilon_w$ and $F(\widehat\beta)$.
In this way, a larger value of $K$ tentatively leads to a smaller value of $\epsilon_w$.

For the initial value of the hyper-parameter $t_0$, a natural motivation is to keep it relatively small.
This is because, with an unnecessarily large $t_0$, one would end up with more outer-loops if one starts.
Consequently, one gets more numerical operations, which leads to higher computational complexity.
So a minimal value of $t_0$ is desired.
In our proposed HOSKY algorithm, we design such minimal $t_0$ to ensure that when $t = t_0$, the initial estimator $\beta^{(0)}$ is going to be bounded by $t_0$ entrywisely.
Under this motivation, we design the minimal value of $t_0$ in equation \eqref{equ: t0} of Lemma \ref{lemma: initial point}.
\begin{mylemma}
	\label{lemma: initial point}
	Suppose in a Lasso problem, we have the response vector $y \in \mathbb R^n$ and a model matrix $ X \in \mathbb R^{n \times p}$.
	For our proposed HOSKY algorithm, there exist a value $t_0$ that satisfies the following:
	\begin{equation}
    \label{equ: t0}
	  t_0
      \in
	  \left\{
      t:
      \left | \sum_{j=1}^{p} M(t)_{ij} ( X'y/n )_j  \right| \leq t,
      \;
      \forall i = 1, \ldots, p,
	  \right\}, \quad  
	\end{equation}
    where
    $
      M(t)
      =
      \left( \frac{X'X}{n} + \frac{ \lambda} {3t^3} \left[ \log(1+t ) \right]^2 I \right)^{-1}.
    $
    When one chooses the aforementioned $t_0$ as the initial point in the proposed algorithm,
    one has $\left| \beta^{(0)}_i \right| \leq t_0 $ for any $ i = 1, \ldots, p$, where $\beta^{(0)}_i$ denotes the $i$-th entry in the vector $\beta^{(0)}= M(t_0) X'y/n$.
\end{mylemma}
\begin{proof}
  See  \ref{proof: lemma initial point}.
\end{proof}

For the updating rule from $t_{k}$ and $t_{k+1}$ for any $k = 0, \ldots, K$, there is a closed-from correlation:
$
  t_{k+1} = t_k(1-h).
$
Here $h \in (0, 1)$ is set to be a predetermined value.
As one get $k \to K$, a decreasing sequence of $t_k$ forms, which makes the surrogate function $f_{t_k}(x)$ become closer and closer to $|x|$ (recall Fig. \ref{fig: ft(x) and |x| closeness}).
In practice, a small value of $h$ is preferred (say $0.1$), since we prefer $f_t(x)$ can get close to $|x|$ gently.

\subsection{Design of Inner-loops}
\label{sec:inner-stop}

In this section, we discuss the inner-loop in HOSKY.
Recall in Fig. \ref{fig: vis HOSKY main idea}, the objective of the inner loop is to minimize $F_{t_k}(\beta)$ by using the AGD algorithm, for any $k = 0, \ldots, K$.

A key question in the inner-loop is to decide the number of AGD iterations, or the number of inner-loops.
If one runs a large number of AGD iterations, the solution will be tentatively close to convergence.
However, a large number of AGD iterations leads to high computational complexity.
To save computation, we do not iterative the AGD algorithm until convergence, instead, we stop the AGD algorithm once a pre-specified inner-loop precision $\widetilde{\epsilon}_k$ is achieved.
Mathematically, in the $k$-th outer-loop, one can stop the AGD algorithm after $S_k$ inner-loops, where $S_k$ is defined as
\begin{equation}
    \label{equ: inner loop presicion}
    S_k = \min_s \left\{s: F_{t_k}(\beta^{(k)[s]}) - F_{\min,k} < \widetilde{\epsilon}_k \right\}.
\end{equation}
Here $\beta^{(k)[s]}$ denotes the AGD solution in the $s$-th inner-loop of the $k$-th outer-loop, and we have
$
  F_{\min,k} = \min_{\beta} F_{t_k}(\beta).
$
This pre-specified inner-loop precision $\widetilde{\epsilon}_k$ is set to control the convergence of the AGD algorithm is not very poor.
In HOSKY, one can set it as
\begin{equation}
  \label{equ: tilde epsilon k}
  \widetilde{\epsilon}_k
  =
  \frac{\lambda p}{3B} \left[ \log(1+t_k) \right]^2.
\end{equation}
The justification of our choice of $\widetilde{\epsilon}_k$ is elaborated in our proof, whose detailed derivation can be found in  \ref{proof: theo number of outer iterations}.
Please note that the above inner-loop precision $\widetilde\epsilon_k$ is different from the warm-up precision $\epsilon_w$ and after-warm-up precision $\epsilon_{w+}$.
Specifically, the warm-up precision $\epsilon_w$ decides the number of outer-loops, while the inner-loop precision $\widetilde\epsilon_k$ decides the number of inner-loops in the $k$-th outer-loop.

It is worth noting that, theoretically, our algorithm can achieve the order of complexity of $O\left( \left(\log(1/\epsilon_w) \right)^2 \right)$ in the warm-up stage.
Yet, in practice, it may not be implementable.
The matter of fact is that the stopping rule of the inner-loop requires knowing the value of $F_{\min,k} = \min_\beta F_{t_k}(\beta)$, which is not possible.
However, we may use some alternatives, such as stopping the inner-loop after a fixed number of inner-loops (say 100).
By using this alternative, if one sets the number of inner-loops conservatively (for example, set as 200 while 100 is theoretically sufficient), then one will end with a computational complexity higher than $O\left( \left(\log(1/\epsilon_w) \right)^2 \right)$.
And this is the reason why we state the proposed HOSKY algorithm has a provable computational complexity of $O\left( \left(\log(1/\epsilon_w) \right)^2 \right)$.

\subsection{Summary of the Proposed HOSKY Algorithm}
\label{sec: HS algorithm}

In this section, we summarize the proposed HOSKY algorithm, which has two layers of loops: the outer-loop and the inner-loop.
The pseudo code to summarize the objective of these two types of loops is available in Algorithm \ref{alg: HS pseudo code}, and a detailed implementation is available in Algorithm \ref{alg: HS}.

In outer-loops, we iterate the sequence of the surrogate functions $F_{t_0}(\beta), F_{t_1}(\beta), \ldots, F_{t_K}(\beta)$ defined in \eqref{equ: F_tk define}.
The difference between $F_{t_k}(\beta)$ and $F(\beta)$ lies in the last item: it is $\lambda f^*_{t_k}(\beta)$ in $F_{t_k}(\beta)$, while it is $\lambda \left\| \beta \right\|_1$ in $F(\beta)$.
By iterating $k$ in outer-loops, it forms a homotopic path with $F_{t_k}(\beta)$ getting closer and closer to $F(\beta)$ as $k \to K$.
And at the beginning of each outer-loop, it takes the stopping position from the previous outer-loop.

In inner-loops of the $k$-th outer loop, we iteratively minimize $F_{t_k}(\beta)$ in \eqref{equ: F_tk define} by the AGD algorithm.
Theoretically, one can save computations by stopping the AGD iterations earlier than convergence.
The theoretical stopping rule is shown in \eqref{equ: inner loop presicion}: we only require AGD to minimize $F_{t_k}(\beta)$ when a pre-specified AGD precision $\widetilde\epsilon_k$ arrives.
Yet in practice, this theoretical stopping rule is hard to exactly achieve given the unknown $F_{\min,k} = \min_\beta F(\beta)$ in \eqref{equ: inner loop presicion}.
Thus, a conservative way in the inner-loop is to run a relatively large number of AGD iterations to ensure  $\widetilde\epsilon_k$ is achieved, though one might end with relatively higher computations. 

The above ideas to develop HOSKY have two advantages.
First, in the inner-loops, we can get a convergence rate of $O(1/\log(\epsilon))$ when minimizing $F_{t_k}(\beta)$ by the AGD algorithm, since $F_{t_k}(\beta)$ is differentiable and strongly convex.
And this is the fastest convergence rate that can be achieved.
Second, in the outer-loops, the homotopy path$\left\{ F_{t_k}(\beta) \right\}_{k = 0, \ldots, K}$ gets closer and closer to $F(\beta)$ when $k \to K$.
And this improved closeness helps to reduce the approximation error between $F_t(\beta)$ to $F(\beta)$.

\begin{algorithm}[htbp]
	\caption{Pseudo code of the proposed HOSKY algorithm }
	\label{alg: HS pseudo code}
	\LinesNumbered
	\KwIn{
	\begin{enumerate}
	    \item $y \in \mathbb R^n$: response vector;
	    \item $X \in \mathbb R^{n \times p}$: model matrix;
	    \item $\lambda$: the turning parameter trading-off the goodness-of-fit and the sparsity of the Lasso estimator;
	    \item $K$: total number of outer-loops;
	    \item $\{S_k\}_{k = 0, \ldots, K}$: total number of inner-loops at the $k$-th outer loops $S_k$.
	\end{enumerate}
	}
	\KwOut{$\beta^{(K)}$: an estimation of $\beta$ after $K$ outer-loops.}
	Initialization \tcp*{see Section \ref{sec:initial-t}} 
	\textit{ $\blacktriangleright$ Outer-Loop: $\blacktriangleleft$ }
    \For{
      \textnormal{$k = 0,  1, \ldots, K$}
    }{
      \textnormal{Set the current objective function as $F_{t_k}(\beta)$} \tcp*{see equation \eqref{equ: F_tk define}} 
      \textit{ $\blacktriangleright$ Inner-Loop: $\blacktriangleleft$}
      \For{$s = 1, 2, \ldots, S_k$}{
         \textnormal{run the AGD algorithm to minimize $F_{t_k}(\beta)$}
      }
      \textnormal{update $t_k$} \tcp*{see Section \ref{sec:initial-t}} 
    }
\end{algorithm}

\begin{algorithm}[htbp]
\caption{A detailed implementation of the proposed HOSKY algorithm}
\label{alg: HS}
	\LinesNumbered
    \KwIn{
	\begin{enumerate}
	    \item $y \in \mathbb R^n$: response vector;
	    \item $X \in \mathbb R^{n \times p}$: model matrix;
	    \item $\lambda$: the turning parameter trading-off the goodness-of-fit and the sparsity of the Lasso estimator;
	    \item $K$: total number of outer-loops;
	    \item $\{S_k\}_{k = 0, \ldots, K}$: total number of inner-loops at the $k$-th outer loops $S_k$.
	\end{enumerate}
	}
	\KwOut{$\beta^{(K)}$: an estimation of $\beta$ after $K$ outer-loops.}
	\textnormal{initialization:}
    \label{algLine: HS initial point}
    $t_0, h, $ 	
	$
	  \beta^{(0)}
      =
	  \left[ 
	    X'X +  \frac{2 n \lambda \left[ \log(1+t_0 ) \right]^2 }{ 3t_0^2 }  I 
	  \right]^{-1} X'y
	$
	\tcp*{see equation \eqref{equ: t0}} 
	\textit{ $\blacktriangleright$ Outer-loop: $\blacktriangleleft$ }
    \label{algLine: HS outer-iteration (start)}
    \For{
        $k = 0, 1, \ldots, K$
    }
    {
        $\beta^{(k)[0]} = \beta^{(k-1)}$ 
        \tcp*{the solution of $\beta$ at the $0$-th inner-loop of the $k$-th outer-loop}
        $\underline{\beta}^{(k)[0]} = \beta^{(k-1)}$
        \label{algLine: HS () and []}
        \tcp*{an auxiliary variable at the $0$-th inner-loop of the $k$-th outer-loop} 
        Get the Lipschitz continuous gradient of $F_{t_k}(\beta)$ and denote it as $L_k$, s.t., 
           $
             \left\|
             \nabla F_{t_k}(x) - \nabla F_{t_k}(y)
             \right\|_2^2
             \leq
             L_k \left\| x - y \right\|_2.
           $ \\
        Get the strongly convexity of $F_{t_k}(\beta)$ and denote it as $\mu_k$ s.t., 
            $
              F_{t_k}(y) \geq 
              F_{t_k}(x) + \nabla F_{t_k}(x)(y-x) + 
              \frac{\mu_k}{2} \left\| y - x \right\|_2^2.
            $ \\
        $\alpha_k = \sqrt{\mu_k / L_k}$; \\
        $q = (\alpha_k - \mu_k/L_k)/(1-\mu_k/L_k)$; \\
        $\gamma = (\alpha_k)/(\mu_k(1 - \alpha_k))$; \\
        \textit{ $\blacktriangleright$ Inner-loop: $\blacktriangleleft$ }
        \label{algLine: HS inner-iteration (start)}
        \For{
            $s = 1, 2, \ldots, S_k$
        }
	    {   
            $
              \bar{\beta}^{(k)[s]}
              =
              (1-q) \beta^{(k)[s-1]} 
              + 
              q \underline{\beta}^{(k)[s-1]}
            $
            \label{algLine: AGD line1}
            $
              \underline{\beta}^{(k)[s]}
              =
              \left[
                \mu_s \bar{\beta}^{(k)[s]}
                +
                \underline{\beta}^{(k)[s-1]}
                -
                \gamma \nabla F_{t_k}(\bar{\beta}^{(k)[s]})
              \right]/ (\mu_s + 1)
            $ 
            \label{algLine: AGD line2}
            $
              \beta^{(k)[s]} 
              = 
              (1-\alpha_s) \beta^{(k)[s-1]} 
              + 
              \alpha_s \underline{\beta}^{(k)[S_k]}
            $
            \label{algLine: AGD line3}
    }
    $\beta^{(k)} = \beta^{(k)[s]}$ 
    \label{algLine: beta k = beta k s} \\
	$t_{k+1} = t_{k} (1-h)$  \label{algLine: HS shinkage t}\\
    }
\end{algorithm}


\section{Order of Complexity of the HOSKY Algorithm}
\label{sec: order of complexity of HS}

This section discusses the order of computational complexity of the proposed HOSKY algorithm.
Recall that the order of computational complexity is defined as the number of numerical operations needed to achieve the warm-up precision $\epsilon_w$ in \eqref{equ: F precision}, and always comes in a big $O(\cdot)$ notation.
Because our proposed HOSKY  algorithm involves two layers of loops, the order of computational complexity is in the order of the product of
(i) the number of inner-loops,
(ii) the number of numerical operations in each inner-loop,
(iii) the number of outer-loops.
In the remainder of this section, we will discuss (i), (ii), and (iii), respectively.

First, the number of inner-loops can be found in Lemma \ref{theo: number of inner-iterations}.
\begin{mylemma}[Number of inner-loops]
\label{theo: number of inner-iterations}
Recall that a Lasso problem has a response vector $y \in \mathbb R^n$ and a model matrix $X \in \mathbb R^{n \times p}$.
To minimize the Lasso objective function
$
  F(\beta)
  =
  \frac{1}{2n} \left\| y - X\beta \right\|_2^2  + \lambda \| \beta \|_1,
$
our proposed HOSKY algorithm minimizes
$
  F_{t_k}(\beta)
  =
  \frac{1}{2n} \left\| y - X\beta \right\|_2^2
  +
  \lambda f_{t_k}(\beta)
$
in the $k$-th outer-loop by the AGD algorithms.
Instead of converging to the minimizer of $F_{t_k}(\beta)$, we apply an early stopping rule \eqref{equ: inner loop presicion} with $\widetilde\epsilon_k$ set as in \eqref{equ: tilde epsilon k}.
It is guaranteed that, under the condition of Lemma \ref{lemma:closeness-between-ft-and-x}, one can achieve the inter-loop precision $\widetilde\epsilon_k$ in \eqref{equ: inner loop presicion} after $C_1 \log(1/\widetilde\epsilon_k)$ inner-loops, where $C_1$ is a constant that does not depend on the value of $k$.
\end{mylemma}
\begin{proof}
  See  \ref{proof: theo number of inner iterations}.
\end{proof}


Second, we notice the number of numerical operations in one inner-loop is of order $O(p^2)$.
This is because the computation is dominated by the matrix multiplication of $\beta' \nabla F_{t_k}(\underline{\beta}^{(k)[s]})$ as shown in Line \ref{algLine: AGD line2} in Algorithm \ref{alg: HS}.

Finally, the minimal number of outer-loops is summarized in Lemma \ref{theo: number of outer iteration}.
\begin{mylemma}[Number of outer-loops]
	\label{theo: number of outer iteration}
	With the conditions listed in Lemma \ref{lemma:closeness-between-ft-and-x} being satisfied, one will get the number of outer-loops
	$
      k
      \geq
      \frac{-1}{\log(1-h)} \log\left( \frac{\lambda p t_0 (2B+1)}{\epsilon_w} \right),
    $
    then the proposed HOSKY algorithm will find a solution $\beta^{(k)}$ such that \eqref{equ: F precision} is satisfied.
\end{mylemma}

\begin{proof}
  See \ref{proof: theo number of outer iterations}.
\end{proof}

With all the above blocks, we develop the main theory, i.e., the order of computational complexity to achieve the warm-up precision $\epsilon_w$ of our proposed HOSKY algorithm.

\begin{mytheorem}[Main theory]
\label{theo: HS order of complexity}
    Under the conditions in Lemma \ref{lemma:closeness-between-ft-and-x}, we can find $\beta^{(k)}$ such that the warm-up precision $\epsilon_W$ in \eqref{equ: F precision} is satisfied after
	$$
	p^2
    O\left(
    \left[
    \frac{-1}{\log(1-h)} \log\left( \frac{\lambda p t_0 (2B+1)}{\epsilon_w} \right)
    \right]^2
    \right)
	$$
	numerical operations.
\end{mytheorem}

\begin{proof}
  See  \ref{proof: HS order of complexity}.
\end{proof}

As we can see from Theorem \ref{theo: HS order of complexity}, the computational complexity of HOSKY is of \textbf{log-polynomial} of $1/\epsilon_w$.
However, the compared benchmarks in Section \ref{sec: literature review} is \textbf{polynomial} of $1/\epsilon_w$.
So we can declare the proposed HOSKY algorithm has lower computational complexity in the warm-up stage.



\section{Numerical Examples}
\label{sec: simulation}
In this section, we compare the performance of HOSKY with three benchmarks through various numerical experiments.
The selected three benchmarks are ridge regression, ISTA, and FISTA (reviewed in Section \ref{sec: literature review}).
Here we exclude CD and SL since they share the same order of computational complexity as ISTA.
Additionally, we exclude PF given its possible unbounded computational complexity and its restriction for applications to general cases.

In this section, three numerical examples are shown.
The first example handles image de-noising. 
Through this example, we will see the proposed HOSKY algorithm can deblur a noised image. 
Although the deblurred image is not completely clear -- since HOSKY focuses on the warm-up stage with warm-up precision $\epsilon_w \nrightarrow 0$ -- but it serves as a good initial point for people to discriminate the major characters in the images at first glance.
The second example compares the computational complexity of HOSKY with its benchmarks in the warm-up stage.
In this example, we use sparse linear regression for illustration purposes.
The third example investigates whether a good warm-up strategy expedites the convergences.
And this example can be regarded as a sequel to the second example.

The selection of the tuning parameters is articulated as follows. In both ISTA and FISTA, we set the Lipschitz continuous gradient $L$ as the maximal eigenvalue of the matrix $X'X/n$ (see detailed implementation in Algorithm \ref{alg: ISTA}, \ref{alg: FISTA}). Additionally, in FISTA, we set $t_1 = 1$ and $t_{k+1} = \frac{1+\sqrt{1+4t_k^2}}{2}$ for any $k = 1, 2, \ldots$ (see detailed implementation in Algorithm \ref{alg: FISTA}). In ridge regression, we calculate its closed-form solution as shown in \eqref{equ: ridge regression close-form solution}, so there is no need to select turning parameters. In HOSKY, we select the turning parameters following Algorithm \ref{alg: HS}. Specifically, in $k$-th outer iteration, we set the Lipschitz continuous gradient $L_k$ as the maximal eigenvalue of the Hessian matrix of $F_{t_k}(\beta)$. And the strongly convexity $\mu_k$ is set as the minimal eigenvalue of the Hessian matrix of $F_{t_k}(\beta)$. With both $L_k$ and $\mu_k$ available, we set $\alpha_k = \sqrt{\mu_k/L_k}, q = (\alpha_k - \mu_k/L_k)/(1-\mu_k/L_k)$ and $\gamma = (\alpha_k)/(\mu_k(1 - \alpha_k))$ to be used in the inner-iterations. Note that users are welcome to calculate the above turning parameters by the line search method \cite{lan2019lectures}.

\subsection{Simulation 1: Image De-nosing}
\label{sec: simulation1}

In this simulation, we compare our proposed HOSKY algorithm with ridge regression, ISTA, and FISTA in the application of image de-noising.
The example image we investigated is a $13 \times 26$ batman image.
The image goes through a Gaussian blur of size $9 \times 9$ and standard deviation $4$ followed by an additive zero-mean white Gaussian noise with standard deviation $10^{-3}$.
The original and observed images are given in Fig. \ref{fig: simu img} (a) and (b), respectively.

For these experiments, we assume reflexive (Neumann) boundary conditions.
We then test ridge regression, ISTA, FISTA, and HOSKY for solving problem \eqref{equ: lasso estimator}, where $y$ represents the vectorized observed image, and $X = RW$ with $R$ as the matrix representing the blur operator and $W$ as the inverse of a three-stage Haar wavelet transform.
The regularization parameter is selected to be $\lambda = 10^{-4}$.

The de-noising results are summarized in Figure \ref{fig: simu img}(c)-(f), where one can easily find that the image produced by HOSKY is of better quality than those created by its benchmarks.
Besides, the computational complexity of HOSKY is lower than its benchmarks, as shown in Table \ref{table: sim -- image --  comp complexity}.
Though it is not exactly lower than FISTA, they are on the same scale.
And the small difference might be caused by the hidden constant in the big $O(\cdot)$ notation.

\begin{figure}[htbp]
  \begin{center}
    \begin{tabular}{cc}
		\centering
		\includegraphics[width = 0.8\textwidth]{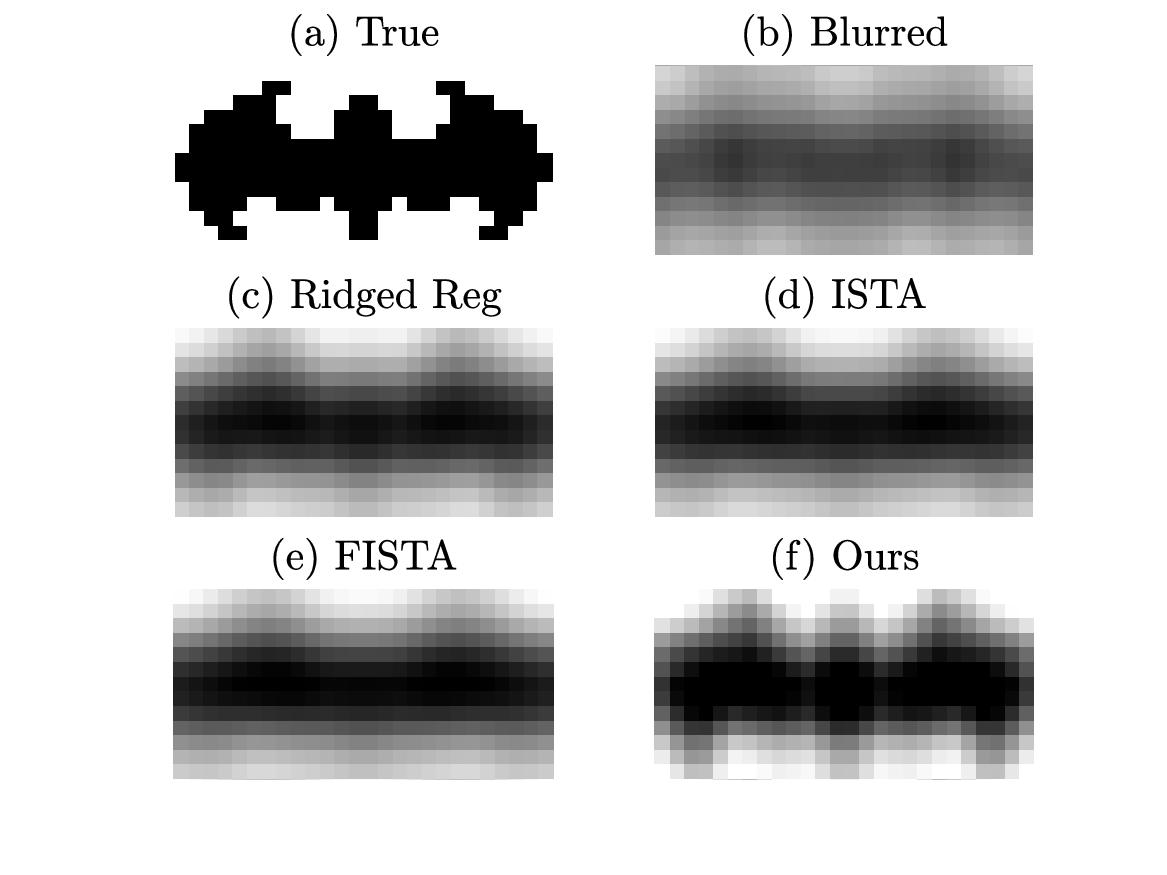}
	\end{tabular}
	\caption{True, blurred and denoised images by ridged regression, ISTA, FISTA and HOSKY.
    \label{fig: simu img}}
  \end{center}
\end{figure}

\begin{table}[htbp]
    \centering
    \begin{adjustbox}{max width=0.95\textwidth}
    \begin{threeparttable}
    \begin{tabular}{c|cccc}
         \hline
         method &
         ridge regression\textsuperscript{1} &
         ISTA\textsuperscript{2} &
         FISTA\textsuperscript{3} &
         HOSKY\\
         \hline
         Number of Numerical Operations & 26,314,540 & 2,098,802 & 663,601 & 867,357 \\
         \hline
    \end{tabular}
    \begin{tablenotes}
    \item[1,2,3] Ridge regression, ISTA, FISTA are introduced in Section \ref{sec: literature review}.
    \end{tablenotes}
    \end{threeparttable}
    \end{adjustbox}
    \caption{The computational complexity to de-noise Fig. \ref{fig: simu img} (b)}
    \label{table: sim -- image --  comp complexity}
\end{table}


\subsection{Simulation 2: Compare Computational Complexity in the Warm-up Stage}
\label{sec: simulation2}

In this section, we compare HOSKY with ridge regression, ISTA, FISTA under the scenario to estimate true parameters in the sparse linear regression models.
The data generation mechanism is described as follows, which follows \cite{glmnet}.
First, we generate Gaussian data with $n$ observations and $p$ covariates, with each predictor is associated with a random vector $X_j \in \mathbb R^n$, and the model matrix is  $X=(X_1, \cdots, X_j, \cdots, X_p)$.
Here we assume that the random vector $X_j$ follows the multivariate normal distribution with zero mean, variances being equal to $1$, and identical population correlation $\rho$, that is, the covariance matrix of $X_j$ has 1 on its diagonal and $\rho$ for its reminder entries.
In this simulation, we set $\rho = 0.1$.
The response values were generated by
\begin{equation}
\label{equ: sim -- generation formula}
  y = \sum_{j=1}^p X_j \beta_j + q z.
\end{equation}
For the value of $\beta_j \; (\forall 1\leq i \leq p)$, we discuss two scenarios:
\begin{itemize}
    \item Scenario 1: $\beta_i=(-1)^i \exp \left( -2(i-1)/20\right)$;
    \item Scenario 2: $\beta_i = (-1)^i \exp\left( -2(i-1)/20 \right) {\mathbf 1} \left\{ i \leq 10 \right\}$, where ${\mathbf 1}\left\{ \cdot \right\}$ is the identity  function, i.e., ${\mathbf 1}\left\{ x \in A \right\} = 1$ if $x\in A$, and ${\mathbf 1}\left\{ x \in A \right\} = 0$ otherwise.
\end{itemize}
Both scenarios are constructed to have alternating signs and to be exponentially decreasing.
And the difference between these two scenarios lies in the sparsity in the second scenario, i.e., its $\beta_i = 0$ when $i > 10$.
This parameter setting assumes that most of the entries in $\beta$ are zero, which renders a case with sparse truth.
Besides, $z = (z_1 \cdots z_p)'$ is the white noise with $z_i$ satisfying the standard normal distribution $N(0,1)$.
The quantity $q$ is chosen so that the signal-to-noise ratio is $3.0$.
The turning parameter $\lambda$ is set to be $10^{-3}$.
And in our simulation, we investigate two sets of $(n,p)$ values, i.e., $(n=50, p=20)$ and $(n=50, p=80)$.

The simulation results are summarized in Table \ref{table: simulation1 table}, Table \ref{table: simulation1 table --  running time} and visualized in Fig. \ref{fig: simulation1 -- curves}, Fig. \ref{fig: simulation1 -- runtime}.
We report the total number of numerical operations to achieve different values of the warm-up precision $\epsilon_w$ in Table \ref{table: simulation1 table} and Fig. \ref{fig: simulation1 -- curves}, which is independent of the computation platform. 
Additionally, we report the running time in Table \ref{table: simulation1 table --  running time} and Fig. \ref{fig: simulation1 -- runtime}, which is based on a Macbook pro with 2.3 GHz Intel Core i5. 
In both Table \ref{table: simulation1 table} and \ref{table: simulation1 table --  running time}, the values in cells are the number of operations/running time to achieve the different warm-up precision $\epsilon_w$.
In both Fig. \ref{fig: simulation1 -- curves}, \ref{fig: simulation1 -- runtime}, the dark, red, blue, and pink lines represent the number of numerical operations/running time of ridge regression, ISTA, FISTA, and HOSKY, respectively.
The x-axis is the $\log(1/\epsilon_w)$ and the y-axis is the logarithms of the number of numerical operations/ running time to achieve the corresponding warm-up precision $\epsilon_w$.

From Table \ref{table: simulation1 table},  \ref{table: simulation1 table --  running time} and Fig. \ref{fig: simulation1 -- curves}, \ref{fig: simulation1 -- runtime}, we find there are some common properties shared among different methods.
For example, iterative algorithms -- like ISTA, FISTA, and HOSKY -- take more numerical operations/running time to achieve a smaller value of $\epsilon_w$.
While for closed-form methods like ridge regression, its numbers of numerical operations are independent of $\epsilon_w$.

From Table \ref{table: simulation1 table},  \ref{table: simulation1 table --  running time} and Fig. \ref{fig: simulation1 -- curves}, \ref{fig: simulation1 -- runtime},
we also find differences among different methods.
Generally speaking, HOSKY has fewer numerical operations or less running time than its benchmarks.
For example, under the first scenario with $\epsilon_w = 0.005, n= 50, p = 80$, HOSKY only requires $3.6457 \times 10^4$ operations. However, ridge regression, ISTA, and FISTA need $3.73414 \times 10^5$, $1.90131 \times 10^5$, and $5.5133 \times 10^4$ operations, respectively.
And we also notice that when $n=50, p=20$ with large $\epsilon_w$, the number of numerical operations of ISTA, FISTA, and HOSKY are very similar.
This is because when $\epsilon_w$ is large, the hidden constant before the complexity ($O(1/\epsilon_w)$ for ISTA, $O(1/\sqrt{\epsilon_w})$ for FISTA, and $O(1/(\log(1/\epsilon_w)^2))$ for HOSKY) are dominated.

\begin{table}[htbp]
\caption{ \label{table: simulation1 table}  Numerical complexity of ridge regression, ISTA, FISTA, HOSKY in Simulation 2}
\centering
\begin{adjustbox}{max width=0.95\textwidth}
  \begin{threeparttable}
	\begin{tabular}{ccccccccccccccc}
		\hline&
        \multicolumn{8}{c}{Precision $\epsilon$} \\
		method &
        $0.05$ &
        $0.03$ &
        $0.02$ &
        $0.01$&
        $0.009$ &
        $0.008$ &
        $0.007$&
        $0.006$&
        $0.005$ \\
		\cline{1-10}
		& \multicolumn{8}{c}{ Scenario 1 $(n=50, \; p=20)$} \\
		\cline{2-10}
		Ridge Regression &
		$ 7,354 $  &
		$ 7,354 $  &
		$ 7,354 $  &
		$ 7,354 $  &
		$ 7,354 $  &
		$ 7,354 $  &
		$ 7,354 $  &
		$ 7,354 $  &
		$ 7,354 $  \\
        ISTA  &
        $ 5,070  $  &
        $ 6,016  $  &
        $ 7,005  $  &
        $ 9,585  $  &
        $ 10,101 $  &
        $ 10,703 $  &
        $ 11,434 $  &
        $ 12,294 $  &
        $ 13,369 $ \\
		FISTA  &
        \bm{$ 4,781 $}  &
        \bm{$ 5,117 $}  &
        \textbf{$ 5,453 $}  &
        $ 6,461 $  &
        $ 6,685 $  &
        $ 6,797 $  &
        $ 7,021 $  &
        $ 7,133 $  &
        $ 7,357 $\\
    	HOSKY     &
        $ 5,478 $  & 
        $ 5,478 $  & 
        \bm{$ 5,479 $}  & 
        \bm{$ 5,479 $}  & 
        \bm{$ 5,479 $}  & 
        \bm{$ 5,479 $}  & 
        \bm{$ 5,479 $}  & 
        \bm{$ 5,479 $}  & 
        \bm{$ 6,005 $}\\
        \cline{2-10}
		& \multicolumn{8}{c}{ Scenario 1 $(n=50, \; p=80)$} \\
		\cline{2-10}
		Ridge Regression &
		$ 373,414 $  &
		$ 373,414 $  &
		$ 373,414 $  &
		$ 373,414 $  &
		$ 373,414 $  &
		$ 373,414 $  &
		$ 373,414 $  &
		$ 373,414 $  &
		$ 373,414 $  \\
        ISTA  &
        $ 37,400  $  &
        $ 50,277 $ &
        $ 65,273 $  &
        $ 109,772 $ &
        $ 119,226 $&
        $ 130,799 $&
        $ 145,306 $&
        $ 164,377 $&
        $ 190,131 $  \\
		FISTA  &
        $ 31,237  $  &
        $ 34,533 $ &
        $ 37,417  $  &
        $ 45,657  $ &
        $ 47,305 $ &
        $ 48,541 $&
        $ 50,189 $&
        $ 52,249 $&
        $ 55,133 $   \\
		HOSKY     &
        \bm{$ 30,919  $}  &
        \bm{$ 30,919  $} & 
        \bm{$ 32,765  $}  &
        \bm{$ 34,611  $} & 
        \bm{$ 34,611 $} &
        \bm{$ 34,611 $}  &
        \bm{$ 36,457 $}  &
        \bm{$ 36,457 $}  &
        \bm{$ 36,457 $}  \\
        \cline{1-10}
		& \multicolumn{8}{c}{ Scenario 2: $(n=50, \; p=20)$} \\
		\cline{2-10}
		Ridge Regression &
		$ 7,354 $  &
		$ 7,354 $  &
		$ 7,354 $  &
		$ 7,354 $  &
		$ 7,354 $  &
		$ 7,354 $  &
		$ 7,354 $  &
		$ 7,354 $  &
		$ 7,354 $  \\
        ISTA  &
        $ 5,242  $  &
        $ 6,274  $  &
        $ 7,263  $  &
        $ 9,370  $  &
        $ 9,757  $  &
        $ 10,187 $  &
        $ 10,703 $  &
        $ 11,305 $  &
        $ 12,122 $ \\
		FISTA  &
        \bm{$ 4,781 $}  &
        \bm{$ 5,229 $}  &
        $ 5,565 $  &
        $ 6,125 $  &
        $ 6,349 $  &
        $ 6,461 $  &
        $ 6,573 $  &
        $ 6,797 $  &
        $ 7,021 $ \\
		HOSKY     &
        $ 5,479 $ &
        $ 5,479 $ & 
        \bm{$ 5,479 $} &
        \bm{$ 6,005 $} & 
        \bm{$ 6,005 $} &
        \bm{$ 6,005 $} &
        \bm{$ 6,005 $} &
        \bm{$ 6,005 $} &
        \bm{$ 6,005 $} \\
        \cline{2-10}
		& \multicolumn{8}{c}{ Scenario 2: $(n=50, \; p=80)$} \\
		\cline{2-10}
		Ridge Regression &
		$ 373,414 $  &
		$ 373,414 $  &
		$ 373,414 $  &
		$ 373,414 $  &
		$ 373,414 $  &
		$ 373,414 $  &
		$ 373,414 $  &
		$ 373,414 $  &
		$ 373,414 $  \\
        ISTA  &
        $ 39,519  $  &
        $ 55,330  $  &
        $ 72,119  $  &
        $ 112,869 $  &
        $ 120,693 $  &
        $ 130,473 $  &
        $ 142,698 $  &
        $ 158,346 $  &
        $ 179,373 $ \\
		FISTA  &
        $ 31,649 $  &
        $ 35,769 $  &
        $ 39,065 $  &
        $ 45,657 $  &
        $ 46,893 $  &
        $ 48,129 $  &
        $ 49,365 $  &
        $ 51,013 $  &
        $ 53,485 $ \\
		HOSKY     &
        \bm{$ 30,918 $}  & 
        \bm{$ 32,763 $}  &
        \bm{$ 32,763 $}  & 
        \bm{$ 32,763 $}  & 
        \bm{$ 32,763 $}  & 
        \bm{$ 32,763 $}  & 
        \bm{$ 32,763 $}  & 
        \bm{$ 32,763 $}  & 
        \bm{$ 32,763 $}  \\
        \cline{1-10}
	\end{tabular}
   \begin{tablenotes}
    \item[1] There is the parameters settings of our HOSKY algorithm: $t_0=3, h=0.1, \beta^{(0)} = \mathbf 0_{p \times 1}$.
   \end{tablenotes}
  \end{threeparttable}
\end{adjustbox}
\end{table}

\begin{table}[htbp]
\caption{ \label{table: simulation1 table --  running time}  Running time of ridge regression, ISTA, FISTA, HOSKY in Simulation 2}
\centering
\begin{adjustbox}{max width=0.95\textwidth}
  \begin{threeparttable}
	\begin{tabular}{ccccccccccccccc}
		\hline&
        \multicolumn{8}{c}{Precision $\epsilon$} \\
		method &
        $0.05$ &
        $0.03$ &
        $0.02$ &
        $0.01$&
        $0.009$ &
        $0.008$ &
        $0.007$&
        $0.006$&
        $0.005$ \\
		\cline{1-10}
		& \multicolumn{8}{c}{ Scenario 1 $(n=50, \; p=20)$} \\
		\cline{2-10}
		Ridged Regression &
		$ 0.0023 $  &
        $ 0.0019 $  &
        $ 0.0021 $  &
        $ 0.0023 $  &
        $ 0.0019 $  &
        $ 0.0017 $  &
        $ 0.0020 $  &
        $ 0.0019 $  &
        $ 0.0017 $ \\
        ISTA  &
        $ 0.0030 $  &
        $ 0.0034 $  &
        $ 0.0040 $  &
        $ 0.0052 $  &
        $ 0.0051 $  &
        $ 0.0055 $  &
        $ 0.0056 $  &
        $ 0.0060 $  &
        $ 0.0062 $ \\
		FISTA  &
        $ 0.0007 $  &
        $ 0.0007 $  &
        $ 0.0007 $  &
        $ 0.0009 $  &
        $ 0.0008 $  &
        $ 0.0009 $  &
        $ 0.0009 $  &
        $ 0.0009 $  &
        $ 0.0010 $ \\
    	HOSKY     &
        \bm{$ 0.0003 $}  & 
        \bm{$ 0.0003 $}  & 
        \bm{$ 0.0003 $}  & 
        \bm{$ 0.0003 $}  & 
        \bm{$ 0.0003 $}  & 
        \bm{$ 0.0003 $}  & 
        \bm{$ 0.0003 $}  & 
        \bm{$ 0.0003 $}  & 
        \bm{$ 0.0003 $}\\
        \cline{2-10}
		& \multicolumn{8}{c}{ Scenario 1 $(n=50, \; p=80)$} \\
		\cline{2-10}
		Ridged Regression &
		$ 0.0122 $  &
        $ 0.0087 $  &
        $ 0.0031 $  &
        $ 0.0030 $  &
        $ 0.0027 $  &
        $ 0.0029 $  &
        $ 0.0030 $  &
        $ 0.0031 $  &
        $ 0.0029 $ \\
        ISTA  &
        $ 0.0146 $  &
        $ 0.0218 $  &
        $ 0.0291 $  &
        $ 0.0680 $  &
        $ 0.0764 $  &
        $ 0.0877 $  &
        $ 0.1017 $  &
        $ 0.1194 $  &
        $ 0.1351 $ \\
		FISTA  &
        $ 0.0018 $  &
        $ 0.0020 $  &
        $ 0.0023 $  &
        $ 0.0037 $  &
        $ 0.0039 $  &
        $ 0.0042 $  &
        $ 0.0046 $  &
        $ 0.0049 $  &
        $ 0.0052 $ \\
		HOSKY     &
        \bm{$ 0.0005  $}  &
        \bm{$ 0.0005  $} & 
        \bm{$ 0.0005  $}  &
        \bm{$ 0.0005  $} & 
        \bm{$ 0.0005 $} &
        \bm{$ 0.0005 $}  &
        \bm{$ 0.0005 $}  &
        \bm{$ 0.0005 $}  &
        \bm{$ 0.0005 $}  \\
        \cline{1-10}
		& \multicolumn{8}{c}{ Scenario 2: $(n=50, \; p=20)$} \\
		\cline{2-10}
		Ridged Regression &
		$ 0.0019 $  &
        $ 0.0020 $  &
        $ 0.0029 $  &
        $ 0.0025 $  &
        $ 0.0094 $  &
        $ 0.0131 $  &
        $ 0.0070 $  &
        $ 0.0022 $  &
        $ 0.0022 $ \\
        ISTA  &
        $ 0.0028 $  &
        $ 0.0036 $  &
        $ 0.0045 $  &
        $ 0.0056 $  &
        $ 0.0065 $  &
        $ 0.0063 $  &
        $ 0.0059 $  &
        $ 0.0061 $  &
        $ 0.0062 $ \\
		FISTA  &
        $ 0.0006 $  &
        $ 0.0007 $  &
        $ 0.0008 $  &
        $ 0.0009 $  &
        $ 0.0010 $  &
        $ 0.0010 $  &
        $ 0.0009 $  &
        $ 0.0010 $  &
        $ 0.0010 $ \\
		HOSKY     &
        \bm{$ 0.0003 $} &
        \bm{$ 0.0011 $} & 
        \bm{$ 0.0003 $} &
        \bm{$ 0.0003 $} & 
        \bm{$ 0.0004 $} &
        \bm{$ 0.0003 $} &
        \bm{$ 0.0003 $} &
        \bm{$ 0.0003 $} &
        \bm{$ 0.0003 $} \\
        \cline{2-10}
		& \multicolumn{8}{c}{ Scenario 2: $(n=50, \; p=80)$} \\
		\cline{2-10}
		Ridged Regression &
		$ 0.0091 $  &
        $ 0.0048 $  &
        $ 0.0029 $  &
        $ 0.0028 $  &
        $ 0.0029 $  &
        $ 0.0029 $  &
        $ 0.0029 $  &
        $ 0.0042 $  &
        $ 0.0039 $ \\
        ISTA  &
        $ 0.0151 $  &
        $ 0.0199 $  &
        $ 0.0307 $  &
        $ 0.0707 $  &
        $ 0.0795 $  &
        $ 0.0927 $  &
        $ 0.1071 $  &
        $ 0.1207 $  &
        $ 0.1383 $ \\
		FISTA  &
        $ 0.0018 $  &
        $ 0.0019 $  &
        $ 0.0024 $  &
        $ 0.0038 $  &
        $ 0.0040 $  &
        $ 0.0044 $  &
        $ 0.0048 $  &
        $ 0.0049 $  &
        $ 0.0053 $ \\
		HOSKY     &
        \bm{$ 0.0010 $}  & 
        \bm{$ 0.0009 $}  & 
        \bm{$ 0.0005 $}  & 
        \bm{$ 0.0005 $}  & 
        \bm{$ 0.0005 $}  & 
        \bm{$ 0.0005 $}  & 
        \bm{$ 0.0006 $}  & 
        \bm{$ 0.0005 $}  & 
        \bm{$ 0.0005 $}  \\
        \cline{1-10}
	\end{tabular}
   \begin{tablenotes}
    \item[1] There is the parameters settings of our HOSKY algorithm: $t_0=3, h=0.1, \beta^{(0)} = \mathbf 0_{p \times 1}$.
    \item[2] The running time is based on a Macbook pro with 2.3 GHz Intel Core i5.
   \end{tablenotes}
  \end{threeparttable}
\end{adjustbox}
\end{table}

\begin{figure}[htbp]
  \begin{center}
    \begin{tabular}{cc}
		\centering
		\includegraphics[width=0.45\textwidth]{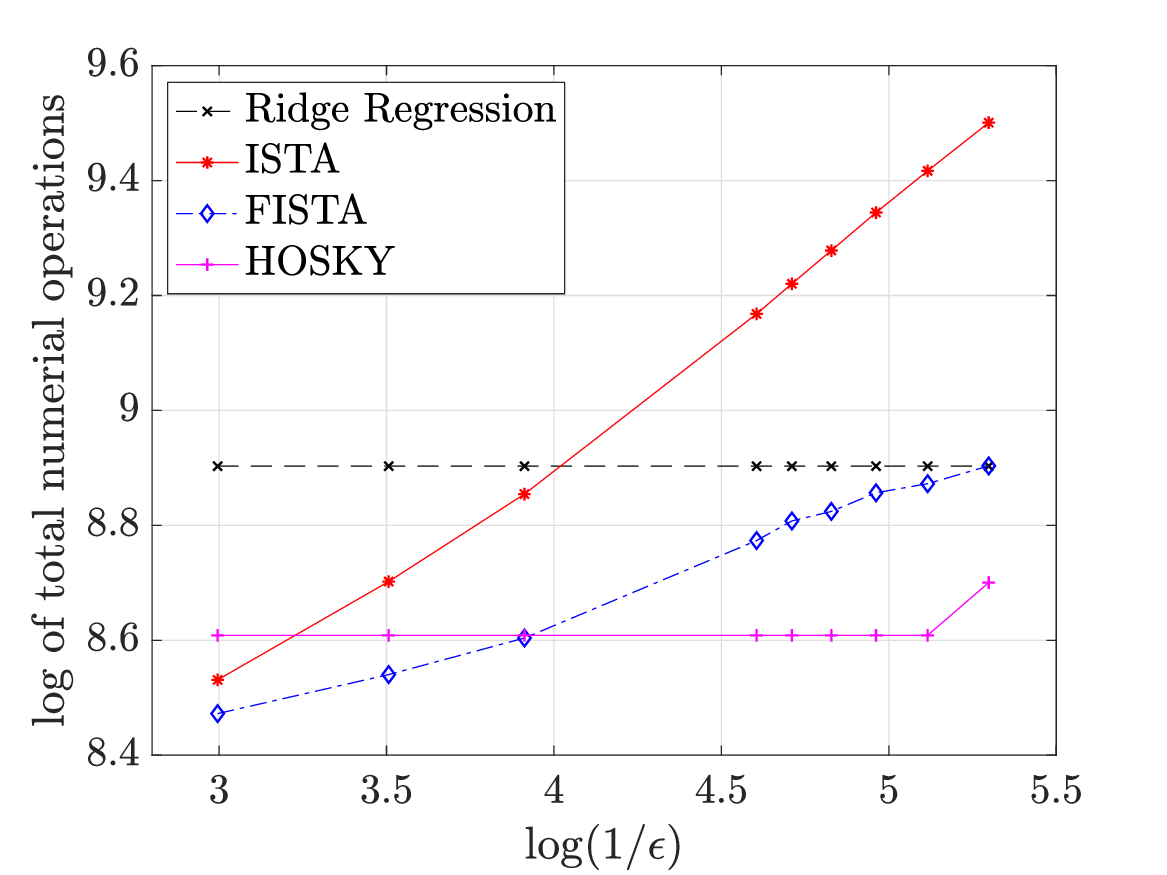}    &
		\includegraphics[width=0.45\textwidth]{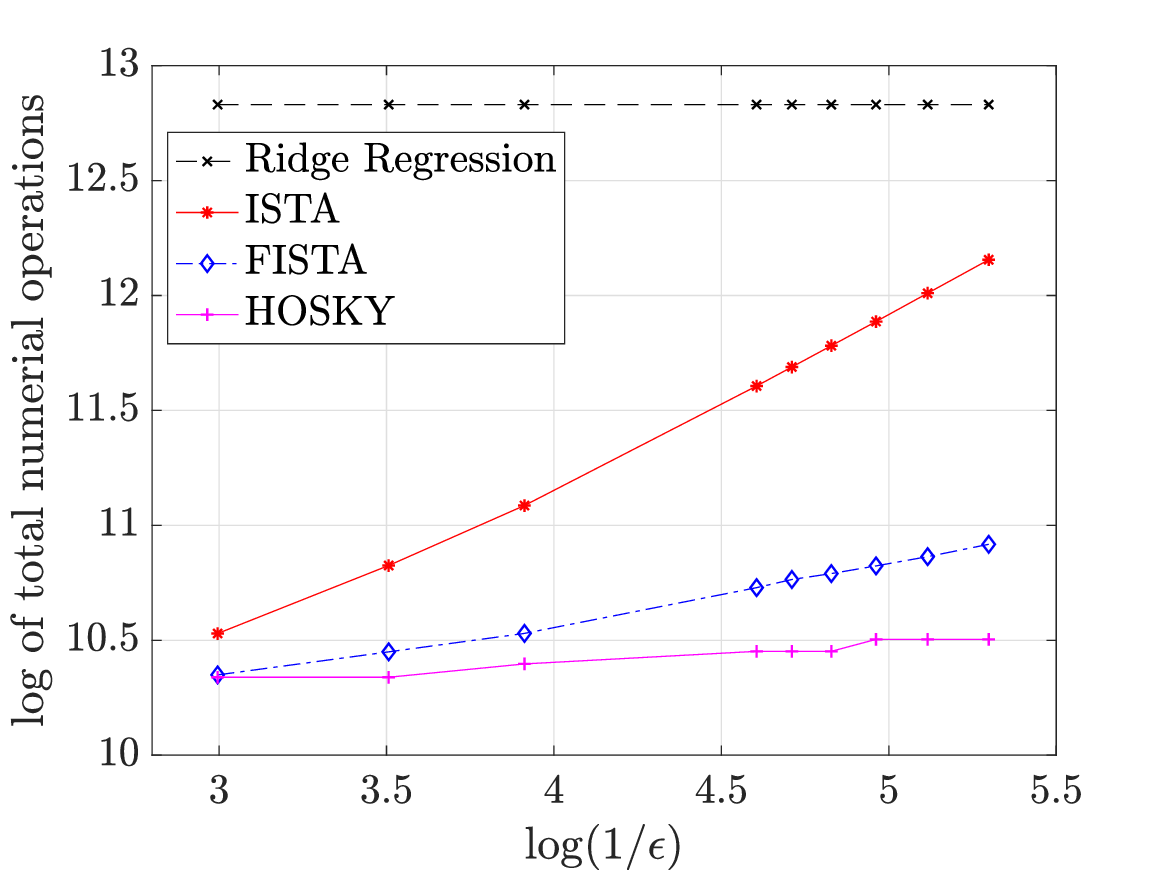}    \\
		(a) Scenario 1 ($n = 50, p = 20$) &
		(b) Scenario 1 ($n = 50, p = 80$) \\
		\includegraphics[width=0.45\textwidth]{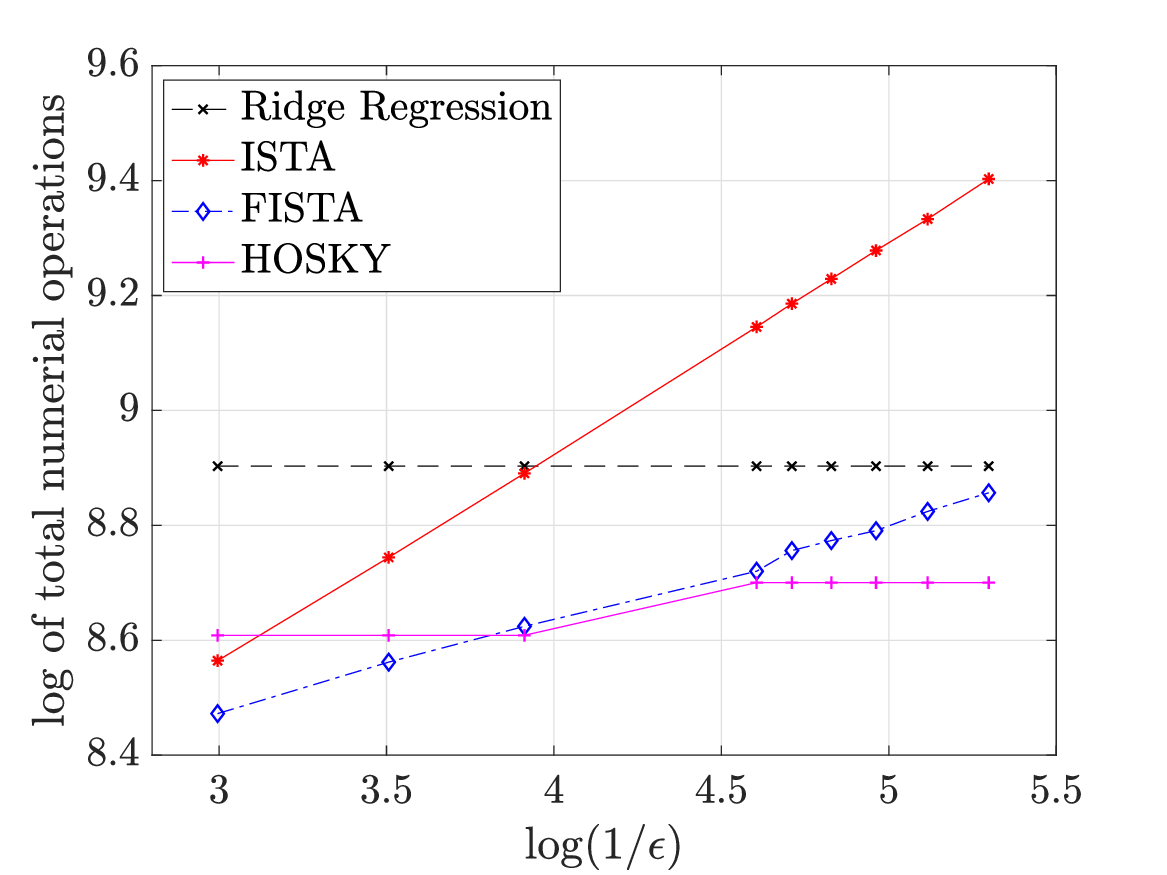}    &
		\includegraphics[width=0.45\textwidth]{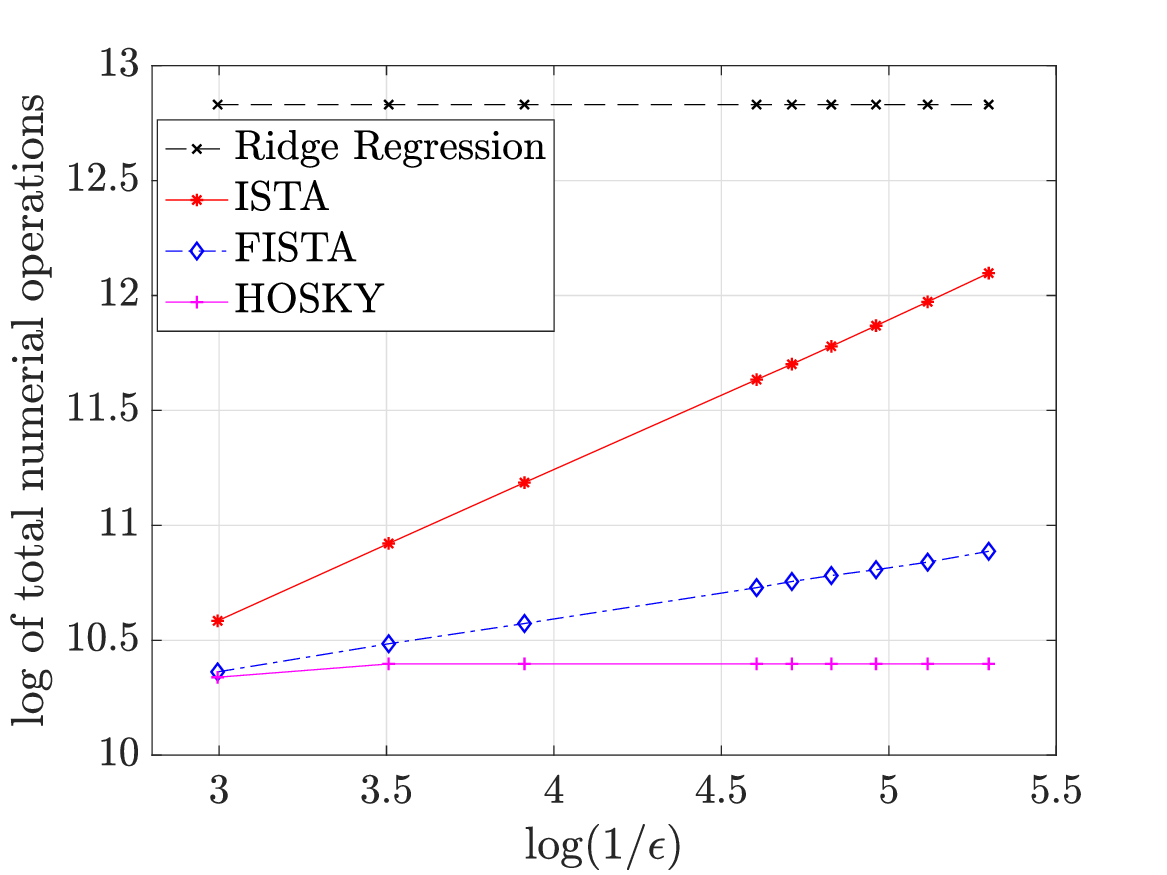} \\
		(c) Scenario 2 ($n = 50, p = 20$) &
		(d) Scenario 2 ($n = 50, p = 80$) \\
	\end{tabular}
	\caption{Number of Operations of ridge regression, ISTA, FISTA and HOSKY under different warm-up precision $\epsilon_w$
    \label{fig: simulation1 -- curves}}
  \end{center}
\end{figure}

\begin{figure}[htbp]
  \begin{center}
    \begin{tabular}{cc}
		\centering
		\includegraphics[width=0.45\textwidth]{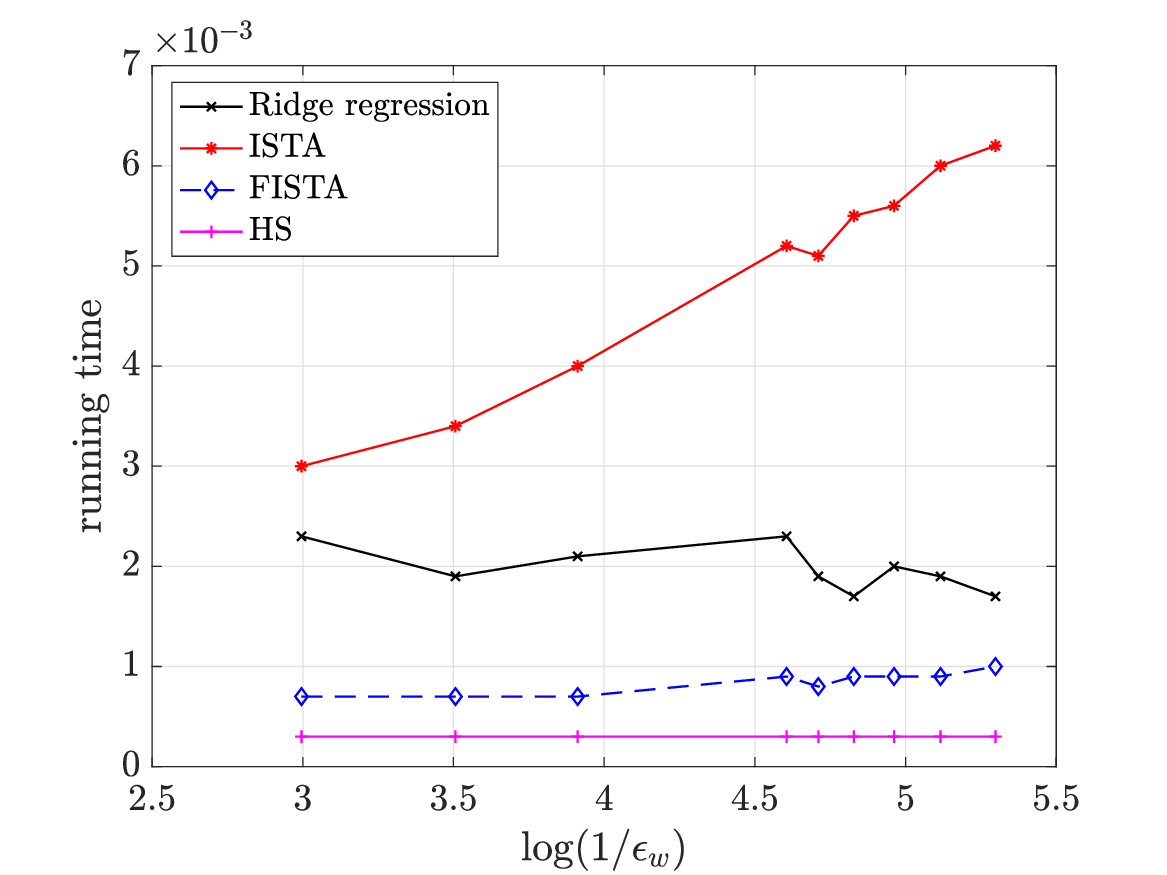}    &
		\includegraphics[width=0.45\textwidth]{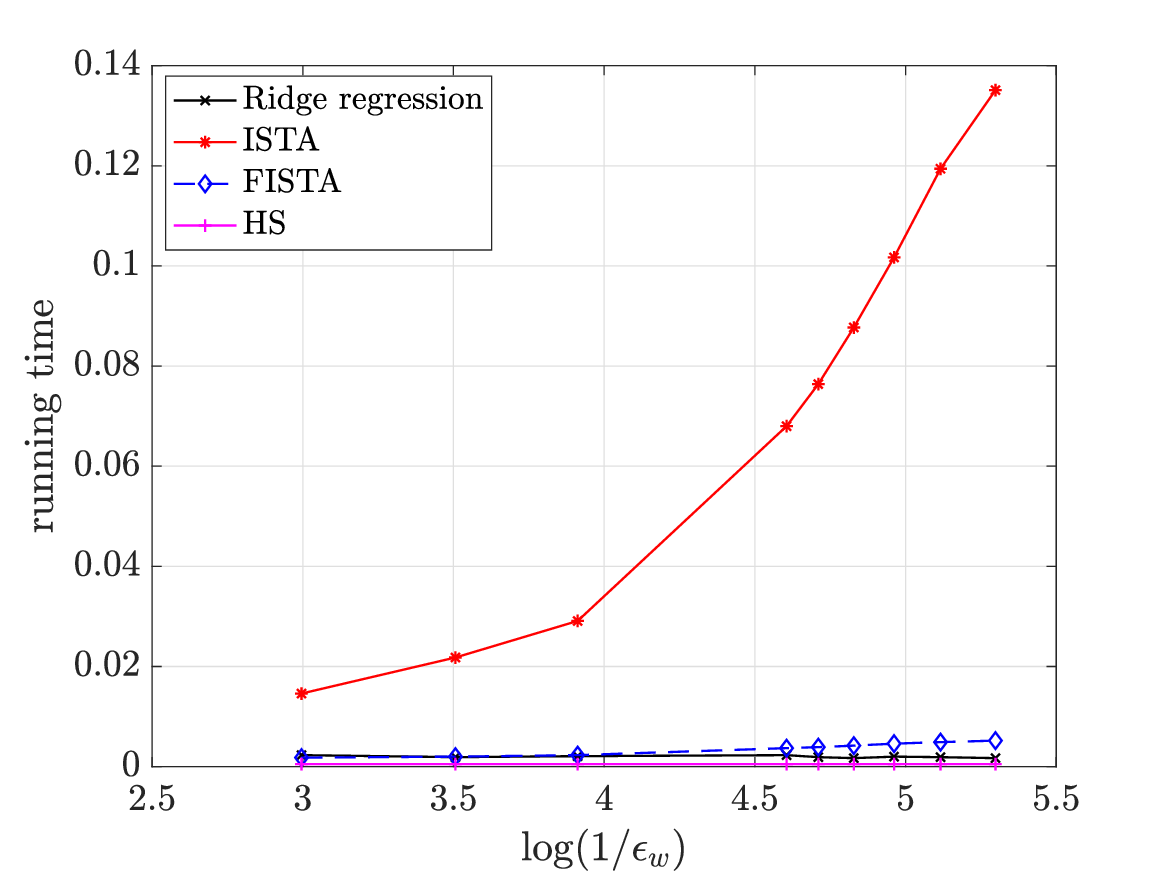}    \\
		(a) Scenario 1 ($n = 50, p = 20$) &
		(b) Scenario 1 ($n = 50, p = 80$) \\
		\includegraphics[width=0.45\textwidth]{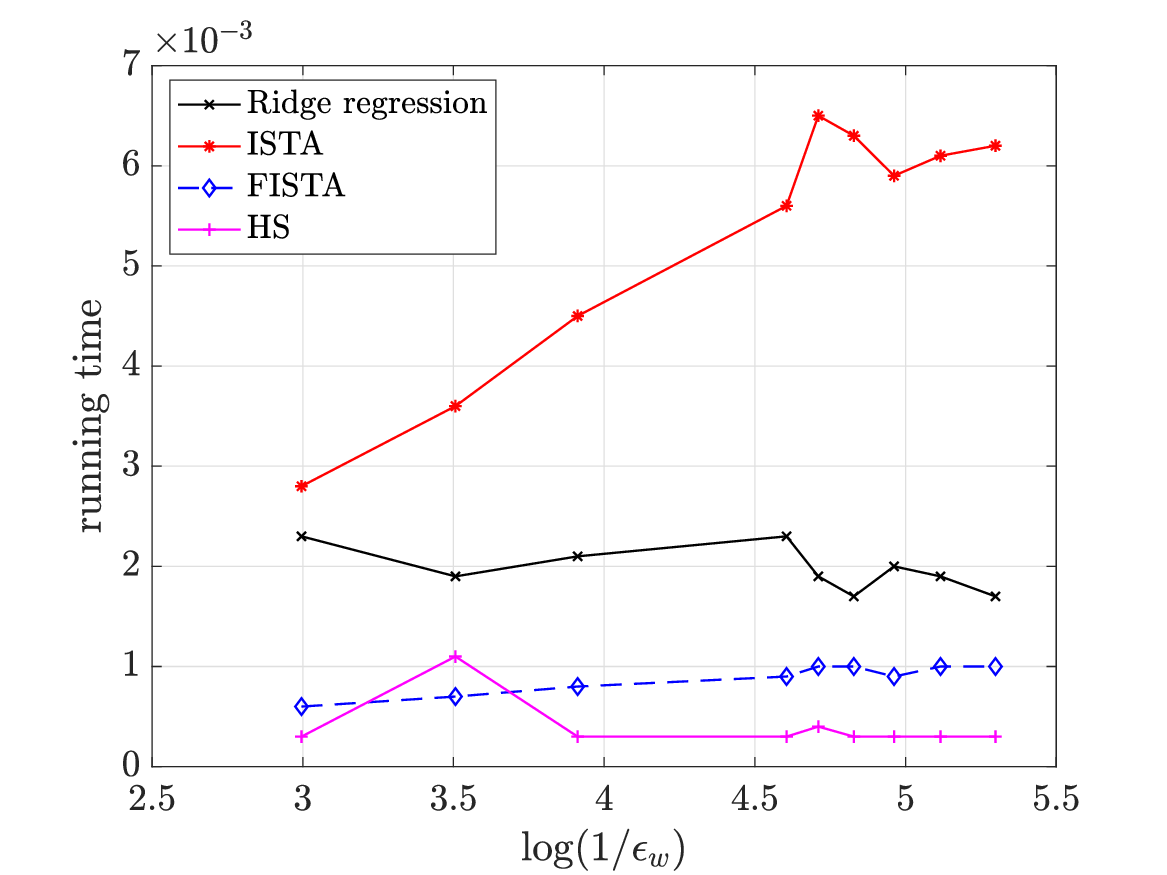}    &
		\includegraphics[width=0.45\textwidth]{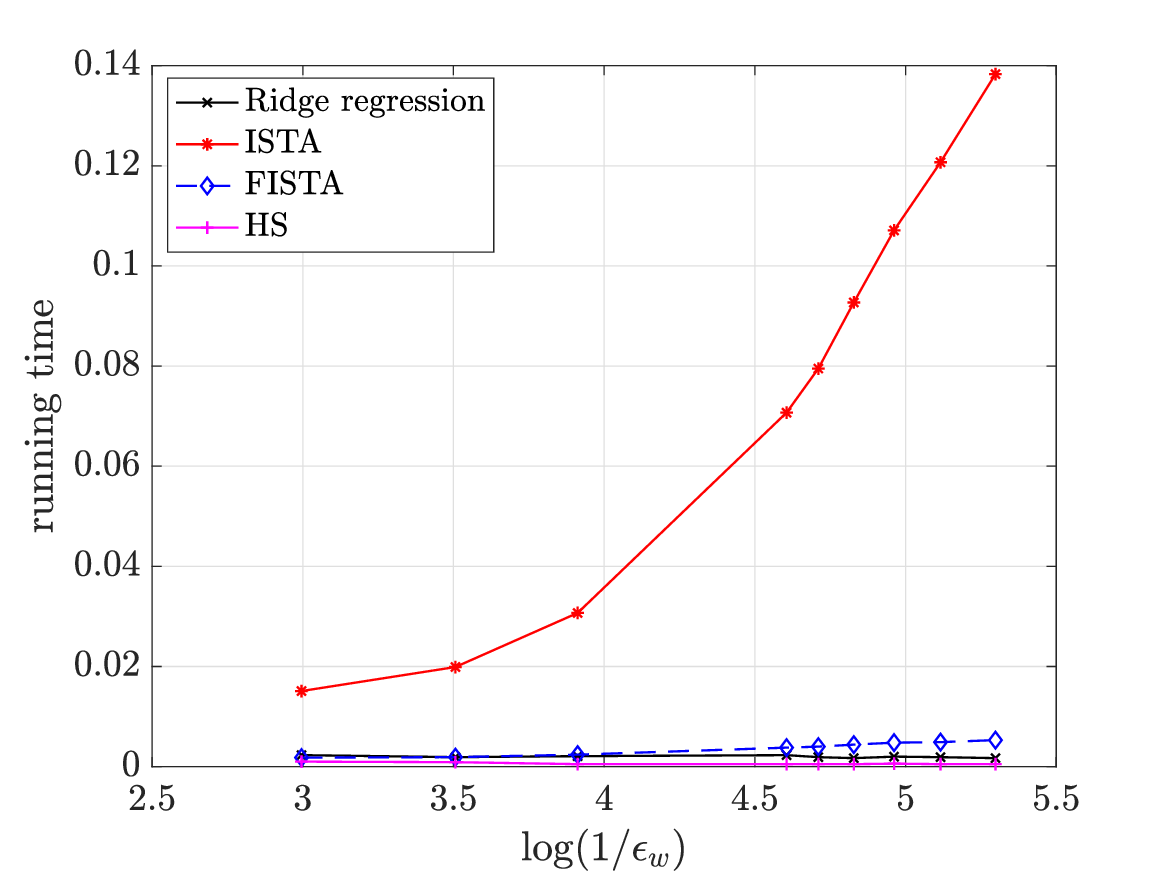} \\
		(c) Scenario 2 ($n = 50, p = 20$) &
		(d) Scenario 2 ($n = 50, p = 80$) \\
	\end{tabular}
	\caption{Running time of ridge regression, ISTA, FISTA and HOSKY under different warm-up precision $\epsilon_w$
    \label{fig: simulation1 -- runtime}}
  \end{center}
\end{figure}


\subsection{Simulation 3: Investigate if Good Warm-up Stages Expedite after-warm-up Stages}
\label{sec: simulation 3}

In this section, we investigate the contribution of the warm-up stage to convergence (the after-warm-up stage).  
Specifically, we will verify that the initial point assigned by HOSKY expedites the convergence in the after-warm-up stage. 
As indicated by both Simulation 1 and Simulation 2, among the benchmarks of HOSKY, FISTA performs the best. 
So in this section, we will use FISTA as a representative benchmark.

The simulation setting is the same as that we have in Section \ref{sec: simulation2}, and the comparison criterion follows Criterion \ref{criterion: warm-up contribution}.
In the warm-up stage, we fix the warm-up precision as $\epsilon_w = 0.05$ for both FISTA and HOSKY. 
In the after-warm-up stage, we run FISTA until the after-warm-up precision $\epsilon_{w+} = \{10^{-3}, 10^{-4}, 10^{-5}, 10^{-6}, 10^{-7}\}$ is achieved.  
For both the warm-up stage and after-warm-up stage, we will calculate the number of numerical operations, as well as the running time. 
To evaluate whether HOSKY expedites the after-warm-up stage, one can check the total number of numerical operations (warm-up + after-warm-up) to achieve the common after-warm-up precision $\epsilon_{w+}$. 
In addition to the total number of numerical operations, one can also compare the total running time.

The simulation results are summarized in Table \ref{table: simu -- warmup percent -- flops}, Table \ref{table: simu -- warmup percent -- running time} and visualized in Fig. \ref{fig: simulation3 -- flops}, Fig. \ref{fig: simulation3 -- runtime}. 
We report the number of numerical operations in Table \ref{table: simu -- warmup percent -- flops} and Fig. \ref{fig: simulation3 -- flops}, which is independent of the platforms. 
Additionally, we report the running time in Table \ref{table: simu -- warmup percent -- running time} and Fig. \ref{fig: simulation3 -- runtime}, which is based on a Macbook pro with 2.3 GHz Intel Core i5.

From the aforementioned two tables and two plots, one can draw two conclusions.
First, to achieve the common warm-up precision $\epsilon_w = 0.05$, HOSKY requires less number of operations (or running time), compared with FISTA. 
To verify this conclusion, one can take a close look at Fig. \ref{fig: simulation3 -- flops} and Fig. \ref{fig: simulation3 -- runtime}.
In these two plots, the pink solid line (HOSKY in the warm-up stage) is always below the blue solid line (FISTA in the warm-up stage).
This conclusion is consistent with the conclusion we have in Section \ref{sec: simulation2}.
Second, the initial points from HOSKY can expedite the calculation in the after-warm-up stage. 
For example, to achieve the common after-warm-up precision $\epsilon_{w+} = 10^{-5}$, if one uses the initial point from HOSKY, one only needs $3.0919 \times 10^4$ numerical operations (0.0006 seconds), while FISTA needs $4.5152 \times 10^4$ numerical operations (0.0015 seconds) in Scenario 1. 
The same conclusion can also be drawn from Fig. \ref{fig: simulation3 -- flops} and Fig. \ref{fig: simulation3 -- runtime}.
In these two plots, the pink dash line (HOSKY in the after-warm-up stage) is always below the blue dash line (FISTA in the after-warm-up stage), for both the number of numerical operations and running time.

\begin{table}[htbp]
\centering
\caption{ 
\label{table: simu -- warmup percent -- flops}  Numerical operations taken in the warm-up stage and after-warm-up stage}
\footnotesize
\begin{adjustbox}{max width=0.95\textwidth}
  \begin{threeparttable}
	\begin{tabular}{ccccccccccccccc}
		\hline 
		& &
        \multicolumn{5}{c}{After-warm-up precision $\epsilon_{w+}$} \\
		Method & Stage &
        $10^{-3}$ & $10^{-4}$ & $10^{-5}$ & $10^{-6}$ & $10^{-7}$ \\
		\cline{1-7}
		& & \multicolumn{5}{c}{ Scenario 1 $(n=50, \; p=20)$} \\
		\cline{2-7}
		FISTA & 
		Warm-up &
		$ 6.3503 \times 10^3 $  &
		$ 6.3503 \times 10^3 $  &
		$ 6.3503 \times 10^3 $  &
		$ 6.3503 \times 10^3 $  &
		$ 6.3503 \times 10^3 $  \\
		& 
		After-warm-up &
		\bm{$ 8.3992 \times 10^3 $}  &
		\bm{$ 1.3525 \times 10^4 $}  &
		$ 2.1710 \times 10^4 $  &
		$ 3.4359 \times 10^4 $  &
		$ 5.2036 \times 10^4 $  \\
		& 
		Total &
		$ 1.4750 \times 10^4 $  &
		$ 1.9875 \times 10^4 $  &
		$ 2.8061 \times 10^4 $  &
		$ 4.0709 \times 10^4 $  &
		$ 5.8386 \times 10^4 $  \\
		& 
		\% of warm-up &
		$ 43.05 $  &
		$ 31.95 $  &
		$ 22.63 $  &
		$ 15.60 $  &
		$ 10.88 $  \\
		HOSKY & 
		Warm-up &
		\bm{$ 5.4790 \times 10^3$}  &
		\bm{$ 5.4790 \times 10^3$}  &
		\bm{$ 5.4790 \times 10^3$}  &
		\bm{$ 5.4790 \times 10^3$}  &
		\bm{$ 5.4790 \times 10^3$}  \\
		& 
		After-warm-up &
		$ 8.4026 \times 10^3 $  &
		$ 1.3608 \times 10^4 $  &
		\bm{$ 2.1418 \times 10^4 $}  &
		\bm{$ 3.3097 \times 10^4 $}  &
		\bm{$ 4.9757 \times 10^4 $}  \\
		& 
		Total &
		\bm{$ 1.3882 \times 10^4 $}  &
		\bm{$ 1.9087 \times 10^4 $}  &
		\bm{$ 2.6897 \times 10^4 $}  &
		\bm{$ 3.8576 \times 10^4 $}  &
		\bm{$ 5.5236 \times 10^4 $}  \\
		& 
		\% of warm-up &
		\bm{$ 39.47 $}  &
		\bm{$ 28.71 $}  &
		\bm{$ 20.37 $}  &
		\bm{$ 14.20 $}  &
		\bm{$ 9.92 $}  \\
		\cline{2-7}
		& &\multicolumn{5}{c}{ Scenario 1 $(n=50, \; p=80)$} \\
		\cline{2-7}
		FISTA & 
		Warm-up &
		$ 4.5152 \times 10^4 $  &
		$ 4.5152 \times 10^4 $  &
		$ 4.5152 \times 10^4 $  &
		$ 4.5152 \times 10^4 $  &
		$ 4.5152 \times 10^4 $  \\
		&
		After-warm-up &
		$ 1.7418 \times 10^5 $  &
		$ 3.0569 \times 10^5 $  &
		$ 4.8379 \times 10^5 $  &
		$ 8.0920 \times 10^5 $  &
		$ 1.3803 \times 10^6 $  \\
		& 
		Total &
		$ 2.1933 \times 10^5 $  &
		$ 3.5084 \times 10^5 $  &
		$ 5.2894 \times 10^5 $  &
		$ 8.5435 \times 10^5 $  &
		$ 1.4254 \times 10^6 $  \\
		& 
		\% of warm-up &
		$ 20.59 $  &
		$ 12.87 $  &
		$ 8.54 $  &
		$ 5.28 $  &
		$ 3.17 $  \\
		HOSKY & 
		Warm-up &
		\bm{$ 3.0919 \times 10^4 $}  &
		\bm{$ 3.0919 \times 10^4 $}  &
		\bm{$ 3.0919 \times 10^4 $}  &
		\bm{$ 3.0919 \times 10^4 $}  &
		\bm{$ 3.0919 \times 10^4 $}  \\
		& 
		After-warm-up &
		\bm{$ 9.9122 \times 10^4 $}  &
		\bm{$ 2.0242 \times 10^5 $}  &
		\bm{$ 3.4636 \times 10^5 $}  &
		\bm{$ 6.0534 \times 10^5 $}  &
		\bm{$ 1.0724 \times 10^6 $}  \\
		& 
		Total &
		\bm{$ 1.3004 \times 10^5 $}  &
		\bm{$ 2.3333 \times 10^5 $}  &
		\bm{$ 3.7728 \times 10^5 $}  &
		\bm{$ 6.3626 \times 10^5 $}  &
		\bm{$ 1.1033 \times 10^6 $}  \\
		& 
		\% of warm-up &
		\bm{$ 23.78 $}  &
		\bm{$ 13.25 $}  &
		\bm{$ 8.20 $}  &
		\bm{$ 4.86 $}  &
		\bm{$ 2.80 $}  \\
        \cline{1-7}
		& &\multicolumn{5}{c}{ Scenario 2 $(n=50, \; p=20)$} \\
		\cline{2-7}
		FISTA & 
		Warm-up &
		$ 6.3452 \times 10^3 $  &
		$ 6.3452 \times 10^3 $  &
		$ 6.3452 \times 10^3 $  &
		$ 6.3452 \times 10^3 $  &
		$ 6.3452 \times 10^3 $  \\
		& 
		After-warm-up &
		$ 8.3681 \times 10^3 $  &
		$ 1.3531 \times 10^4 $  &
		$ 2.1611 \times 10^4 $  &
		$ 3.3959 \times 10^4 $  &
		$ 5.1187 \times 10^4 $  \\
		& 
		Total &
		$ 1.4713 \times 10^4 $  &
		$ 1.9877 \times 10^4 $  &
		$ 2.7956 \times 10^4 $  &
		$ 4.0305 \times 10^4 $  &
		$ 5.7532 \times 10^4 $  \\
		& 
		\% of warm-up &
		$ 43.13 $  &
		$ 31.92 $  &
		$ 22.70 $  &
		$ 15.74 $  &
		$ 11.03 $  \\
		HOSKY & 
		Warm-up &
		\bm{$ 5.479 \times 10^3$}  &
		\bm{$ 5.479 \times 10^3$}  &
		\bm{$ 5.479 \times 10^3$}  &
		\bm{$ 5.479 \times 10^3$}  &
		\bm{$ 5.479 \times 10^3$}  \\
		& 
		After-warm-up &
		\bm{$ 8.2111 \times 10^3 $}  &
		\bm{$ 1.3256 \times 10^4 $}  &
		\bm{$ 2.0681 \times 10^4 $}  &
		\bm{$ 3.2014 \times 10^4 $}  &
		\bm{$ 4.8112 \times 10^4 $}  \\
		& 
		Total &
		\bm{$ 1.3690 \times 10^4 $}  &
		\bm{$ 1.8735 \times 10^4 $}  &
		\bm{$ 2.6160 \times 10^4 $}  &
		\bm{$ 3.7493 \times 10^4 $}  &
		\bm{$ 5.3591 \times 10^4 $}  \\
		& 
		\% of warm-up &
		\bm{$ 40.02 $}  &
		\bm{$ 29.24 $}  &
		\bm{$ 20.94 $}  &
		\bm{$ 14.61 $}  &
		\bm{$ 10.22 $}  \\
		\cline{2-7}
		& \multicolumn{6}{c}{ Scenario 2 $(n=50, \; p=80)$} \\
		\cline{2-7}
		FISTA & 
		Warm-up &
		$ 4.4986 \times 10^4 $  &
		$ 4.4986 \times 10^4 $  &
		$ 4.4986 \times 10^4 $  &
		$ 4.4986 \times 10^4 $  &
		$ 4.4986 \times 10^4 $  \\
		& 
		After-warm-up &
		$ 1.7068 \times 10^5 $  &
		$ 2.9763 \times 10^5 $  &
		$ 4.6303 \times 10^5 $  &
		$ 7.6909 \times 10^5 $  &
		$ 1.3041 \times 10^6 $  \\
		& 
		Total &
		$ 2.1567 \times 10^5 $  &
		$ 3.4262 \times 10^5 $  &
		$ 5.0802 \times 10^5 $  &
		$ 8.1407 \times 10^5 $  &
		$ 1.3491 \times 10^6 $  \\
		& 
		\% of warm-up &
		$ 20.86 $  &
		\bm{$ 13.13$}  &
		\bm{$ 8.86 $}  &
		$ 5.53 $  &
		$ 3.33 $  \\
		HOSKY & 
		Warm-up &
		\bm{$ 3.0919 \times 10^4 $}  &
		\bm{$ 3.0919 \times 10^4 $}  &
		\bm{$ 3.0919 \times 10^4 $}  &
		\bm{$ 3.0919 \times 10^4 $}  &
		\bm{$ 3.0919 \times 10^4 $}  \\
		& 
		After-warm-up &
		\bm{$ 9.3406 \times 10^4 $}  &
		\bm{$ 1.8609 \times 10^5 $}  &
		\bm{$ 3.1290 \times 10^5 $}  &
		\bm{$ 5.4918 \times 10^5 $}  &
		\bm{$ 9.6698 \times 10^5 $}  \\
		& 
		Total &
		\bm{$ 1.2432 \times 10^5 $}  &
		\bm{$ 2.1700 \times 10^5 $}  &
		\bm{$ 3.4382 \times 10^5 $}  &
		\bm{$ 5.8010 \times 10^5 $}  &
		\bm{$ 9.9790 \times 10^5 $}  \\
		& 
		\% of warm-up &
		\bm{$ 24.87 $}  &
		$ 14.25 $  &
		$ 8.99 $  &
		\bm{$ 5.33 $}  &
		\bm{$ 3.10 $}  \\
        \cline{1-7}
	\end{tabular}
   \begin{tablenotes}
     \item[1] There is the parameters settings of our HOSKY algorithm: $t_0=3, h=0.1, \beta^{(0)} = \mathbf 0_{p \times 1}$.
   \end{tablenotes}
  \end{threeparttable}
\end{adjustbox}

\end{table}

\begin{table}[htbp]
\centering
\caption{ 
\label{table: simu -- warmup percent -- running time}  Running time taken in the warm-up stage and after-warm-up stage}
\begin{adjustbox}{max width=0.95\textwidth}
  \begin{threeparttable}
	\begin{tabular}{ccccccccccccccc}
		\hline 
		&  &
        \multicolumn{5}{c}{After-warm-up precision $\epsilon_{w+}$} \\
		Method & Stage &
        $10^{-3}$ & $10^{-4}$ & $10^{-5}$ & $10^{-6}$ & $10^{-7}$ \\
		\cline{1-7}
		& &\multicolumn{5}{c}{ Scenario 1 $(n=50, \; p=20)$} \\
		\cline{2-7}
		FISTA &
		Warm-up &
		$ 0.0006 $  & 
		$ 0.0006 $  & 
		$ 0.0007 $  & 
		$ 0.0007 $  & 
		$ 0.0007 $  \\ 
		& 
		After-warm-up &
		$ 0.0010 $  &
		$ 0.0016 $  &
		$ 0.0031 $  &
		$ 0.0053 $  &
		$ 0.0077 $  \\
		& 
		Total &
		$ 0.0017 $  &
		$ 0.0022 $  &
		$ 0.0037 $  &
		$ 0.0060 $  &
		$ 0.0084 $  \\
		& 
		\% of warm-up &
		$ 38.88 $  &
		$ 27.46 $  &
		$ 17.66 $  &
		$ 12.29 $  &
		$ 8.16 $  \\
		HOSKY & 
		Warm-up &
		\bm{$ 0.0005 $}  & 
		\bm{$ 0.0005 $}  & 
		\bm{$ 0.0005 $}  & 
		\bm{$ 0.0005 $}  & 
		\bm{$ 0.0005 $}  \\ 
		& 
		After-warm-up &
		\bm{$ 0.0009 $}  &
		\bm{$ 0.0016 $}  &
		\bm{$ 0.0029 $}  &
		\bm{$ 0.0051 $}  &
		\bm{$ 0.0073 $}  \\ 
		& 
		Total &
		\bm{$ 0.0014 $}  &
		\bm{$ 0.0020 $}  &
		\bm{$ 0.0034 $}  &
		\bm{$ 0.0057 $}  &
		\bm{$ 0.0078 $}  \\
		& 
		\% of warm-up &
		\bm{$ 35.57 $}  &
		\bm{$ 22.68 $}  &
		\bm{$ 14.90 $}  &
		\bm{$ 9.62 $}  &
		\bm{$ 6.61 $}  \\
		\cline{2-7}
		& &\multicolumn{5}{c}{ Scenario 1 $(n=50, \; p=80)$} \\
		\cline{2-7}
		Pre-assigned value &
		Warm-up &
		$ 0.0016 $  &
		$ 0.0015 $  &
		$ 0.0015 $  &
		$ 0.0016 $  &
		$ 0.0015 $  \\
		& 
		After-warm-up &
		$ 0.0094 $  &
		$ 0.0159 $  &
		$ 0.0257 $  &
		$ 0.0432 $  &
		$ 0.0737 $  \\
		& 
		Total &
		$ 0.0110 $  &
		$ 0.0174 $  &
		$ 0.0272 $  &
		$ 0.0447 $  &
		$ 0.0752 $  \\
		& 
		\% of warm-up &
		$ 14.71 $  &
		$ 8.53 $  &
		$ 5.45 $  &
		$ 3.47 $  &
		$ 2.03 $  \\
		HOSKY & 
		Warm-up &
		\bm{$ 0.0007 $}  & 
		\bm{$ 0.0006 $}  & 
		\bm{$ 0.0006 $}  & 
		\bm{$ 0.0006 $}  & 
		\bm{$ 0.0006 $} \\ 
		& 
		After-warm-up &
		\bm{$ 0.0047 $}  &
		\bm{$ 0.0099 $}  &
		\bm{$ 0.0177 $}  &
		\bm{$ 0.0315 $}  &
		\bm{$ 0.0565 $}  \\
		& 
		Total &
		\bm{$ 0.0054 $}  &
		\bm{$ 0.0105 $}  &
		\bm{$ 0.0183 $}  &
		\bm{$ 0.0321 $}  &
		\bm{$ 0.0572 $}  \\
		& 
		\% of warm-up &
		\bm{$ 13.40 $}  &
		\bm{$ 5.98 $}  &
		\bm{$ 3.44 $}  &
		\bm{$ 2.03 $}  &
		\bm{$ 1.15 $}  \\
        \cline{1-7}
		& &\multicolumn{5}{c}{ Scenario 2 $(n=50, \; p=20)$} \\
		\cline{2-7}
		FISTA & 
		Warm-up &
		$ 0.0006 $  & 
		$ 0.0007 $  & 
		$ 0.0006 $  & 
		$ 0.0006 $  & 
		$ 0.0006 $  \\ 
		& 
		After-warm-up &
		$ 0.0009 $  & 
		$ 0.0021 $  &
		$ 0.0028 $  &
		$ 0.0043 $  &
		$ 0.0066 $  \\
		&
		Total &
		$ 0.0015 $  &
		$ 0.0028 $  &
		$ 0.0034 $  &
		$ 0.0049 $  &
		$ 0.0072 $  \\
		& 
		\% of warm-up &
		$ 40.86 $  &
		$ 24.72 $  &
		$ 18.37 $  &
		$ 12.46 $  &
		$ 8.61 $  \\
		HOSKY & 
		Warm-up &
		\bm{$ 0.0005 $}  & 
		\bm{$ 0.0005 $}  & 
		\bm{$ 0.0005 $}  & 
		\bm{$ 0.0005 $}  & 
		\bm{$ 0.0005 $}  \\ 
		& 
		After-warm-up &
		\bm{$ 0.0008 $}  & 
		\bm{$ 0.0017 $}  &
		\bm{$ 0.0025 $}  &
		\bm{$ 0.0040 $}  &
		\bm{$ 0.0064 $}  \\
		& 
		Total &
		\bm{$ 0.0013 $}  &
		\bm{$ 0.0022 $}  &
		\bm{$ 0.0030 $}  &
		\bm{$ 0.0044 $}  &
		\bm{$ 0.0068 $}  \\
		& 
		\% of warm-up &
		\bm{$ 35.95 $}  &
		\bm{$ 23.94 $}  &
		\bm{$ 16.09 $}  &
		\bm{$ 10.29 $}  &
		\bm{$ 6.94 $}  \\
		\cline{1-7}
		& &\multicolumn{5}{c}{ Scenario 2 $(n=50, \; p=80)$} \\
		\cline{2-7}
		FISTA & 
		Warm-up &
		$ 0.0015 $  &
		$ 0.0015 $  &
		$ 0.0015 $  &
		$ 0.0016 $  &
		$ 0.0018 $  \\
		&
		After-warm-up &
		$ 0.0085 $  &
		$ 0.0158 $  &
		$ 0.0243 $  &
		$ 0.0459 $  &
		$ 0.0830 $  \\
		&
		Total &
		$ 0.0100 $  &
		$ 0.0173 $  &
		$ 0.0257 $  &
		$ 0.0475 $  &
		$ 0.0848 $  \\
		&
		\% of warm-up &
		$ 14.80 $  &
		$ 8.52 $  &
		$ 5.72 $  &
		$ 3.45 $  &
		$ 2.11 $  \\
		HOSKY & 
		Warm-up &
		\bm{$ 0.0006 $}  & 
		\bm{$ 0.0007 $}  & 
		\bm{$ 0.0006 $}  & 
		\bm{$ 0.0007 $}  & 
		\bm{$ 0.0008 $}  \\ 
		& 
		After-warm-up &
		\bm{$ 0.0041 $}  &
		\bm{$ 0.0091 $}  &
		\bm{$ 0.0155 $}  &
		\bm{$ 0.0317 $}  &
		\bm{$ 0.0614 $}  \\
		& 
		Total &
		\bm{$ 0.0047 $}  &
		\bm{$ 0.0097 $}  &
		\bm{$ 0.0161 $}  &
		\bm{$ 0.0324 $}  &
		\bm{$ 0.0623 $}  \\
		&
		\% of warm-up &
		\bm{$ 13.26 $}  &
		\bm{$ 7.05 $}  &
		\bm{$ 3.92 $}  &
		\bm{$ 2.24 $}  &
		\bm{$ 1.36 $}  \\
        \cline{1-7}
	\end{tabular}
   \begin{tablenotes}
     \item[1] There is the parameters settings of our HOSKY algorithm: $t_0=3, h=0.1, \beta^{(0)} = \mathbf 0_{p \times 1}$.
     \item[2] The running time is based on a Macbook pro with 2.3 GHz Intel Core i5.
   \end{tablenotes}
  \end{threeparttable}
\end{adjustbox}
\end{table}

\begin{figure}[htbp]
  \begin{center}
    \begin{tabular}{cc}
		\centering
		\includegraphics[width=0.45\textwidth]{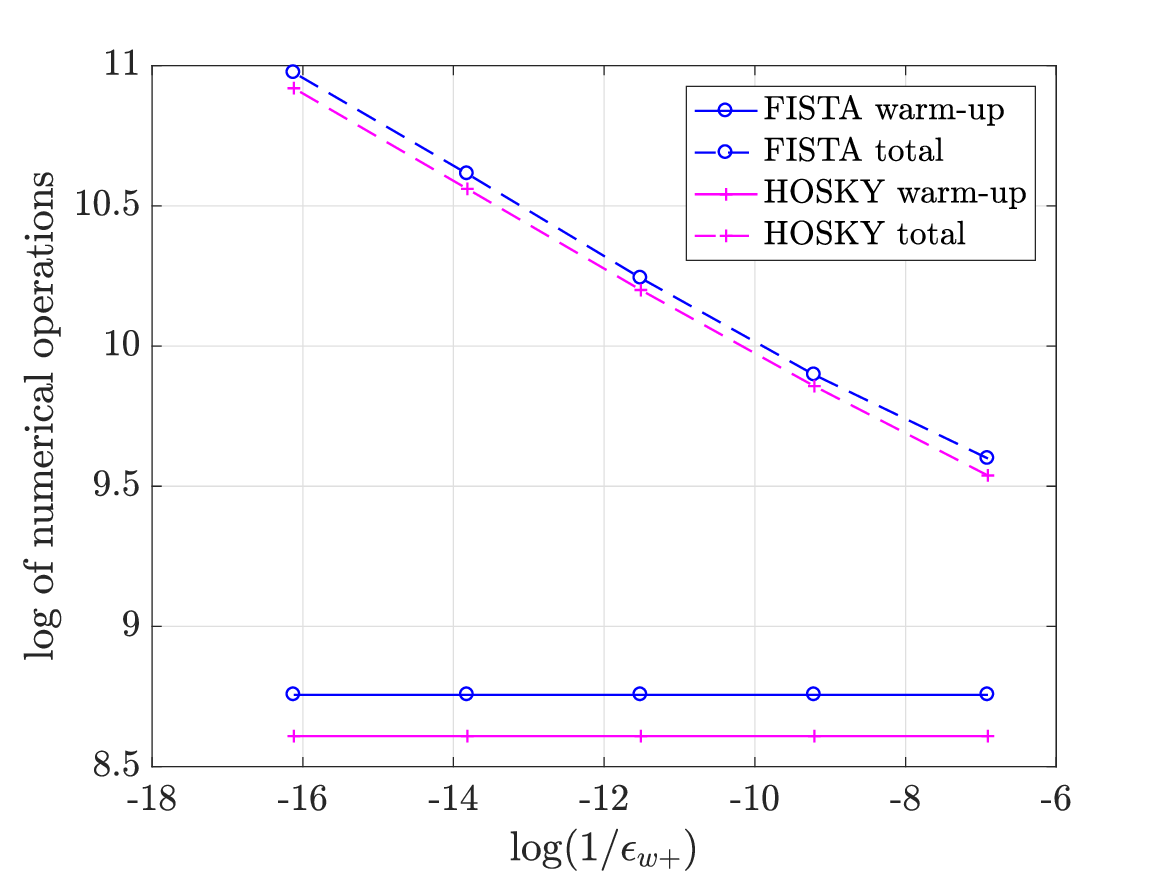}    &
		\includegraphics[width=0.45\textwidth]{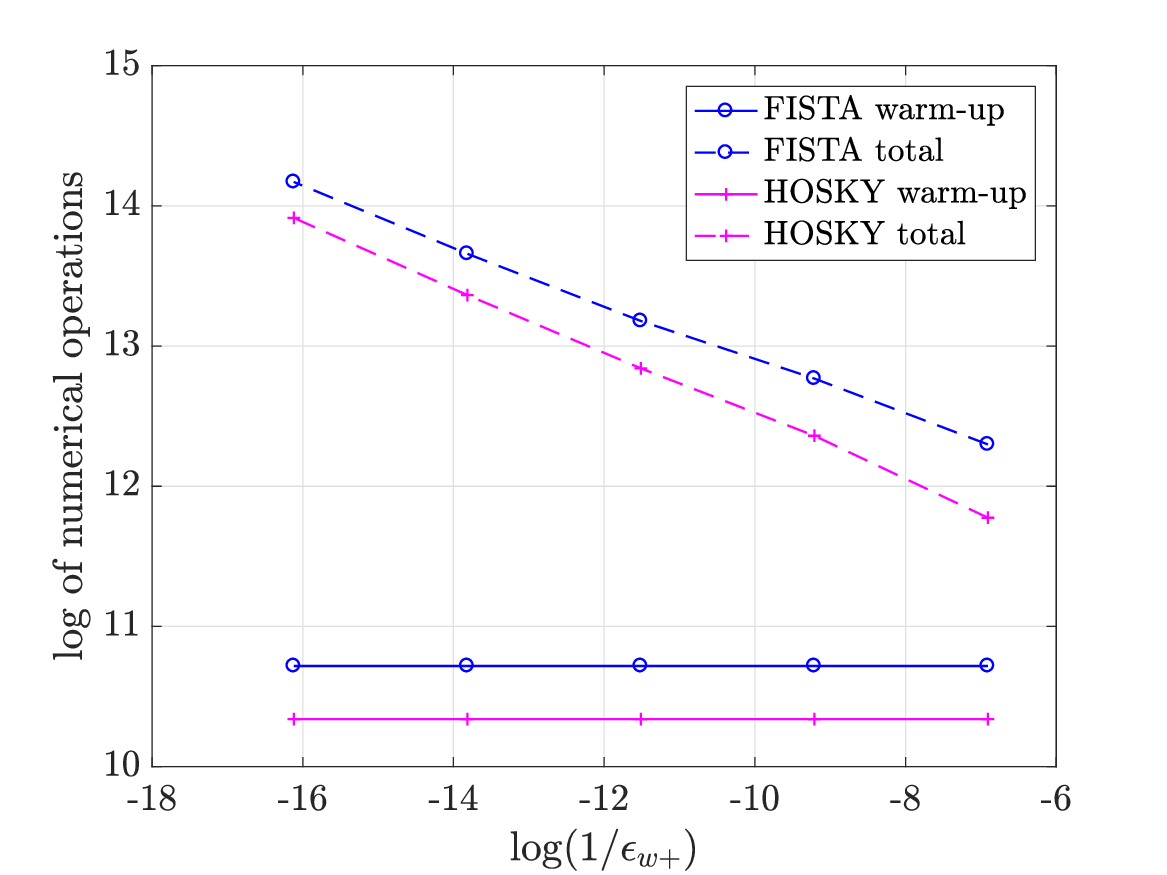}    \\
		(a) Scenario 1 ($n = 50, p = 20$) &
		(b) Scenario 1 ($n = 50, p = 80$) \\
		\includegraphics[width=0.45\textwidth]{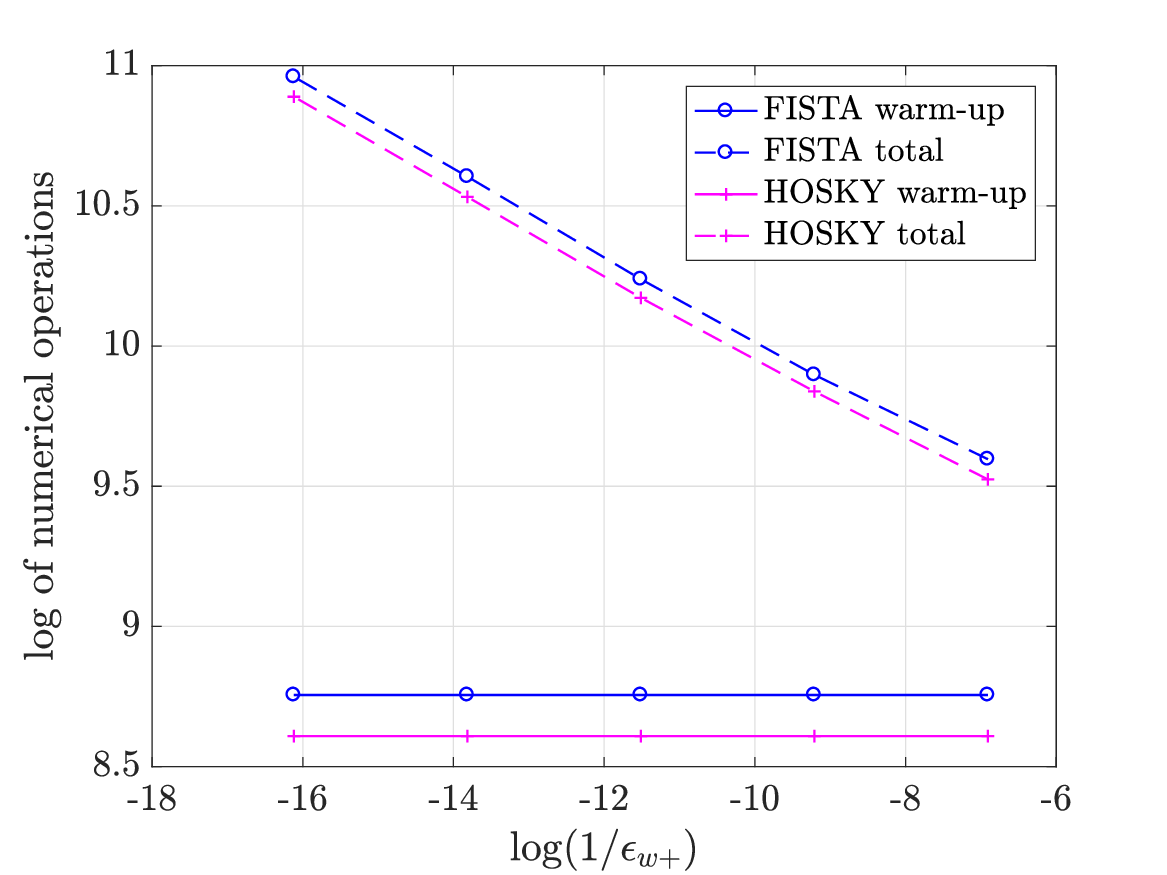}    &
		\includegraphics[width=0.45\textwidth]{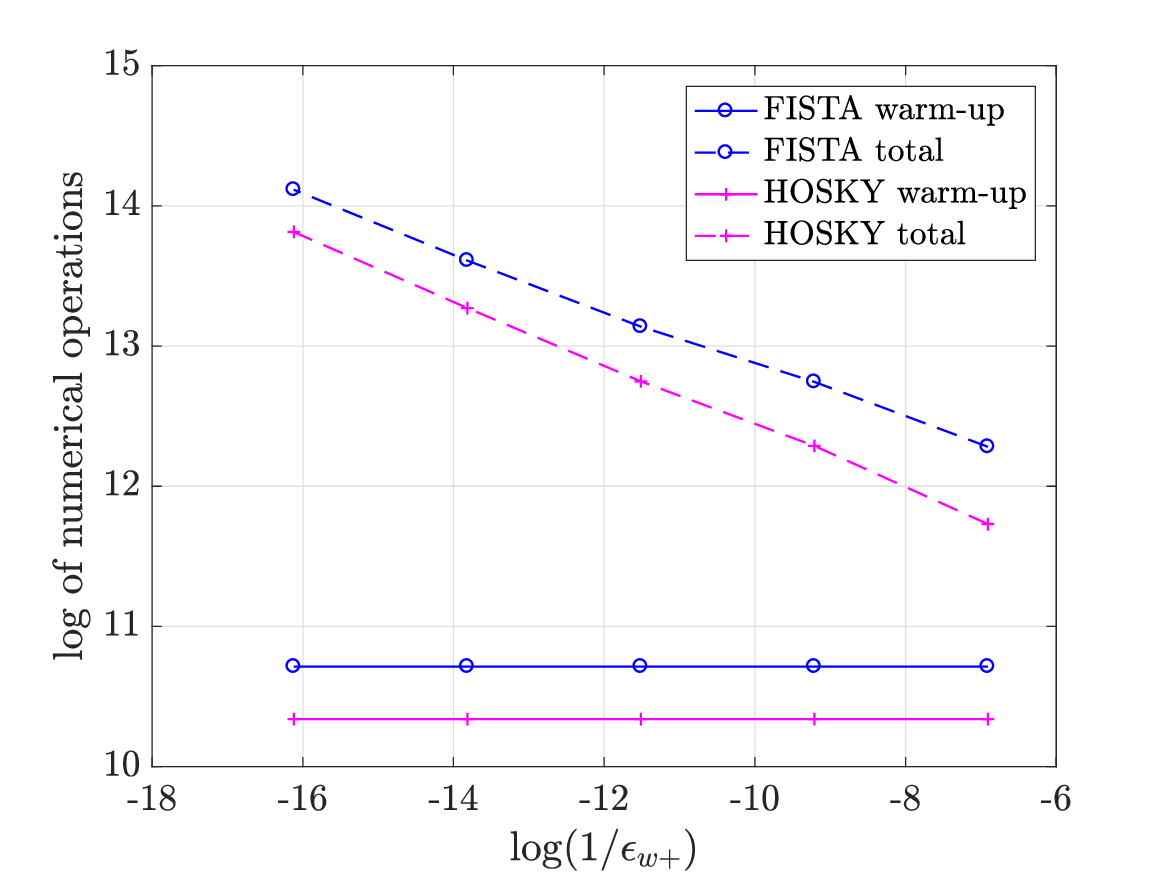} \\
		(c) Scenario 2 ($n = 50, p = 20$) &
		(d) Scenario 2 ($n = 50, p = 80$) \\
	\end{tabular}
	\caption{Number of Operations of FISTA and HOSKY in the warm-up stage and after-warm-up stage under different after-warm-up precision $\epsilon_{w+}$
    \label{fig: simulation3 -- flops}}
  \end{center}
\end{figure}

\begin{figure}[htbp]
  \begin{center}
    \begin{tabular}{cc}
		\centering
		\includegraphics[width=0.45\textwidth]{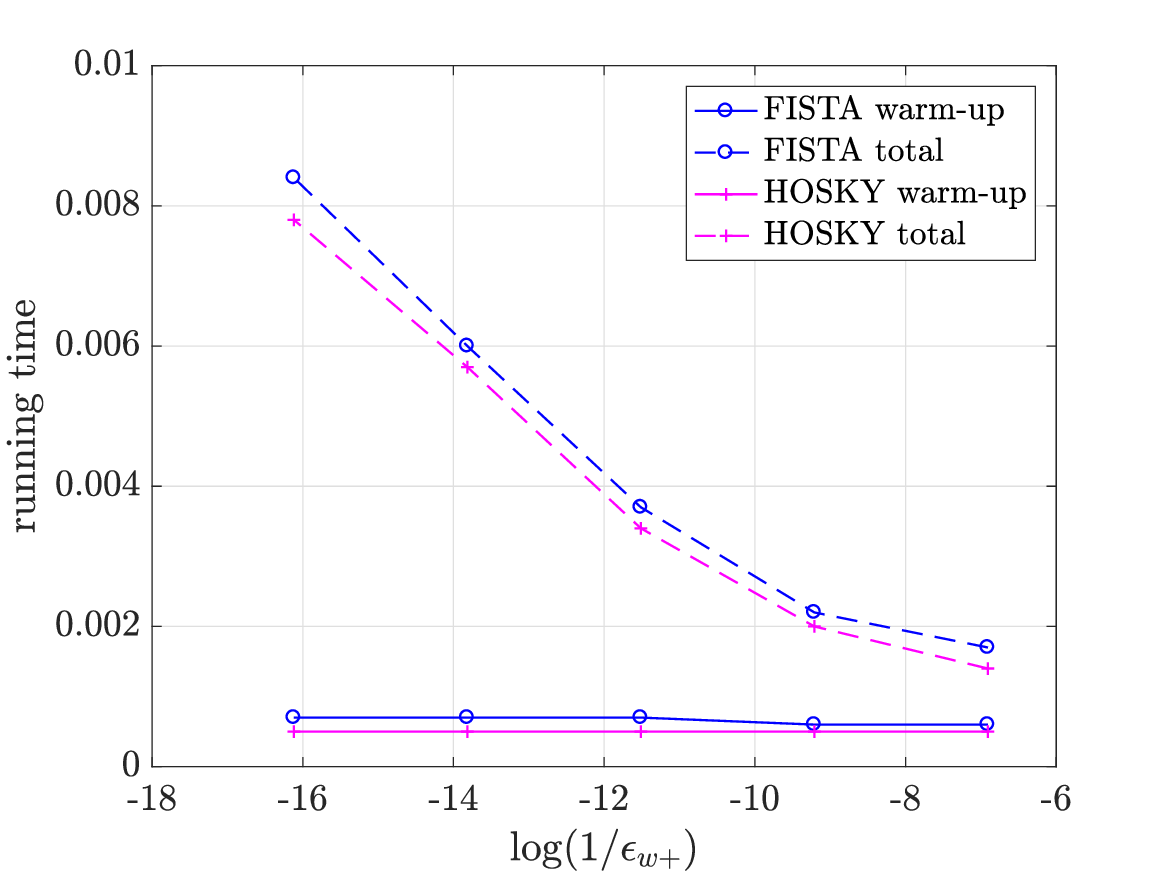}    &
		\includegraphics[width=0.45\textwidth]{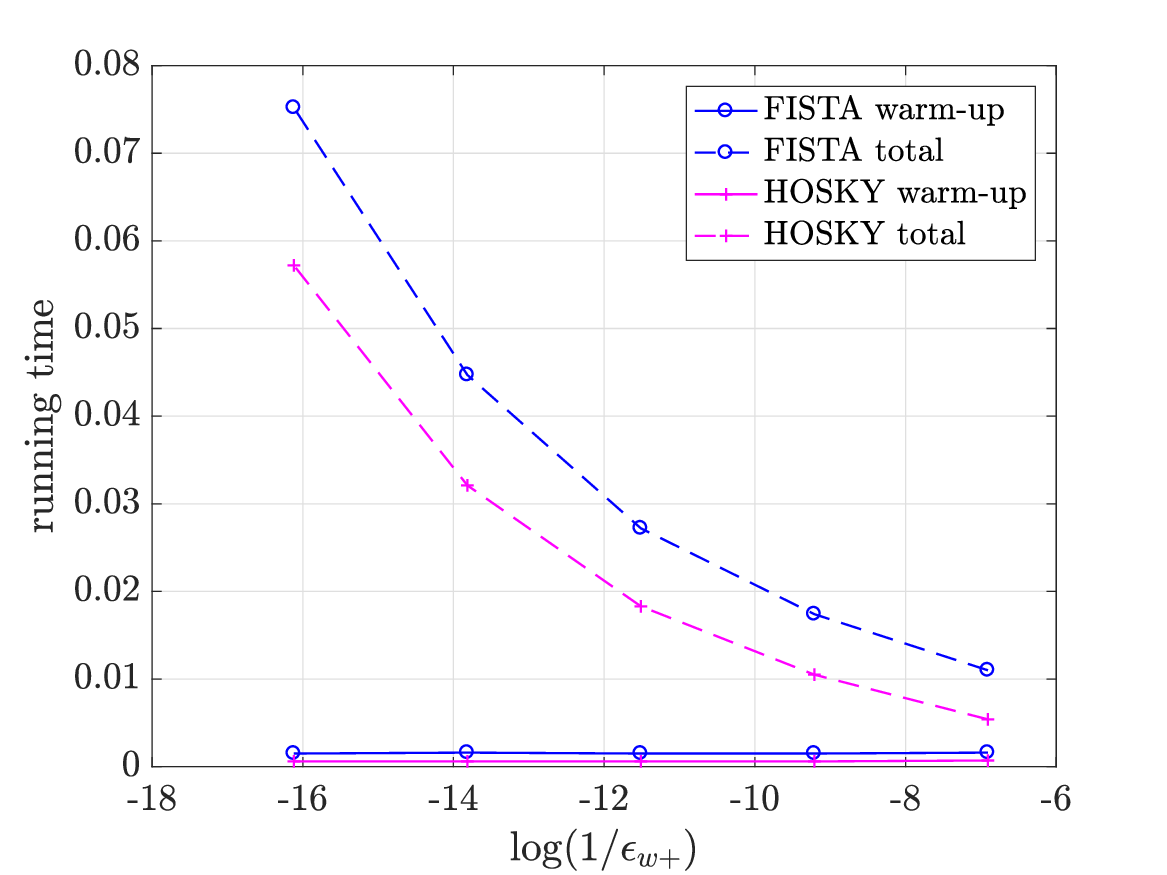}    \\
		(a) Scenario 1 ($n = 50, p = 20$) &
		(b) Scenario 1 ($n = 50, p = 80$) \\
		\includegraphics[width=0.45\textwidth]{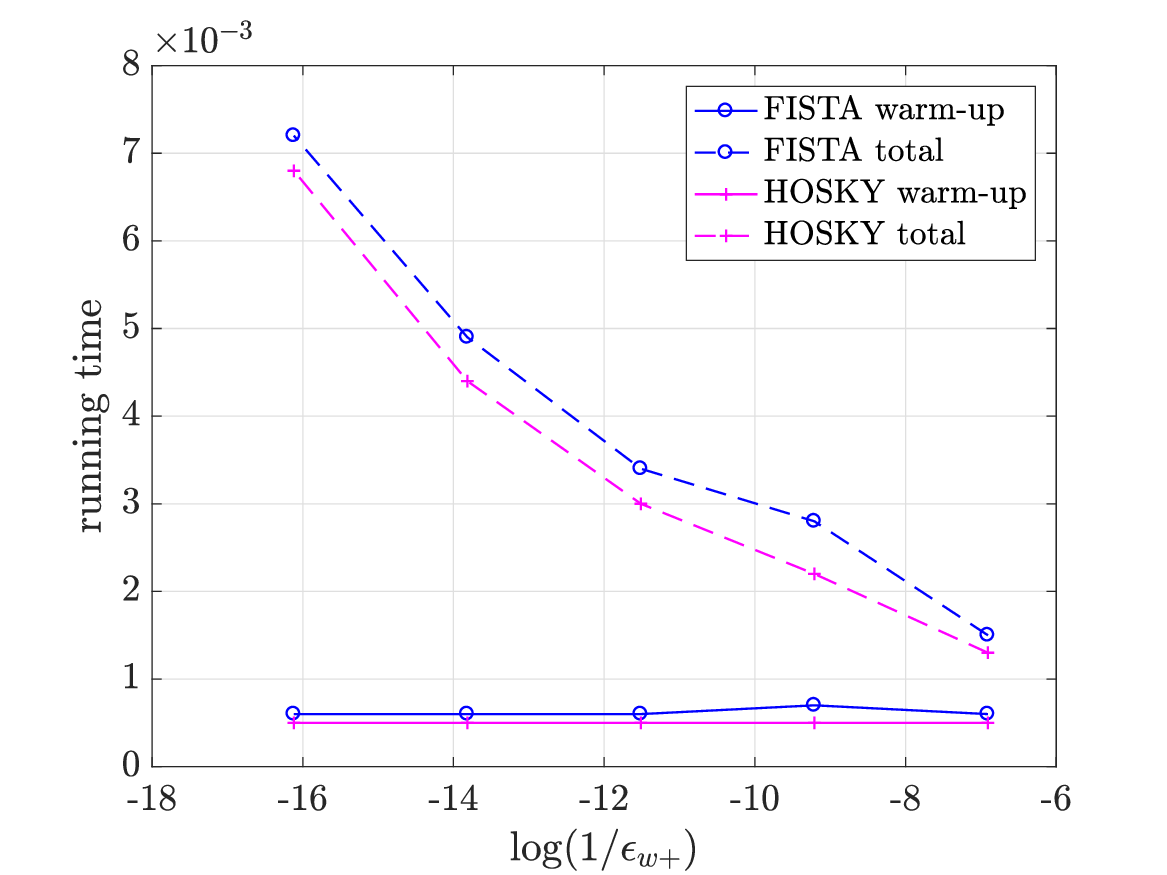}    &
		\includegraphics[width=0.45\textwidth]{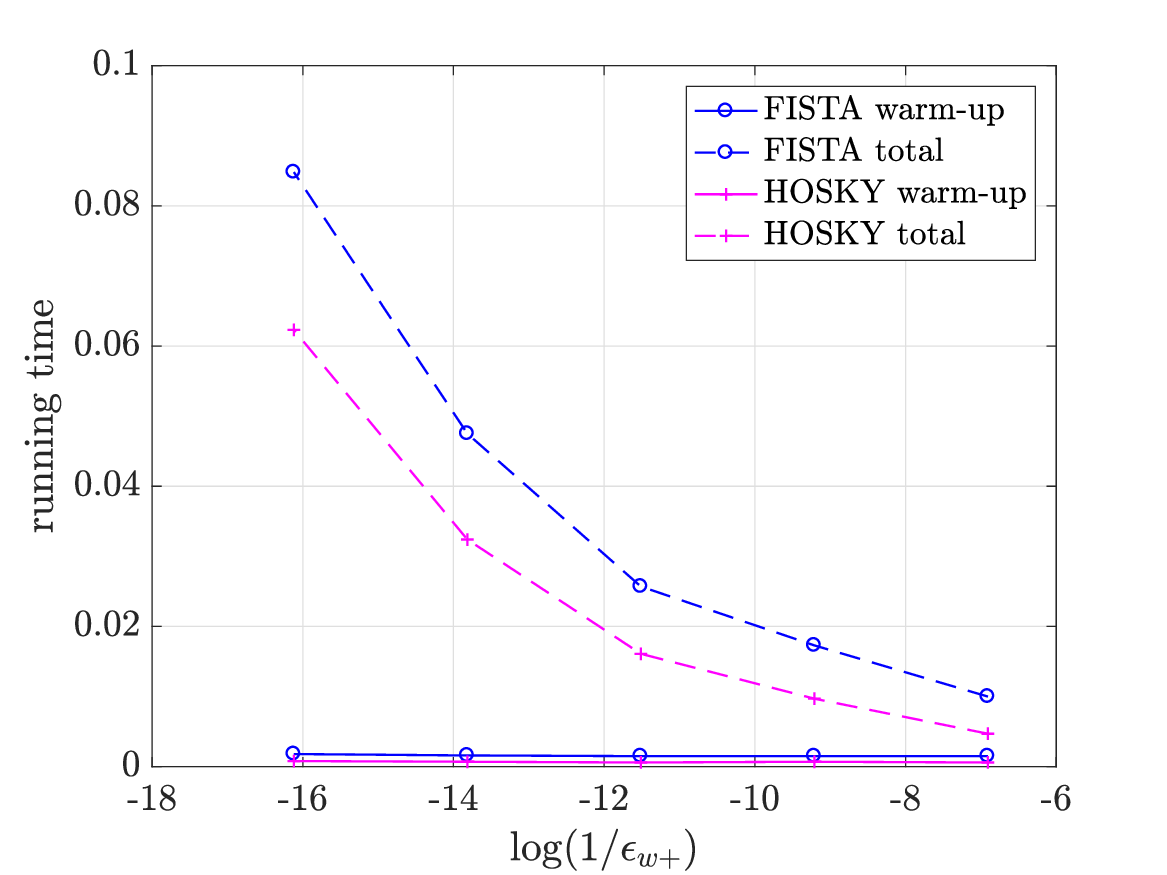} \\
		(c) Scenario 2 ($n = 50, p = 20$) &
		(d) Scenario 2 ($n = 50, p = 80$) \\
	\end{tabular}
	\caption{Running time of FISTA and HOSKY in the warm-up stage and after-warm-up stage under different after-warm-up precision $\epsilon_{w+}$
    \label{fig: simulation3 -- runtime}}
  \end{center}
\end{figure}

\section{Discussion}
\label{sec:discuss}

In our theoretical result, we required the presence of a constant $\tau:= t_K > 0$, such that $t_k \ge \tau$ for all $k = 0, \ldots, K$.
Such a condition prevents the hyper-parameter $t$ from converging to zero.
It will be interesting to study whether there is a way to relax this condition. 
In Section \ref{sec:support}, we show that when the $\tau$ is chosen to be small enough, an early-stopped homotopic approach will find the support of the global solution, therefore, one can simply run the ordinary regression on this support set, without losing anything.
In Section \ref{sec:other-homo}, we discuss other seemingly similar homotopic ideas and articulate the differences between theirs and the work that is presented in this paper.

\subsection{Support Recovery and the Need for Hyper-parameter $t$ to Converge to Zero}
\label{sec:support}

In Theorem \ref{theo: number of inner-iterations}, we assume that there is a constant $\tau > 0$, such that $t_k \ge \tau$ for all $k$.
Such a condition prevents the hyper-parameter $t$ from converging to zero.
Therefore, our result just applies to the warm-up stage of a homotopic approach to solving the Lasso problem.
In this subsection, we show that under some standard conditions that have appeared in the literature, as long as we set $\tau$ to be small enough, the associated algorithm will find a solution with both small ``prediction error'' and small ``estimation error.''
The mathematical meaning of ``prediction error'' is
$$
  \frac{1}{n} \left\| X \left( \widetilde{\beta} - \widehat{\beta} \right) \right\|_2^2,
$$
where
$
  \widetilde{\beta}
  =
  \arg\min_{\beta} \frac{1}{2n} \left\| y - X \beta \right\|_2^2 + \lambda f^*_t(\beta)
$
for a general $t$ and $f^*_t(\beta)$ defined in \eqref{equ: ft}, and
$
  \widehat\beta
  =
  \arg\min_{\beta} \frac{1}{2n} \left\| y - X \beta \right\|_2^2 + \lambda \left\| \beta \right\|_1
$.
And the mathematical meaning of ``estimation error'' in our paper is
$$
  \left\| \widetilde\beta - \widehat\beta \right\|_2^2.
$$
In the remaining of this section, we will give two propositions, where we develop the conditions the a small prediction error and estimation error hold.

We begin with the prediction error, i.e., $\frac{1}{n} \left\| X \left( \widetilde{\beta} - \widehat{\beta} \right) \right\|_2^2$.
In Proposition \ref{prop: discusssion -- prediction error}, we declare that there are no additional conditions needed to guarantee the small prediction error.
That is, as long as we converge $t \to 0$, our proposed algorithm can guarantee the prediction error goes to zero as well.
\begin{myprop}
\label{prop: discusssion -- prediction error}
  For our proposed algorithm, when $t \to 0$, we have the perdition error
  $
    \frac{1}{n}\left\| X \left( \widetilde{\beta} - \widehat{\beta} \right) \right\|_2^2 \to 0,
  $
  where
  $
    \widetilde{\beta}
    =
    \arg\min_{\beta} \frac{1}{2n} \left\| y - X \beta \right\|_2^2 + \lambda f^*_t(\beta)
  $
  for a general $t$ and $f^*_t(\beta)$ defined in \eqref{equ: ft}.
  And
  $
    \widehat\beta
    =
    \arg\min_{\beta} \frac{1}{2n} \left\| y - X \beta \right\|_2^2 + \lambda \left\| \beta \right\|_1
  $.
\end{myprop}
\begin{proof}
  See  \ref{proof: discussion -- prediction error}.
\end{proof}

After developing the prediction error, we now discuss the estimation error.
A nice property of Lasso is that, it can potentially achieve the sparse estimation when $n < p$, i.e., most of the entries in the Lasso estimator $\widehat\beta$ are zero, and only a few of them are non-zero.
The index set of these non-zero entries are called \textit{support set}, i.e.,
$
  S = \left\{ i : \text{if}\; \widehat\beta_i \neq 0 \; \forall\; i = 1,2, \ldots, p \right\}.
$
To show how Lasso can realize the sparse estimation, we take $X=I$ as an illustration example, where $I$ is the identity matrix.
(More complicated model matrix $X$ can also be used, but here we use $X=I$ to create an example.)
Then we have the linear regression model as
$$
  y = \beta + w,
$$
where $y$ is the response vector, and $w$ is the white noise.
The Lasso estimator of the above linear regression model is
$
  \widehat \beta
  =
  \arg\min_{\beta} \frac{1}{2n} \left\| y - X \beta \right\|_2^2 + \lambda \left\| \beta \right\|_1.
$
It can be verified that
\begin{equation*}
  \widehat \beta_i
  =
  \left\{
  \begin{array}{ll}
    \text{sign}(y_i) (|y_i| - n \lambda), & \mbox{if } |y_i| > n \lambda; \\
    0, & \text{otherwise},
  \end{array}
  \right.
\end{equation*}
is the solution of the Lasso problem.
Note that $\widehat \beta$ is sparse if $y$ has many components with small magnitudes.
However, if we consider $f^*_t(\beta)$, instead of the $\ell_1$ penalty $\left\| \beta \right\|_1$, we have
$$
  \widetilde \beta
  =
  \arg\min_{\beta} \frac{1}{2n} \left\| y - X \beta \right\|_2^2 + \lambda f^*_t(\beta).
$$
We can show that $\widetilde \beta_i = 0$ if and only if $y_i = 0$.
This shows that $\widetilde \beta$ is not guaranteed to be sparse.

Although $\widetilde \beta$ is not sparse, we can still verify that $\widetilde\beta$ has very small estimation error under some specific assumptions of the model matrix $X$.
\begin{myprop}
\label{prop: discussion -- estimation error}
  Suppose the model matrix $X$ in the Lasso problem has the following three properties:
  \begin{enumerate}
    \item $\left\| \left( X_S' X_S \right)^{-1} X_S' \right\|_F $ can be bounded by a constant, where
        $
          S = \{ i: \widehat\beta_i \neq 0, \forall i = 1,2,\ldots, p\}
        $
        with
        $
          \widehat \beta
          =
          \frac{1}{2n} \left\| y -  X\beta \right\|_2^2 + \lambda \left\| \beta \right\|_1.
        $
        And $\left\| \cdot \right\|_F$ is the Frobenius norm defined as $\left\| A_{m \times n} \right\|_F = \sqrt{ \sum_{i = 1}^{m} \sum_{j = 1}^{n} |a_{ij}|^2}$, where $a_{ij}$ is the $(i,j)$th entry in matrix $A$.

    \item $\left\| X_{S^c}^\dagger \right\|_F$ can be bounded by a constant, where $S^c$ is the complement set of $S$.
        And $X_{S^c}^\dagger$ is the pseudo-inverse of matrix $X_{S^c}$.
        The mathematical meaning of pseudo-inverse is that, suppose
        $
        X_{S^c} = U \Sigma V
        $, which is the singular value decomposition (SVD) of $X_{S^c}$.
        Then
        $
          X_{S^c}^\dagger = V' \Sigma^\dagger U'.
        $
        For the rectangular diagonal matrix $\Sigma$, we get $\Sigma^\dagger$ by taking the reciprocal of each non-zero element on the diagonal, leaving the zeros in place, and then transposing the matrix.

    \item
        $
          \sigma_{\max} \left(\Sigma_1\right)
          <
          \min \left\{ 2, 2 \sigma_{\min}\left(\Sigma_2 \right) \right\}
        $,
        where $\sigma_{\max} \left(\Sigma_1\right)$ returns the maximal absolute diagonal values of matrix $\Sigma_1$,
        and $\sigma_{\min} \left(\Sigma_2 \right)$ returns the minimal absolute diagonal values of matrix $\Sigma_2$.
        Matrix $\Sigma_1$ is the diagonal matrix in the SVD of matrix
        $
          \left(X_S' X_S \right)^{-1} X_S' X_{S^c} + \left(X_{S^c}^\dagger X_S \right)',
        $
        i.e.,
        $
          \left(X_S' X_S \right)^{-1} X_S' X_{S^c} + \left(X_{S^c}^\dagger X_S \right)'
          =
          U_1 \Sigma_1 V_1.
        $
        Matrix $\Sigma_2$ is the diagonal matrix of the SVD of the matrix
        $
         \frac{1}{2} X_{S^c}^\dagger X_{S^c} + \frac{1}{2} \left( X_{S^c}^\dagger X_{S^c} \right)',
        $
        i.e.,
        $
          \frac{1}{2} X_{S^c}^\dagger X_{S^c} + \frac{1}{2} \left( X_{S^c}^\dagger X_{S^c} \right)'
          =
          U_2 \Sigma_2 V_2
        $.
  \end{enumerate}
  Then we have $\left\| \widetilde \beta - \widehat \beta \right\|_2^2 \to 0$ when $t \to 0$.
\end{myprop}
\begin{proof}
  See  \ref{proof: discussion -- estimation error}.
\end{proof}

We notice that the above proposition requires a strong condition on the model matrix $X$ to achieve the support recovery.
Releasing the conditions in the above proposition is an interesting future research topic.

\subsection{Other Related Homotopic Ideas}
\label{sec:other-homo}

It is worth noting that, in recent research, some researchers also realize the log-polynomial order of complexity (seeing  \cite{xiao2013proximal, lin2014adaptive, wang2014optimal, zhao2018pathwise, pang2017parametric}) in a framework similar to Lasso-algorithms.
However, we would like to clarify that, there are some essential differences between our paper and these papers.
First, the problem formulation in these papers is different from ours.
The problem formulation these papers solve is that, they start at some initial objective problem, which is a Lasso-type objective function:
\begin{equation}
\label{equ: objective function in Zhao Tuo}
  \frac{1}{2n}\|y-X\beta\|_2^2 + \lambda^{(0)} \|\beta\|_1,
\end{equation}
and then they gradually decrease the large $\lambda^{(0)}$ until the target regularization $\lambda^{(\text{target})}$ is reached.
When the $\lambda^{(\text{target})}$ is reached, the algorithm is stopped.
However, this algorithmic solution is not optimal in \eqref{equ: objective function in Zhao Tuo}.
In other words, the solution of these papers is not exactly the Lasso solution.
While in our paper, our objective function stays the same as \eqref{equ: lasso estimator} from the beginning to the end of our algorithm.
Therefore, the solution we iteratively calculated is the minimizer of the Lasso problem in  \eqref{equ: lasso estimator}.
In addition to the difference in the objective function, the assumptions between our algorithm and these papers are also different.
Specifically, these papers require more additional assumptions than ours, such as the \textit{restricted isometry property (RIP)}, which is used to ensure that all solution path is sparse.
Finally, through both our paper and these papers are called the ``homotopic'' method, the definition of the ``homotopic'' is different.
Specifically, these papers use the homotopic path in the penalty parameter $\lambda$: they start from a very large $\lambda$ and then shrinkage to the target $\lambda$.
This type of method is also called ``path following'' in other papers, such as \cite{rosset2007piecewise}, \cite{allen2013weightedkernel}, \cite{allen2013NMR} and etc, instead of ``homotopic path''.
However, our paper uses the homotopic path in the $\ell_1$ penalty $\lambda \left\| \beta \right\|_1$: we replace the $\ell_1$ regularization term with a surrogate function, and then by adjusting the parameters in the surrogates, to get the surrogate approximates more close to the original $\ell_1$ regularization term.

\section*{Acknowledgement}
This project is partially supported by the Transdisciplinary Research Institute for Advancing Data Science (TRIAD), http://triad.gatech.edu, which is a part of the TRIPODS program at NSF and locates at Georgia Tech, enabled by the NSF grant CCF-1740776. The authors are also partially sponsored by NSF grants 1613152 and 2015363.





\bibliographystyle{elsarticle-num}
\bibliography{reference}

\appendix
\clearpage
\begin{center}
{\large\bf SUPPLEMENTARY MATERIAL}
\end{center}


\section{Review of some State-of-the-art Algorithms}
\label{appendix: Review of some State-of-the-art Algorithms}
In this section, we will show the algorithm mechanism of these four representative we select, namely ISTA \citep{ISTA} in Section \ref{sec: ISTA}, FISTA \cite{FISTA} in Section \ref{sec: FISTA}, CD \citep{glmnet} in Section \ref{sec: CD}, and SL \citep{smoothlasso} in Section \ref{sec: SL}.
For each algorithm, we show (i) their number of operations in a loop, (ii) the number of loops to meet the $\epsilon$-precision in equation \eqref{equ: F precision}, (iii) and their according order of complexity.

\subsection{Iterative Shrinkage-Thresholding Algorithms (ISTA) }
\label{sec: ISTA}
ISTA aims at the minimization of a summation of two functions, $g+f$, where the first function $g: \mathbb R^p\rightarrow \mathbb R$ is continuous convex and the other function $f: \mathbb R^p\rightarrow \mathbb R$ is smooth convex with a Lipschitz continuous gradient.
Recall the definition of Lipschitz continuous gradient as follows:
\begin{equation*}
	\Vert\nabla f(x)-\nabla f(y)\Vert_2\leq L\Vert x-y\Vert_2.
\end{equation*}
If we let $g(\beta)=\lambda\Vert\beta\Vert_1$ and $f(\beta)=\frac{1}{2n}\Vert Y-X\beta\Vert_2^2$ with the Lipschitz continuous gradient $L$ taking the largest eigenvalue of matrix $X'X/n$, noted as $\sigma_{\max}(X'X/n)$, then Lasso is a special case of ISTA.

The key point of ISTA lies in the updating rule from $\beta^{(k)}$ to $\beta^{(k+1)}$, i.e., $\beta^{(k)} \rightarrow \beta^{(k+1)}$.
It is realized by updating $\beta^{(k+1)}$ through the quadratic approximation function  of $f(\beta)$ at value $\beta^{(k)}$:
\begin{equation}
	\label{equ: ISTA update1}
	\beta^{(k+1)} = \arg\min_\beta f(\beta^{(k)}) +\langle(\beta-\beta^{(k)}),\nabla f(\beta^{(k)})\rangle +\frac{\sigma_{\max}(X'X/n)}{2} \Vert\beta-\beta^{(k)}\Vert_2^2 +\lambda\Vert\beta\Vert_1.
\end{equation}
Simple algebra shows that (ignoring constant terms in $\beta$), minimization of equation  \eqref{equ: ISTA update1} is equivalent to the minimization problem in the following equation:
\begin{equation}
	\label{equ: ISTA update2}
	\beta^{(k+1)}
	=\arg\min_\beta
	\frac{\sigma_{\max}(X'X/n)}{2}
	\left\Vert\beta-(\beta^{(k)}-\frac{\frac{1}{n}(X'X\beta^{(k)}- X'y)}{\sigma_{\max}(X'X/n)})\right\Vert_2^2
	+\lambda\Vert\beta\Vert_1,
\end{equation}
where the soft-thresholding function in equation \eqref{equ: soft thresholding} can be used to solve the problem in equation \eqref{equ: ISTA update2}:
\begin{eqnarray}
	\label{equ: soft thresholding}
	S(x,\alpha)&=&\left\{
	\begin{array}{ll}
		x-\alpha, & \mbox{ if } x\geq\alpha, \\
		x+\alpha, & \mbox{ if } x\leq -\alpha, \\
		0, &   \mbox{ otherwise. }
	\end{array}
	\right.
\end{eqnarray}
The summary of ISTA algorithm is presented in Algorithm \ref{alg: ISTA}.

\begin{algorithm}[H]
	\label{alg: ISTA}
	\caption{Iterative Shrinkage-Thresholding Algorithms (ISTA) }
	\LinesNumbered
	\KwIn{$y_{n\times1}, X_{n\times p}$, $L= \sigma_{\max}(X'X/n)$}
	\KwOut{an estimator of $\beta$ satisfies the $\epsilon$-precision, noted as $\beta^{(k)}$}
	\bfseries{initialization}\;  	
	$\beta^{(0)}, k =0$ \\
	\While{$F(\beta^{(k)})-F(\widehat\beta) > \epsilon$}{
		$\beta^{(k+1)}=S(\beta^{(k)}-\frac{1}{nL}(X'X\beta^{(k)} - X'y), \lambda/L)$
        \label{algLine: ISTA generation}\\
		$k=k+1$
	}
\end{algorithm}

It can be seen from line 4 in Algorithm \ref{alg: ISTA}  that the number of operations in one loop of ISTA is $O(p^2)$.
This is because that the main computation of each loop in ISTA is the matrix multiplication in $X'X\beta^{(k)}$.
Note that the matrix $X'X$ can be pre-calculated and saved, therefore, the order of computational complexity is $p(2p-1)$ \citep{algebraForMatrixCalculation}.

In addition to the operations in each loop, we also develop the convergence analysis of ISTA in the following equation \citep[Theorem 3.1]{FISTA}.
To make it more clear, we list  \citep[Theorem 3.1]{FISTA} below with several changes of notation.
The notations are changed to be consistent with the terminology that are used in this paper.
\begin{mytheorem}
  Let $\left\{ \beta^{(k)}\right\}$ be the sequence generated by Line \ref{algLine: ISTA generation} in Algorithm \ref{alg: ISTA}.
  Then for any $k \geq 1$, we have
  \begin{equation}
	\label{equ: ISTA prediction error bound}
	F(\beta^{(k)})-F(\widehat\beta) \leq \frac{\sigma_{\max}(X'X/n)\Vert\beta^{(0)}-\widehat\beta\Vert_2^2}{2k}.
  \end{equation}
\end{mytheorem}

Therefore, to achieve the $\epsilon$-precision, i.e., $F(\beta^{(k)})-F(\widehat\beta)\leq \epsilon$, at least $\frac{\sigma_{\max}(X'X/n)\Vert\beta^{(0)}-\widehat\beta\Vert_2^2}{2 \epsilon}$ loops are required, which leads to the order of complexity $O( \frac{\sigma_{\max}(X'X/n)\Vert\beta^{(0)}-\widehat\beta\Vert_2^2}{2\epsilon}p^2)=O(p^2/\epsilon )$.

\subsection{Fast Iterative Shrinkage-Thresholding Algorithms (FISTA) }
\label{sec: FISTA}

Motivated by ISTA, \cite{FISTA}  developed another algorithm called Fast Iterative Shrinkage-Thresholding Algorithms (FISTA).
The main difference of ISTA and FISTA is that FISTA employs an auxiliary variable $\alpha^{(k)}$ to update from $\beta^{(k)}$ to $\beta^{(k+1)}$ in the second-order Taylor expansion step (i.e., the one in equation \eqref{equ: ISTA update1}); More specifically, they have
\begin{equation}
	\label{equ: FISTA update1}
	\beta^{(k+1)} =
	\arg\min_\alpha f(\alpha^{(k)}) +
	\langle(\alpha-\alpha^{(k)}),\nabla f(\alpha^{(k)})\rangle +
	\frac{\sigma_{\max}(X'X/n)}{2} \Vert\alpha-\alpha^{(k)}\Vert_2^2 +
	\lambda\Vert\alpha\Vert_1,
\end{equation}
where $\alpha^{(k)}$ is a specific linear combination of the previous two estimator $\beta^{(k-1)}, \beta^{(k-2)}$, in particular, we have $\alpha^{(k)}=\beta^{(k-1)}+\frac{t_{k-1}-1}{t_{k}}(\beta^{(k-1)}-\beta^{(k-2)})$.
FISTA falls in the framework of Accelerate Gradient Descent(AGD), as it takes additional past information to utilize an extra gradient step via the auxiliary sequence $\alpha^{(k)}$, which is constructed by adding a ``momentum'' term $\beta^{(k-1)} - \beta^{(k-2)}$ that incorporates the effect of second-order changes.
For completeness, the FISTA is shown in Algorithm \ref{alg: FISTA}.

\begin{algorithm}[H]
	\label{alg: FISTA}
	\caption{ Fast Iterative Shrinkage-Thresholding Algorithms (FISTA) }
	\LinesNumbered
	\KwIn{$y_{n\times1}, X_{n\times p}$, $L= \sigma_{\max}(X'X/n)$ }
	\KwOut{an estimator of $\beta$, noted as $\beta^{(k)}$, which satisfies the $\epsilon$-precision.}
	\bfseries{initialization}\;  	
	$\beta^{(0)}$ ,  $t_1=1$, $k=0$ \\
	\While{ $F(\beta^{(k)})-F(\widehat\beta) > \epsilon$ }{

		$\beta^{(k)}=S(\alpha^{(k)}-\frac{1}{nL}(X'X\alpha^{(k)}- X'y), \lambda/L)$
        \label{algLine: FISTA generation beta}   \\
		$t_{k+1}=\frac{1+\sqrt{1+4t_k^2}}{2}$\\
		$\alpha^{(k+1)}=\beta^{(k)}+\frac{t_k-1}{t_{k+1}}(\beta^{(k)}-\beta^{(k-1)})$
        \label{algLine: FISTA generation alpha}\\
		$k=k+1$
	}
\end{algorithm}

Obviously, the main computational effort in both ISTA and FISTA remains the same, namely, in the soft-thresholding operation of line 4 in Algorithm \ref{alg: ISTA} and \ref{alg: FISTA}.
The number of operations in each loop of FISTA is still $O(p^2)$.
Although for both ISTA and FISTA, they have the same number of operation in one loop, FISTA has improved convergence rate than ISTA, which is shown in the following theorem \citep[Theorem 4.4]{FISTA}.
\begin{mytheorem}
  Let $\left\{ \alpha^{(k)}\right\}, \left\{ \beta^{(k)}\right\}$ be a sequence generated by Line \ref{algLine: FISTA generation alpha} and Line \ref{algLine: FISTA generation beta} in Algorithm \ref{alg: FISTA}, respectively.
  Then for any $k \geq 1$, we have that
  \begin{equation}
	\label{equ: FISTA prediction error}
	F(\beta^{(k)})-F(\widehat\beta)
    \leq
    \frac{2\sigma_{\max}(X'X/n)\Vert\beta^{(0)}-\widehat\beta\Vert_2^2}{(k+1)^2}.
   \end{equation}
\end{mytheorem}

Consequently, FISTA has a faster convergence rate  than ISTA, which improves from $O(1/k)$ to $O(1/k^2)$.
This is because that, to update from $\beta^{(k-1)}$ to $\beta^{(k)}$, ISTA only considers $\beta^{(k-1)}$, however, FISTA takes both $\beta^{(k-1)}$ and $\beta^{(k-2)}$ into account.
To achieve the precision $F(\beta^{(k)})-F(\widehat\beta)\leq \epsilon$, at least $\frac{2\sigma_{\max}(X'X/n)\Vert\beta^{(0)}-\widehat\beta\Vert_2^2}{\sqrt\epsilon}$ loops are required, which leads to an order of complexity of $O( \frac{2\sigma_{\max}(X'X/n)\Vert\beta^{(0)}-\widehat\beta\Vert_2^2}{\sqrt\epsilon}p^2)=O(p^2/\sqrt\epsilon )$.

\subsection{Coordinate Descent (CD)}
\label{sec: CD}
The updating rule in both ISTA and FISTA involve all coordinates simultaneously.
In contrast, \cite{glmnet} proposed a Lasso-algorithm that cyclically chooses one coordinate at a time and performs a simple analytical update.
Such an approach is called coordinate gradient descent.

The updating rule (from $\beta^{(k)}$ to $\beta^{(k+1)}$) in CD is that, it optimizes with respect to only the $j$th entry of $\beta^{(k+1)}$ $(j=1,\cdots, p)$ where the gradient at $\beta_j^{(k)}$
in the following equation is used for the updating process:
\begin{equation}\label{equ: CD gradient information}
  \frac{\partial}{\partial \beta_j} F(\beta^{(k)})
  =
  \frac{1}{n}
  \left( e_j' X'X \beta^{(k)}  - y'X e_j\right)
  +
  \lambda \rm{sign}(\beta_j)
\end{equation}
where $e_j$ is a vector of length $p$, whose entries are all zero expect that the $j$th entry is equal to $1$.
Imposing the gradient in equation \eqref{equ: CD gradient information} to be $0$,  we can solve for $\beta^{(k+1)}_j$ as follows:
$$
\beta^{(k+1)}_j
=
S\left(
y' X e_j - \sum_{l \neq j} \left( X'X \right)_{jl} \beta^{(k)}_k, n\lambda
\right)
\bigg/ \left( X'X\right)_{jj},
$$
where $S(\cdot)$ is the soft-thresholding function defined in equation \eqref{equ: soft thresholding}.
This algorithm has been implemented into the a R package, \textit{glmnet}, and we summarize it in Algorithm \ref{alg: CD}.

\begin{algorithm}[H]
	\label{alg: CD}
	\caption{Coordinate Descent(CD) to solve Lasso} 
	\LinesNumbered
	\KwIn{$y_{n\times1}, X_{n\times p}$, $\lambda$ }
	\KwOut{an estimator of $\beta$, noted as $\beta^{(k)}$, which satisfies the $\epsilon$-precision.}
	\bfseries{initialization}\;  	
	$\beta^{(0)}, k = 0$ \\
	\While{$F(\beta^{(k)})-F(\widehat\beta) > \epsilon$ }{
		\For {$j=1\cdots p$}{
			$\beta^{(k+1)}_j
               =
               S\left( y' X e_j - \sum_{l \neq j} \left( X'X \right)_{jl} \beta^{(k)}_k, n\lambda \right)
               \bigg/
               \left( X'X\right)_{jj} $
               \label{algLine: CD generation}
		}
	}
\end{algorithm}

After reviewing the algorithm of CD, we develop the order of complexity of CD.
Firstly, the number of operations in each loop of CD is $O(p^2)$.
It can be explained by the following two reasons.
(i) While updating $\beta^{(k+1)}_j$ (line 5 in Algorithm \ref{alg: CD}), it costs $O(p)$ operations because of
$
\sum_{l \neq j} \left( X'X \right)_{jl} \beta^{(k)}_k
$.
(ii) From line 4 in Algorithm \ref{alg: CD}, we can see that all $p$ entries of $\beta^{(k+1)}$ are updated one by one.
Combining (i) and (ii), we can see that the number of operations need in one loop of CD is of the order $O(p^2)$.


The convergence rate of CD is derived as a corollary in \cite[Corollary 3.8]{CDconvergence} and here we list the corollary as a theorem below.
We changed several notations to adopt the terminology in this paper:
\begin{mytheorem}
\label{theo: cd convergence rate}
  Let $\left\{ \beta^{(k)}  \right\}$ be the sequence generated by the Line \ref{algLine: CD generation} in Algorithm \ref{alg: CD}.
  Then we have that
  \begin{equation}
	\label{equ: CD prediction error}
	F(\beta^{(k)})-F(\widehat\beta)
    \leq
    \frac{4\sigma_{\max}(X'X/n)(1+p)\Vert\beta^{(0)}-\widehat\beta\Vert_2^2}{k+(8/p)}.
  \end{equation}
\end{mytheorem}

The above equation shows that,
to achieve the precision $\epsilon$-precision, at least
$$
\frac{4\sigma_{\max}(X'X/n)(1+p)\Vert\beta^{(0)}-\widehat\beta\Vert_2^2}{\epsilon}
-
\frac{8}{p}
$$
loops are required, which leads to an order of complexity of
$$
O([\frac{4\sigma_{\max}(X'X/n)(1+p)\Vert\beta^{(0)}-\widehat\beta\Vert_2^2}{\epsilon}-\frac{8}{p}]p^2)
=
O(p^2/\epsilon-8p)
=
O(p^2/\epsilon).
$$

As suggested by our reviewers, it is worth noting that CD is a generic algorithm, where both Lasso and ridge regression are two applications. So we also list its application in ridge regression below and its computational complexity is similar to the conclusion in Theorem \ref{theo: cd convergence rate}.
\begin{algorithm}[htbp]  
        \caption{Coordinate Descent(CD) to solve ridge regression \label{alg: CD for ridge regression}}
        \LinesNumbered
        \KwIn{
        \begin{enumerate}
            \item The response vector $y_{n\times1}$ with its $i$-th entry denoted as $y_i$;
            \item The model matrix $X_{n\times p}$ with its $(i,j)$-th entry denoted as $x_{ij}$;
            \item The penalty parameter $\lambda > 0$;
            \item The total number of iterations $K$;
            \item The soft -thresholding function 
            $$
            S(x, \gamma) = \left\{
            \begin{array}{ll}
            x - \gamma & \text{if } x > 0 \text{ and } \gamma < |x|;\\
            x + \gamma & \text{if } x < 0 \text{ and } \gamma < |x|;\\
            0          & \gamma \geq |x.|
            \end{array}
            \right.
            $$
        \end{enumerate}
        }
        \KwOut{an estimator of $\beta$ after $K$ iterations, noted as $\beta^{(K)}$.}
        \bfseries{Initialization:}	$\beta^{(0)}$ \\
        \For{$k=0,1,\ldots, K$ }{
        \textnormal{The current solution of $\beta$ is $\beta^{(k)} = (\beta_1^{(k)}, \ldots, \beta_p^{(k)})'$}\\
		\For {$j = 1, \cdots, p$}{
            \For{$i \in \{1, \cdots, p\}$ \textnormal{and} $i \neq j$}{
            $\mathcal Y _i^{-j} = \sum_{\ell \neq j} x_{i\ell} \beta_\ell^{(k)}$\\
            }
            $
            \beta_j^{(k+1)} = 
            S\left( 
            \frac{1}{n} \sum_{i=1}^n x_{ij} 
            (y_i - \mathcal Y_i^{-j}),
            0
            \right)
            $
		}
	  }
   
        \end{algorithm}

\subsection{Smooth Lasso (SL)}
\label{sec: SL}
The aforementioned Lasso-algorithms all aim exactly at minimizing the function $F(\beta)$.
On the contrary, \cite{smoothlasso}  used an approximate objective function to solve the Lasso.
Their method is called a Smooth-Lasso (SL) algorithm.
The main idea of SL is that it use a smooth function---$\phi_\alpha(u)=\frac{2}{u}\log(1+e^{\alpha u})-u$--- to approximate the $\ell_1$ penalty, and Accelerated Gradient Descent (AGD) algorithm is applied after the replacement.
Therefore, the objective function of SL becomes $F_\alpha(\beta)=\frac{1}{2n}\|y-X\beta\|_2^2 + \lambda \sum_{i=1}^{p}\phi_\alpha(\beta_i)$. The pseudo code of SL is displayed in Algorithm \ref{alg: SL}.

\begin{algorithm}[H]
	\label{alg: SL}
	\caption{Smooth Lasso (SL) }
	\LinesNumbered
	\KwIn{$y_{n\times1}, X_{n\times p}$, $\mu=\left[\sigma^2_{\max}(X/\sqrt n)+\lambda\alpha/2\right]^{-1}$}
	\KwOut{an estimator of $\beta$, noted as $\beta^{(k)}$, which satisfies the $\epsilon$-precision.}
	\bfseries{initialization}\;  	
	$\beta^{(0)}$ , $k=0$\\
	\While{ $F(\beta^{(k)})-F(\widehat\beta) > \epsilon$  }{  	
		$w^{(k+1)}=\beta^{(k)}+\frac{k-2}{k+1}(\beta^{(k)}-\beta^{(k-1)})$\\
		$\beta^{(k+1)}=w^{(k+1)}-\mu\nabla F_\alpha( w^{(k)})$
        \label{algLine: SL generation}\\
		$k=k+1$
	}
\end{algorithm}

For the computational effort, it mainly lies in the calculation of $\nabla F_\alpha(w)=\frac{X'X}{n}w-\frac{X'y}{n}+v$, where the $v$ is a vector of length $p$, whose $i${th} entry is $\frac{-2}{w_i^2}\log(1+e^{\alpha w_i})+\frac{2\alpha e^{\alpha w_i}}{w_i(1+e^{\alpha w_i})}-1$.
Accordingly, the main computational effort of each loop of SL is the matrix multiplication in $X'Xw^{(k)}$, which cost $O(p^2)$ operations.
On the other side, proved by \cite{smoothlasso}, the approximation error of $\beta^{(k)}$ in SL is shown in equation \eqref{equ: SL prediction error}.
\begin{mytheorem}
  Let $\left\{ \beta^{(k)}\right\}$ be a sequence generated as in Line \ref{algLine: SL generation} of Algorithm \ref{alg: SL}.
  Then we have
  \begin{equation}
	\label{equ: SL prediction error}
	F(\beta^{(k)})-F(\widehat\beta)
    \leq  \frac{4\Vert\beta^{(0)}-\widehat\beta\Vert_2^2\sigma_{\max}^2(\frac{X}{\sqrt{n}})}{k^2}
    +
    \frac{4\sqrt{2\lambda n\log2}\Vert\beta^{(0)}-\widehat\beta\Vert_2}{k}.
\end{equation}
\end{mytheorem}
So to achieve the $\epsilon$-precision, SL needs $O(1/\epsilon)$, which results in the order of complexity $O(p^2/\epsilon)$.

\subsection{Path Following Lasso-Algorithm}
\label{sec: discussion -- path following}
As mentioned in Section \ref{sec: introduction}, the path following Lasso-algorithm has two drawbacks.
First, it is not guaranteed to work in general cases.
Second, there is no theoretical guarantee that the order of complexity of a path following Lasso-algorithm is low, considering that the maximum number of loops can be as large as $2^p$, where $p$ is the number of predictors.
In this section, we provide mathematical details to support the above two drawbacks.
The structure of this section is described as follows.
In Section \ref{discussion -- path following -- drawback1}, we provide a counter example that the path following Lasso-algorithm is not workable, which represents a general category of design matrix $X$ and coefficient $\beta$.
In Section \ref{discussion -- path following -- drawback2}, we provide mathematical details to support the second drawback of the path following Lasso-algorithm.

\subsubsection{Details to Support the first Drawback of Path Following Lasso Algorithm}
\label{discussion -- path following -- drawback1}

In this section, we provide a counter example that the path following Lasso-algorithm is not workable.
This counter example represents a general category of design matrix $X$ and coefficient $\beta$.
We use the following counterexample to argue that a path following approach does not work in the most general setting.

Before representing the concrete counter example, let us discuss the key step in designing a path following Lasso-algorithm.
For a general solution derived by path following Lasso-algorithm, i.e., $\widehat\beta(\lambda)$, it is the minimizer of \eqref{equ: lasso estimator}, so it must satisfy the first order condition of \eqref{equ: lasso estimator}:
\begin{equation}
\label{equ: first-order condition of lasso in path following}
  q - \lambda \text{sign}( \widehat\beta(\lambda)) = X'X \widehat\beta(\lambda),
\end{equation}
where $q = X'y$ and  is $\text{sign}( \widehat\beta(\lambda))$ a vector, whose $i$th component is the sign function of $\widehat\beta(\lambda)$:
$$
  \text{sign}(\beta_i(\lambda)) =
  \left\{
  \begin{array}{cc}
    1                  & \text{if } \beta_i(\lambda)>0 \\
    -1                 & \text{if } \beta_i(\lambda)<0 \\
    \left[-1,1 \right] & \text{if } \beta_i(\lambda) = 0
  \end{array}
  \right..
$$
If we divide the indices of $q, \beta, X$ into
$
  S= \{i: \widehat\beta_i(\lambda) \neq 0, \;\forall\; i= 1,\ldots,p\}
$
and its complements $S^c$, then we can rewrite equation \eqref{equ: first-order condition of lasso in path following} as
\begin{equation*}
  \left(
  \begin{array}{c}
    q_S  \\
    q_{S^c}
  \end{array}
  \right)
  -
  \left(
  \begin{array}{c}
    \lambda \text{sign}( \widehat\beta_S(\lambda)) \\
    \lambda \text{sign}( \widehat\beta_{S^c}(\lambda))
  \end{array}
  \right)
  =
  \left(
  \begin{array}{cc}
    X_S^\top X_S& X_S^\top X_{S^c} \\
    X_{S^c}^\top X_S & X_{S^c}^\top X_{S^c}
  \end{array}
  \right)
  \left(
  \begin{array}{c}
    \widehat\beta_S(\lambda) \\
    0
  \end{array}
  \right),
\end{equation*}
where $\widehat\beta_S(\lambda)$ is the subvector of $\beta$ only contains elements whose indices are in $S$ and $\widehat\beta_{S^c}(\lambda)$ is the complement of $\beta_S$.
Besides, $\text{sign}(\widehat\beta_S(\lambda))$ is the subset of $\text{sign}(\widehat\beta(\lambda))$, only contains the elements whose indices are in $S$, and $\text{sign}(\widehat\beta_{S^c}(\lambda))$ is the complement to $\text{sign}(\widehat\beta_S(\lambda))$.
Matrix $X_S$ is the columns of $X$ whose indices are in $S$, and $X_{S^c}$ is the complement of $X_S$.

Suppose we are interested in parameter estimated under $\lambda$ and $\lambda - \Delta (\Delta\in(0,\lambda))$, i.e., $\widehat\beta(\lambda), \widehat\beta(\lambda - \Delta)$.
Then $\widehat\beta(\lambda), \widehat\beta(\lambda - \Delta)$ must satisfy the following two system of equations:
\begin{equation}
\label{equ: lambda}
  \left\{
  \begin{array}{rcl}
    q_S - \lambda \text{sign}( \widehat\beta_S(\lambda))
    &=&
    X_S' X_S \widehat\beta_S(\lambda) \\
    q_{S^c} - \lambda \text{sign}(\widehat\beta_{S^c}(\lambda) )
    &=&
    X_{S^c}' X_S \widehat\beta_S(\lambda)
  \end{array}
  \right.,
\end{equation}

\begin{equation}
\label{equ: lambda - delta}
  \left\{
  \begin{array}{rcl}
    q_S - (\lambda - \Delta) \text{sign}(\widehat\beta_S(\lambda - \Delta))
    & = &
    X_S' X_S \widehat\beta_S(\lambda - \Delta) \\
    q_{S^c} - (\lambda - \Delta) \text{sign}(\widehat\beta_{S^c}(\lambda - \Delta) )
    & = &
    X_{S^c}' X_S \widehat\beta_S(\lambda - \Delta)
  \end{array}
  \right..
\end{equation}
From the above two system of equations, we have the following:
\begin{equation}
\label{equ: lambda -> lambda-Delta}
   - (\lambda - \Delta) \text{sign}(\widehat\beta_{S^c} (\lambda - \Delta))
  =
   - \lambda \text{sign}(\widehat\beta_{S^c}(\lambda)) + \Delta X_{S^c}' X_S (X_S' X_S)^{-1} \text{sign}(\widehat\beta_S (\lambda)).
\end{equation}
That is, if one decrease $\lambda$ to $\lambda-\Delta$, one must strictly follow \eqref{equ: lambda -> lambda-Delta}.

Following the above key step in the path following Lasso-algorithm, we represent a counter example as follows.
Suppose
$
  \beta_1 > \beta_2 > \beta_3 > \beta_4 = \beta_5 = \ldots = \beta_p = 0
$.
The model matrix
$
  X = (X_1, X_2, X_3,\ldots, X_p)
$,
where
$
  X_1, X_2 \in \mathbb R^n
$
is the first two columns from a orthogonal matrix
$
  (X_1, X_2, \widetilde X_3, \ldots, \widetilde X_p)
$,
and for $j \geq 3$, we have
$
  X_j = \alpha_j X_1 + (1-\alpha_j) X_2 + \sqrt{1 - \alpha_j^2 - (1-\alpha_j)^2} \widetilde X_j$ with $\alpha_j \in (0,1)
$.
The response vector $y$ is generated by
$$
  y = \sum_{j=1}^{p} \beta_j X_j.
$$
If $\beta_1, \beta_2$ are very large number, say, 200, 100, and $\beta_3$ is not that large, say, 1.
Then the following algorithm works as follows:
\begin{itemize}
  \item Loop 0: We start with $\lambda_0 = +\infty$, then we know that $\widehat\beta(\lambda_0) = 0$ and $S_0 = \varnothing$.
  \item Loop 1: When $\lambda$ changes from $\lambda_0 = +\infty$ to $\lambda_1 = \| q \|_\infty$, from \eqref{equ: first-order condition of lasso in path following}, we know that $S_1 = \{ 1 \}$.
  \item Loop 2:  Similar to the first loop, when $\lambda$ decrease to $\lambda_2$, we have $S_2 = \{1, 2 \}$.
  \item Loop 3:  This is where problem happens.
  From \eqref{equ: lambda -> lambda-Delta}, we know that $\forall \lambda_2 - \Delta \in (\lambda_3, \lambda_2]$, we have
  $$
    \text{sign}(\widehat\beta_{S_2^c} (\lambda - \Delta))
    =
    X_{S_2^c}' X_{S_2} (X_{S_2}' X_{S_2})^{-1} \text{sign}(\widehat\beta_{S_2}(\lambda)).
  $$
  Since
  $
    \text{sign}(\widehat\beta_{S_2}(\lambda)) = (1, 1)'
  $
  and
  $
    X_{S_2} = (X_1, X_2), X_{S_2^c} = (X_3, \ldots, X_p)
  $,
  we have the right hand side of the above equation as a all-one vector, i.e, $(1, 1, \ldots, 1)^\top$.
  To make the left hand side $\text{sign}(\widehat\beta_{S_2^c} (\lambda_2 - \Delta))$ equals to $(1, 1, \ldots, 1)'$, we can only take $\Delta = \lambda_2$, which gives us $S_3 = \{1,2,3,\ldots, p\}$.
\end{itemize}
However, from the data generalization, we know that the true support set is $\{1,2,3\}$.
Therefore, one will not be able to develop a path following algorithm to realize correct support set recovery.
At least not in the sense of inserting one at a time to the support set. In the above example, since a path following approach can only visit three possible subsets, it won't solve the Lasso problem in general.

\subsubsection{Details to Support the Second Drawback of Path Following Lasso-Algorithm}
\label{discussion -- path following -- drawback2}

In this section, we provide more technical details to support the second drawback of path following Lasso-Algorithm.
Recall the main idea of path following Lasso-Algorithm is that, it begins with a large $\lambda_0$, which makes the estimated $\widehat\beta(\lambda_0) = 0$, and accordingly its support set $S_0 = \varnothing$ (empty set).
Then it tries to identify a sequence of the penalty parameter $\lambda$ as follows:
$$
  \lambda_0
  > \lambda_1
  > \lambda_2
  > \ldots
  > \lambda_{T-1}
  > \lambda_T = 0,
$$
such that for any $k \geq 1$, when we have
$
  \lambda \in [\lambda_k, \lambda_{k-1}],
$
the support of $\widehat\beta(\lambda)$ (which is a function of $\lambda$) i.e., $S_k$, remains unchanged.
Moreover, within the interval $[\lambda_k, \lambda_{k-1}]$, vector $\widehat\beta(\lambda)$ elementwisely is a linear function of $\lambda$.
However, when one is over the kink point, the support is changed/enlarged, i.e., we have $S_{k} \neq S_{k-1}$ or even $S_{k} \subseteq S_{k-1}$.

A point deserves attention is that, if $T$, the total number of kink points is small, then the path following algorithm is efficient, i.e., it only requires $O(n T p^2)$ numerical operations.
In particular, if the size of supports are strictly increasing, i.e., we have
$$
  |S_{k-1}| < |S_k| \;\; \forall k \geq 1,
$$
then we have $T \leq p$, and accordingly the computational complexity can be bounded by $O(n p^3)$.
However, it turns out bounding the value of $T$ is an open question.
In recent papers such as \cite{tibshirani2011solution, rosset2007piecewise}, we can see that bounding $T$ is an open problem.


\section{An Important Theorem}
\label{appendix: george lan theorem on AGD}

Our proof will rely on a result on the number of steps in achieving certain accuracy in using the accelerate gradient descent (AGD) when the objective function is strongly convex.
The result is the Theorem 3.7 in \cite{lan2019lectures}.
We represent the theorem here for readers' convenience.
We introduce some notations first.
Suppose ones wants to  minimize a convex function $f: X \to \mathbb R$ in a feasible closed convex set $X \in \mathbb R^p$.
We further assume that $f$ is a differentiable convex function with Lipschitz continuous gradients $L$, i.e., $\forall x,y \in X$, we have
$$
    \left\|
    \nabla f(x) - \nabla f(y)
    \right\|_2
    \leq
    L \left\| x-y \right\|_2,
$$
where $\nabla f(x)$ represents the gradient of function $f(x)$.
Furthermore, we assume that $f$ is a strongly convex function, i.e., $\forall x,y \in X$, there exist $\mu>0$, such that we have
$$
    f(y) \geq f(x) + \nabla f(x)(y-x) + \frac{\mu}{2} \left\| y - x \right\|_2^2.
$$
This type of function $f$ is called the $L$-smooth and $\mu$-strongly convex function.
Recall that our objective is to solve the following problem:
$$
    \min_{x \in X} f(x).
$$

In the following, we present one version of the accelerated gradient descent (AGD) algorithm.
Given $(x^{(t-1)}, \bar{x}^{(t-1)}) \in X \times X$ for $t= 1,2, \ldots$, we set
\begin{eqnarray}
    \underline x^{(t)}
    & = &
    \label{equ: proof -- lan -- AGD procedure 1}
    (1 - q_t) \bar x^{(t-1)} + q_t x^{(t-1)}  \\
    x^{(t)}
    & = &
    \label{equ: proof -- lan -- AGD procedure 2}
    \arg\min\limits_{x \in X}
    \left\{
    \gamma_t
    \left[
    x' \nabla f \left( \underline x^{(t)} \right) + \mu V \left(\underline x^{(t)}, x \right)
    \right]
    + V \left( x^{(t-1)}, x \right)
    \right\} \\
    \bar x^{(t)}
    \label{equ: proof -- lan -- AGD procedure 3}
    & = &
    \left(1 - \alpha_t \right) \bar x^{(t-1)} + \alpha_t x^{(t)},
\end{eqnarray}
for some $q_t \in [0,1], \gamma_t \geq 0$, and $\alpha_t \in [0,1]$.
And here $V(x,z)$ is the prox-function (or Bregman's distance), i.e.,
$$
    V(x,z) = v(z) - \left[ v(x) + (z-x)' \nabla v(x) \right],
$$
with $v(x) = \left\| x \right\|_2^2 / 2$.
By applying AGD as shown in \eqref{equ: proof -- lan -- AGD procedure 1}-\eqref{equ: proof -- lan -- AGD procedure 3}, the following theorem presents an inequality that can be utilized to determine the number of loops when certain precision of a solution is given.
\begin{mytheorem}
\label{theo: appendix -- lan -- AGD -- convergence rate}
  Let $\left( \underline x^{(t)}, x^{(t)}, \bar x^{(t)} \right) \in X \times X \times X$ be generated by accelerated gradient descent method in \eqref{equ: proof -- lan -- AGD procedure 1}-\eqref{equ: proof -- lan -- AGD procedure 3}.
  If $\alpha_t = \alpha, \gamma_t = \gamma$ and $q_t = q$, for $t= 1,\ldots, k$, satisfy
  $\alpha \geq q$,
  $\frac{L(\alpha - q)}{1 - q} \leq \mu$,
  $\frac{L q (1 - \alpha)}{ 1 - q} \leq \frac{1}{\gamma}$,
  and
  $
    \frac{1}{\gamma(1 - \alpha)} \leq \mu + \frac{1}{\gamma},
  $
  then for any $x \in X$, we have
  $$
    f\left(\bar x^{(k)} \right) - f(x) + \alpha \left( \mu + \frac{1}{\gamma} \right) V \left( x^{(k-1)}, x \right)
    \leq
    (1-\alpha)^k
    \left[
    f \left( \bar x^{(0)} \right) - f(x) + \alpha \left(\mu + \frac{1}{\gamma} \right) V\left( x^{(1)}, x \right),
    \right].
  $$
  In particular, if
  $$
    \alpha = \sqrt{\frac{\mu}{L}}, q = \frac{\alpha - \mu/L}{1-\mu/L}, \gamma = \frac{\alpha}{\mu(1-\alpha)},
  $$
  then for any $x \in X$, we have
  \begin{equation}
  \label{equ: appendix -- lan -- AGD -- convergence rate}
    f \left( \bar x^{(k)} \right) - f(x) + \alpha \left( \mu + \frac{1}{\gamma} \right) V \left(x^{(k-1)}, x \right)
    \leq
    \left(1 - \sqrt{\frac{\mu}{L}} \right)^k
    \left[
    f\left(\bar x^{(0)} \right) - f(x) + \alpha \left( \mu + \frac{1}{\gamma} \right) V\left(x^{(1)}, x \right)
    \right].
  \end{equation}
\end{mytheorem}
The above theorem gives a convergence rate of AGD under the scenario when the objective function is strongly convex.
This result will be utilized in the proof of Theorem \ref{theo: number of inner-iterations}, which can be found in  \ref{proof: theo number of inner iterations}.

\section{Proofs}
\label{app:proofs}

\subsection{Proof of a Lemma}
\label{proof: lemma initial point}

The proof of Lemma \ref{lemma: initial point} is as follows.
\begin{proof}
    In this proof, we will do two parts.

    First, we will prove the existence of the initial point $t_0$ stated in \eqref{equ: t0}.	
    We know that
    \begin{eqnarray*}
    \lim\limits_{t \rightarrow +\infty}
    \frac{\sum_{j=1}^{p} M(t)_{ij} (X'y/n)_j}{t}
    & = &
    \lim\limits_{t \rightarrow +\infty}
    \sum_{j=1}^{p}
    \left( \left[ \frac{X'X}{n} +  \frac{ \lambda \left[ \log( 1+t ) \right]^2 }{ 3 t^3 }  I \right] ^{-1} \right)_{ij}
    \left( \frac{X'y}{n} \right)_j
    \frac{1}{t} \\
    & = &
    \lim\limits_{t \rightarrow +\infty}
    \sum_{j=1}^{p}
    \left( \left[ \frac{X'X t}{n} +  \frac{ \lambda \left[ \log( 1+t ) \right]^2 }{ 3 t^2 }  I \right] ^{-1} \right)_{ij}
    \left( \frac{X'y}{n} \right)_j. \\
    & = &
    0.
    \end{eqnarray*}
    The above indicates that when $t$ is very large, the $t_0$ defined in \eqref{equ: t0} will exist.

    Next, we will verify that, if $t_0$ is chosen as shown in \eqref{equ: t0}, i.e,
    $$
    t_0
    \in
	\left\{
    t:
    \left | \sum_{j=1}^{p} M(t)_{ij} ( X'y/n )_j  \right| \leq t,
    \forall i = 1, \ldots, p
	\right\},
    $$
    we have $\left| \beta^{(0)}_i \right| < t_0$.
    It can be verified that,
    \begin{equation}
	\label{equ: object function under large t0}
        M(t) \frac{X'y}{n}
        =
        \arg \min\limits_{\beta}
        \left\{
        \underbrace{	
        \frac{1}{2n}\| y - X\beta \|_2^2
        +
        \frac{1}{3 t^3}   \left[ \log(1+t) \right]^2 \beta'\beta
        }_{\mathcal G(\beta)}
        \right\},
	\end{equation}
    where $\mathcal G(\beta)$ is a special case of $F_t(\beta)$ when $t$ is large enough to include all the coefficient $\beta_i$ into the interval $[-t, t]$.
    Utilizing the above fact that the minimizer in \eqref{equ: object function under large t0} when $t=t_0$ satisfies the condition that its coordinates are within $[-t_0, t_0]$,  we have
    $$
        \left| \beta_i (t_0) \right|
        = \left| \left( M(t_0) \frac{X'y}{n} \right)_i \right|
        =  \left| \sum_{j=1}^{p} M(t_0)_{ij} ( X'y/n )_j  \right|
         \leq t_0.
    $$
    Thus, if we choose $t_0$ as shown in \eqref{equ: t0}, i.e.,
    \begin{equation*}
	t_0
    \in
	\left\{
    t:
    \left | \sum_{j=1}^{p} M(t)_{ij} ( X'y/n )_j  \right| \leq t,
    \forall i = 1, \ldots, p
	\right\},
	\end{equation*}
    we can verify that $\forall i = 1 , 2, \ldots p$, $\left | \beta_i(t_0) \right |\leq t_0$, i.e., $\left | \beta^{(0)}_i \right |\leq t_0$.
\end{proof}

\subsection{Proof of a Lemma}
\label{proof: lemma closeness between ft and x}

The Proof of Lemma \ref{lemma:closeness-between-ft-and-x} is as follows.
\begin{proof}

    Because $f_t(x)$ is a even function, we only consider the positive $x$ in the remaining of the proof.
	
	First, when $ 0 \leq x \leq t$, one has
	$$
	f_t(x)-x=\frac{1}{3 t^3}\left[ \log(1+t) \right]^2  x^2-x,
	$$
	which is a quadratic function with the axis of symmetry, $\frac{3t^3}{2[\log(1+t)]^2}$ being larger than $t$. Therefore, one has
	\begin{eqnarray*}
		\frac{1}{3t} \left[ \log(1+t) \right]^2 -t & \leq f_t(x)-x \leq & 0.
	\end{eqnarray*}
	
	Then we discuss the scenario when $x>t$, where
	$$
	f_t(x)-x
	=    \left[ \left[ \frac{ \log(1+t) }{t}\right]^2 -1 \right] x +\frac{1}{3x} \left[ \log(1+t) \right]^2 -\frac{1}{t} \left[ \log(1+t) \right]^2,
	$$
	which is a decreasing function of variable $x$.
	Therefore,
	\begin{equation*}
f_t(B)- B = \left[ f_t(x)-x \right] |_{x=B} \leq  f_t(x)-x  \leq \left[ f_t(x)-x \right] |_{x=t} = f_t(t)-t,
	\end{equation*}
	where we further have $f_t(t) - t = \frac{1}{3t} \left[ \log(1+t) \right]^2 -t \leq 0$.
	
	By the combination of two scenario ($x\leq t$ and $x>t$), we prove the statement in equation \eqref{equ: fx-x}.
\end{proof}

\subsection{Proof of a Theorem}\label{proof: theo number of inner iterations}

The proof of Theorem \ref{theo: number of inner-iterations} can be found below.
\begin{proof}
To begin with, we revisit some notations in linear algebra.
For matrix $A$, we use $A_{ij}$ to indicate the $(i,j)$th entry in matrix $A$.
Besides, its maximal/minimal eigenvalue is $\lambda_{\min}(A) / \lambda_{\min}(A)$, respectively.

It is known that, the condition number of function
$
F_{t_k}(\beta) = \frac{1}{2n} \left\| y - X \beta \right \|_2^2
+
\lambda f_{t_k}(\beta)
$
is defined by the ratio between the maximal and minimal eigenvalue of its Hessian.
Recall that, the $(i, j)$th entry of the Hessian matrix of the surrogate function $f_{t_k}(\beta)$, noted as $H_{t_k, i,j}$, is
	\begin{eqnarray*}
		H_{t_k, i,j}	&=& \left\{
		\begin{array}{ll}
			\frac{2}{3}
             \left[
             \log(1+t_k)
             \right] ^ 2
             \max\{ |\beta_i|, t_k \}^{-3} ,
             & \mbox{ if } i=j, \\
			 0	,
             &  \mbox{ otherwise. }
		\end{array}
		\right.
	\end{eqnarray*}
Note that the Hessian matrix $H_{t_k}$ is diagonal and positive definite; therefore one can easily find its minimum and maximum eigenvalues.
So the condition number of function
$
F_{t_k}(\beta) = \frac{1}{2n} \left\| y - X \beta \right \|_2^2
+
\lambda f_{t_k}(\beta)
$,
noted as $\kappa_k$, is
\begin{eqnarray}
  \kappa_k
  & = &
  \label{equ: condNum 1}
  \frac{\lambda_{\max}\left( \frac{X'X}{n} + \lambda H_{t_k} \right)}
  {\lambda_{\min}\left( \frac{X'X}{n} + \lambda H_{t_k} \right)} \\
  & \leq &
  \label{equ: condNum 2}
  \frac{\lambda_{\max}\left( \frac{X'X}{n} \right) + \lambda \lambda_{\max}\left(H_{t_k} \right) }
  {\lambda_{\min}\left( \frac{X'X}{n} \right)  + \lambda \lambda_{\min}\left(H_{t_k} \right)}  \\
  & \le &
  \nonumber
  \frac{\lambda_{\max}\left( \frac{X'X}{n} \right) + \lambda \lambda_{\max}\left(H_{t_k} \right) }
  { \lambda \lambda_{\min}\left(H_{t_k} \right)}  \\
  & = &
  \label{equ: condNum 4}
  \frac{\lambda_{\max}\left( \frac{X'X}{n} \right) + \frac{2 \lambda}{3t_k^3} \left[\log(1+t_k)\right]^2}
  { \frac{2 \lambda}{3 x ^3} \left[\log(1+t_k)\right]^2 } \\
  & = &
  \nonumber
  \frac{3x^3 \lambda_{\max}\left( \frac{X'X}{n} \right)}{2 \lambda \left[\log(1+t_k)\right]^2}
  +
  \frac{x^3}{t_k^3}\\
  & \leq &
  \label{equ: condNum 3}
  \frac{3 B^3 \lambda_{\max}\left( \frac{X'X}{n} \right)}{2 \lambda \left[\log( 1 + \tau )\right]^2}
  +
  \left( \frac{B}{\tau} \right)^3.
\end{eqnarray}
Equation \eqref{equ: condNum 1} is due to the definition of the condition number.
Inequality \eqref{equ: condNum 2} is because of the two fact.
First, for the maximal eigenvalue of summation of two matrix $A+B$, i.e., $\lambda_{\max}(A+B)$, is no more than  summation of maximal eigenvalue  separately, $\lambda_{\max}(A) + \lambda_{\max}(B)$.
Second, similar to the maximal eigenvalue, the minimal eigenvalue follows the similar rule that  $\lambda_{\min}(A+B) \geq \lambda_{\min}(A) + \lambda_{\min} (B)$.
The equality in \eqref{equ: condNum 4} is due to the fact that matrix $H_{t_k}$ is diagonal with positive diagonal entries.
The $x$ in \eqref{equ: condNum 4} refers to
$$
x = \max\left\{\left | \beta_i \right |: \beta_i \mbox{ is the } i\mbox{th entry in } \beta \right\}.
$$
Inequality \eqref{equ: condNum 3} is because that $t_k \geq \tau$ and we assume that throughout the algorithm, all elements in $\beta^{(k)} (k = 1,2 \ldots)$ is bounded by $B$.

By calling Theorem 3.7 in \cite{lan2019lectures} (i.e., the theorem in \ref{appendix: george lan theorem on AGD} in this paper), we can prove the statement in Theorem \ref{theo: number of inner-iterations}.
The details of the proof are listed as follows.
Recall that we want to minimize $F_{t_k}(\beta)$ for a fixed $k$.
From the previous analysis, we can find that $F_{t_k}(\beta)$ is $L_k$-smooth and $\mu_k$-strongly convex, where $L_k = \lambda_{\max}\left( \frac{X'X}{n} + \lambda H_{t_k} \right)$ and $\mu_k = \lambda_{\min}\left( \frac{X'X}{n} + \lambda H_{t_k} \right)$.
Consequently, the condition number in the $k$th outer-loop $\kappa_k = \frac{L_k}{\mu_k}$ can be upper bounded by
$
\frac{3 B^3 \lambda_{\max}\left( \frac{X'X}{n} \right)}{2 \lambda \left[\log( 1 + \tau )\right]^2}
  +
\left( \frac{B}{\tau} \right)^3
$
for any $\left\{ k = 0,1,2,\ldots: t_k \geq \tau \right\}$.

For a fixed $k$, when applying AGD to minimize $F_{t_k}(\beta)$, our steps, which are line \ref{algLine: AGD line1}-\ref{algLine: AGD line3} in Algorithm \ref{alg: HS}, follows the AGD steps that are presented in \eqref{equ: proof -- lan -- AGD procedure 1}-\eqref{equ: proof -- lan -- AGD procedure 3}, by setting
$
    \alpha_k = \sqrt{\frac{\mu_k}{L_k}},
    q_k = \frac{\alpha_k - \mu_k / L_k}{1 - \mu_k / L_k},
    \gamma_k = \frac{\alpha_k}{\mu_k (1 - \alpha_k)}.
$
According to \eqref{equ: appendix -- lan -- AGD -- convergence rate} in Theorem \ref{theo: appendix -- lan -- AGD -- convergence rate}, if
\begin{eqnarray}
\left(1 - \alpha_k \right)^s
    \underbrace{
    \left[
    F_{t_k}(\bar{\beta}^{(k)[0]}) - F_{k,\min} + \alpha_k \left( \mu_k + \frac{1}{\gamma_k} \right) V(\beta^{(k-1)[1]}, \widehat{\beta}_k)
    \right]
    }_{\mathcal C_k} && \nonumber \\
- \underbrace{
    \alpha_k \left( \mu_k + \frac{1}{\gamma_k} \right) V(\beta^{(k-1)[s-1]}, \widehat{\beta}_k)
    }_{\mathcal D_k} & \leq & \widetilde \epsilon_k,\label{equ: proof -- lan}
\end{eqnarray}
then $F_{t_k}(\beta^{(k)[s]}) - F_{k,\min} \leq \widetilde \epsilon_k$ with a given $\widetilde \epsilon_k$.
Here in \eqref{equ: proof -- lan}, we have $\widehat{\beta}_k = \arg\min_{\beta} F_{t_k}(\beta)$ and function $V(\cdot, \cdot)$ has been defined in  \ref{appendix: george lan theorem on AGD}.

We then solve the inequality in \eqref{equ: proof -- lan} to get an explicit formula for the quantity $s$.
To achieve this goal, we simplify \eqref{equ: proof -- lan} first.
Note that quantities ${\mathcal C_k}$ and ${\mathcal D_k}$ are defined via underlining in \eqref{equ: proof -- lan}.
It can be verified that
$
    \mathcal D_k \ge 0.
$
This is because $v(x) = \left\| x \right\|_2^2/2$ (recall the definition of $v(x)$ in \ref{appendix: george lan theorem on AGD}) is a convex function, i.e., we have
$$
    V(\beta^{(k-1)[s-1]}, \widehat{\beta}_k)
    =
    v(\widehat{\beta}_k) - \left[ v(\beta^{(k-1)[s-1]}) + (\widehat{\beta}_k - \beta^{(k-1)[s-1]})' \nabla v(\beta^{(k-1)[s-1]}) \right] \ge 0.
$$
Since
$
    \mathcal D_k
    > 0,
$
if we have
$$
  \left(1 - \alpha_k \right)^s \mathcal C_k
  \leq
  \widetilde \epsilon_k,
$$
then the inequality in \eqref{equ: proof -- lan} will be satisfied.
By introducing simple linear algebra, the above inequality can be rewritten as
$$
    \left(1 - \alpha_k \right)^s
    \leq
    \frac{\widetilde \epsilon_k}{\mathcal C_k}.
$$
By taking logarithm of both sides, we have
$$
    s \log \left(1 - \alpha_k \right)
    \leq
    \log\left( \frac{\widetilde \epsilon_k}{\mathcal C_k} \right),
$$
which gives
\begin{equation}
\label{equ: proof -- inner iteration -- i>=XX}
    s
    \geq
    \frac{ -\log\left( \frac{\widetilde \epsilon_k}{\mathcal C_k} \right) }{ -\log \left(1 - \alpha_k \right)}
    =
    \frac{ \log\left( \frac{\mathcal C_k}{\widetilde \epsilon_k} \right) }{ -\log \left(1 - \alpha_k \right)}.
\end{equation}
Furthermore, we know $\log \left( \frac{1}{1-x} \right) \geq x$ for $0<x<1$, so if
\begin{equation}
\label{equ: proof -- inner iteration -- i>=XXX}
    s
    \geq
    \frac{\log\left( \frac{\mathcal C_k}{\widetilde \epsilon_k} \right)}{\alpha_k},
\end{equation}
then the inequality in \eqref{equ: proof -- inner iteration -- i>=XX} holds.
In summary, if we have \eqref{equ: proof -- inner iteration -- i>=XXX}, then we have $F_{t_k}(\beta^{(k)[s]}) - F_{t_k}(\widehat\beta_k) < \widetilde\epsilon_k$.

Now we will show that, both $\frac{1}{\alpha_k} = \sqrt{\frac{L_k}{\mu_k}}$ and $\log(\mathcal C_k)$  in \eqref{equ: proof -- inner iteration -- i>=XXX} can be bounded by a constant that does not depend on $k$ (or equivalently, $t_k$).
First, we prove that $\frac{1}{\alpha_k} = \sqrt{\frac{L_k}{\mu_k}}$ can be bound.
This is essentially the argument that have been used in the step \eqref{equ: condNum 3}.
Second, we prove that $\mathcal C_k$ is also bounded.
Because we have
$$
    \mathcal C_k
    =
    \underbrace{
    F_{t_k}(\beta^{(k)[0]}) - F_{k,\min}
    }_{\mathcal C_{k,1}}
    +
    \underbrace{
    \alpha_k \left( \mu_k + \frac{1}{\gamma_k} \right)
    }_{\mathcal C_{k,2}}
    \underbrace{
    V(\beta^{(k-1)[1]}, \widehat{\beta}_k)
    }_{\mathcal C_{k,3}}.
$$
Note that quantities ${\mathcal C_{k,1}}$, ${\mathcal C_{k,2}}$, and ${\mathcal C_{k,3}}$ are defined via underlining in the above equation.
It is evident that $\mathcal C_{k,1}$ and $\mathcal C_{k,3}$ are bounded.
For $\mathcal C_{k,2}$, we have
$$
C_{k,2} = \mu_k,
$$
because we set $\gamma_k = \frac{\alpha_k}{\mu_k(1-\alpha_k)}$.
Since $\mu_k$ is bounded above by a constant, quantity $\mathcal C_{k,2}$ is bounded as well.
By combining the above several block, we know $\log(\mathcal C_k)$ is bounded.

In conclusion, after $C_1 \log(1/\widetilde\epsilon_k)$ inner-loops, one is guaranteed to achieve the following precision
$$
    F_{t_k}(\beta^{(k)})-F_{\min, k} \leq \widetilde\epsilon_k,
$$
where $\widetilde{\epsilon}_k = \frac{\lambda p}{3B} \left[ \log(1+t_k) \right]^2$ and $C_1$ is a constant that does not depend on the value of $t_k$ (or $k$).

\end{proof}

\subsection{Proof of a Theorem}
\label{proof: theo number of outer iterations}

The proof of Theorem \protect\ref{theo: number of outer iteration} is as follows.
\begin{proof}
 We start by showing that, for any $t\geq 0$, one has
	\begin{equation*}
	F(\beta^{(k)})-F(\widehat\beta)\leq \lambda p (2 B +1) t_k .
	\end{equation*}
	This is because of the following sequence of inequalities for any $\beta \in \mathbb R^p$:
	\begin{eqnarray}
		F(\beta^{(k)})
        \nonumber
		&=&
		\frac{1}{2n} \left\| y - X \beta^{ (k) } \right\|_2^2 +
		\lambda  \left\| \beta^{ (k) } \right\|_1   \\
        \nonumber
		&=&
        \label{proof outer-iter: replace |x| with fx}
		\frac{1}{2n} \left\| y - X \beta^{ (k) } \right\|_2^2 +
		\lambda\sum_{i=1}^{p} \left| \beta_i^{ (k) } \right| \\
		& \leq &
        \label{proof outer-iter: plug in fx}
		\frac{1}{2n} \left\| y - X \beta^{ (k) } \right\|_2^2 +
		\lambda \sum_{i=1}^{p} f_{t_k}( \beta_i^{ (k) } ) -
		\lambda p \left[ f_{t_k}(x)-x \right] |_{x=B}  \\
		& = &
        \nonumber
		\frac{1}{2n} \left\|y-X\beta^{(k)} \right\|_2^2 +
		\lambda\sum_{i=1}^{p}  f_{t_k} (\beta_i^{(k)}) + \\
		& &
		\lambda p B  \left[ 1- \left[ \frac{ \log(1 + t_k) }{t_k} \right]^2  \right] -
		\frac{\lambda p }{3 B} \left[ \log(1 + t_k) \right]^2 +
		\frac{ \lambda p }{t_k} \left[ \log(1 +t_k) \right]^2  \\
		& \leq &
        \nonumber
		\frac{1}{2n} \left\| y - X\beta^{ (k) } \right\|_2^2 +
		\lambda \sum_{i=1}^{p}  f_{t_k}(\beta_i^{(k)}) + \\
		& &
        \label{proof outer-iter: use log(1+x) is larger than x/(1+x)}
		\lambda p B \left[ 1-  \left( \frac{1}{1+t_k} \right)^2   \right]    -
		\frac{\lambda p }{3 B} \left[ \log(1 + t_k) \right]^2 +
		\lambda p  t_k \\
		& \leq &
        \label{proof outer-iter: 1-1/(1+t)^2 is less than 2t}
		\frac{1}{2n} \left\| y - X \beta^{ (k) } \right\|_2^2 +
		\lambda\sum_{i=1}^{p}  f_{t_k}(\beta_i^{(k)}) +
		2 \lambda p B t_k     -
		\frac{\lambda p }{3 B} \left[ \log(1 + t_k) \right]^2 +
		\lambda p  t_k \\
		& = &
        \nonumber
		\frac{1}{2n} \left\|y-X\beta^{(k)} \right\|_2^2 +
		\lambda\sum_{ i=1 }^{ p }  f_{ t_k }( \beta_i ^ { (k) }) +
		\lambda p ( 2 B + 1 ) t_k     -
		\frac{\lambda p }{3 B} \left[ \log(1 + t_k) \right]^2 \\
		& \leq &
        \label{proof outer-iter: beta k and beta k hat}
		\frac{1}{2n} \left \|y-X\widehat\beta^{(k)} \right\|_2^2 +
		\lambda\sum_{i=1}^{p}  f_{t_k}(\widehat \beta_i^{(k)}) +
		\lambda p  (2 B + 1) t_k     -
		\frac{\lambda p }{3 B} \left[ \log(1 + t_k) \right]^2  +
		\widetilde\epsilon_k  \\
		& = &
        \label{proof outer-iter: plug in widetilde epsilon}
		\frac{1}{2n} \left \| y - X\widehat\beta^{(k)} \right\|_2^2 +
		\lambda\sum_{i=1}^{p}  f_{t_k}(\widehat \beta_i^{(k)}) +
		\lambda p   (2 B + 1)  t_k    \\
        & \leq &
        \label{proof outer-iter: change hat beta k to hat beta}
		\frac{1}{2n} \left \| y - X\widehat\beta \right\|_2^2 +
		\lambda\sum_{i=1}^{p}  f_{t_k}(\widehat \beta_i) +
		\lambda p   (2 B + 1)  t_k    \\
        \label{proof outer-iter: ft is less than x}
        & \leq &
		\frac{1}{2n} \left\| y - X\widehat\beta \right\|_2^2 +
		\lambda \left\| \widehat \beta \right\|_1 +
		\lambda p   (2 B + 1)  t_k \\
        & = &
        \nonumber
        F(\widehat\beta) + \lambda p   (2 B + 1)  t_k
	\end{eqnarray}
    where inequality \eqref{proof outer-iter: replace |x| with fx} is due to the left side hand of inequality \eqref{equ: fx-x}, i.e., $\left[ f_{t_k}(x)- |x| \right] |_{x=B} \leq f_t(x)-|x| $.
    And equation \eqref{proof outer-iter: plug in fx} is by plugging in the value of $\left[ f_{t_k}(x)-x \right] |_{x=t_0}$.
    Inequality \eqref{proof outer-iter: use log(1+x) is larger than         x/(1+x)} utilizes the inequality that $\frac{t_k}{1+t_k} \leq \log(1+t_k)$ and inequality $\log(1+t_k)\leq t_k$.
    Inequality \eqref{proof outer-iter: 1-1/(1+t)^2 is less than 2t} uses inequality $1-\frac{1}{ (1+t_k)^2 } \leq 2 t_k$.
    Inequality \eqref{proof outer-iter: beta k and beta k hat} is because that  we assume the precision in $k$th inner-loop is $F_{t_k}( \beta^{ (k) } ) - F_{t_k}(\widehat \beta ^ {(k)} ) \leq \widetilde \epsilon_k $.
    Equation \eqref{proof outer-iter: plug in widetilde epsilon} is owing to the fact that we set $\widetilde \epsilon_k = \frac{\lambda p }{3 B} \left[ \log(1 + t_k) \right]^2 $.
    Inequality \eqref{proof outer-iter: ft is less than x} is due to the right hand side of inequality \eqref{equ: fx-x}, i.e., $ f_t(x)-|x| \leq 0$.
    Inequality \eqref{proof outer-iter: change hat beta k to hat beta} is because $\widehat{\beta}^{(k)}$ is the minimizer of $F_{t_k}(\beta)$, so $F_{t_k}(\widehat{\beta}^{(k)}) < F_{t_k}(\widehat{\beta} )$.
    Inequality \eqref{proof outer-iter: ft is less than x} is because $f_{t_k}(x) - |x| \leq 0$ in Lemma \ref{lemma:closeness-between-ft-and-x}.

    Through the above series of equalities and inequalities, we know that
    \begin{equation}\
    \label{equ: proof F(beta^k) - F(hat beta) < epsilon}
        F(\beta^{(k)}) - F(\widehat \beta)
        \leq
        \lambda p (2B+1) t_k.
    \end{equation}
    Besides, in the statement of the theorem, we have
    $$
        k
        \geq
        \frac{-1}{\log(1-h)} \log\left( \frac{\lambda p (2B+1) t_0}{\epsilon} \right),
    $$
    which is equivalent to
    $$
        \lambda p (2B+1) t_k
        \leq
        \epsilon.
    $$
    So the right side of inequality \eqref{equ: proof F(beta^k) - F(hat beta) < epsilon} isn't larger than $\epsilon$.
    Thus, we prove that, when
    $
        k
        \geq
        \frac{-1}{\log(1-h)} \log\left( \frac{\lambda p (2B+1) t_0}{\epsilon} \right),
    $
    we have $F(\beta^{(k)}) - F(\widehat \beta) \leq \epsilon$.

\end{proof}


\subsection{Proof of a Theorem}
\label{proof: HS order of complexity}

The proof of Theorem \ref{theo: HS order of complexity} is as follows.

\begin{proof}
The total number of numeric operations is determined by three factors, namely (1) the number of out-loops, (2) the number of inner-loops, and (3) the number of numeric operations in each inner-loops.
We adopt the assumption that different basic operations can be treated equally.
We have discussed (1) and (2) in Section \ref{sec: order of complexity of HS}, and we discuss (3) briefly here.
The main computational cost of an inner-loop in our proposed algorithm lies in Line \ref{algLine: AGD line2} of Algorithm \ref{alg: HS}, which is the matrix multiplication in $\frac{\partial}{\partial \beta^{(k)[s]} }F_{t_k}(\beta^{(k)[s]}) = \frac{X'X}{n}\beta^{(k)[s]} - \frac{X'y}{n} + \frac{\partial}{\partial \beta^{(k)[s]}}f_{t_k}(\beta^{(k)[s]})$.
With matrix $\frac{X'X}{n},\frac{X'y}{n}$ being pre-calculated and stored at the beginning of the execution, the calculation of $\frac{\partial}{\partial \beta^{(k)[i]} }F_{t_k}(\beta^{(k)[s]})$ requires $O(p^2)$ operations.

Now we count the total number of numerical operations that are need in our proposed method to achieve the $\epsilon$ precision.
We know that to achieve $F(\beta^{(k)}) - F_{\min} < \epsilon$, we need at least (Theorem \ref{theo: number of outer iteration})
$$
N \stackrel{\Delta}{=} \frac{-1}{\log(1-h)} \log\left( \frac{\lambda p (2B+1) t_0}{\epsilon} \right)
$$
outer-loops.
Furthermore, we know that the number inner-loop in an inner-loop $k$ is $O(\log(\frac{1}{\widetilde{\epsilon}_k}))$ with a hidden constant which can be universally bounded, and the number of operations in each inner-loop is $p^2$.
Therefore, the total number of numerical operations to get the estimator $\beta^{(k)}$ with precision $F(\beta^{(k)}) - F(\widehat{\beta}) \leq \epsilon$ can be upper bounded by the following quantity:
	\begin{eqnarray}
		p^2\sum_{k=1}^{ N }
        \log \left( \frac{1}{\widetilde{\epsilon}_k} \right)
		& = &
        \label{equ: proof of orderofcomplexity 1}
		p^2 \sum_{ k=1 }^{ N }
		\log\left(   \frac{3 B}{\lambda p} \left[ \left[ \log(1 + t_k) \right]^2  \right]^{-1} \right)  \\
		& = &
        \nonumber
		p^2 \sum_{ k=1 }^{ N } \log\left( \frac{3 B}{ \lambda p} \right) -
		p^2 \sum_{ k=1 }^{ N } \log\left( \left[ \log(1 + t_k) \right]^2   \right)  \\
		& = &
        \nonumber
		p^2 N \log \left(  \frac{3 B}{ \lambda p} \right) -
		2 p^2 \sum_{ k=1 }^{ N }  \log\left(  \log(1 + t_k)    \right)   \\
		& \leq &
        \label{equ: proof of orderofcomplexity 2}
		p^2 N \log\left(   \frac{3 B}{\lambda p} \right) -
		2 p^2 \sum_{ k=1 }^{ N }  \log\left(  \frac{t_k}{1 + t_k}    \right)   \\
		& = &
        \nonumber
		p^2 N \log\left(   \frac{3 B}{\lambda p} \right) -
		2 p^2 \sum_{ k=1 }^{ N }  \log \left( t_k \right) +
		2 p^2 \sum_{ k=1 }^{ N }  \log \left( 1 + t_k \right)   \\
        & = &
        \nonumber
		p^2 N \log\left(   \frac{3 B}{\lambda p} \right) -
		2 p^2 \sum_{ k=1 }^{ N }  \log \left( t_0 (1-h)^k \right) +
		2 p^2 \sum_{ k=1 }^{ N }  \log \left( 1 + t_k \right)   \\
        & \leq &
        \label{equ: proof of orderofcomplexity 3}
		p^2 N \log\left(   \frac{3 B}{\lambda p} \right) -
		2 p^2 \sum_{ k=1 }^{ N }  \log \left( t_0 (1-h)^k \right) +
		2 p^2 \sum_{ k=1 }^{ N }  t_k   \\
        & = &
        \nonumber
		p^2 N \log\left(   \frac{3 B}{\lambda p} \right) -
		2 p^2 \sum_{ k=1 }^{ N }  \left[ \log \left( t_0 \right) + k \log \left( 1-h \right) \right] +
		2 p^2 \sum_{ k=1 }^{ N }  t_0 (1-h)^k   \\
        & = &
        \nonumber
		p^2 N \log\left(   \frac{3 B}{\lambda p} \right) -
		2 p^2 N \log \left( t_0 \right) - 2 p^2 \log \left( 1-h \right) \sum_{ k=1 }^{ N } k +
		2 p^2 \sum_{ k=1 }^{ N } t_0 (1-h)^k    \\
        & = &
        \nonumber
		p^2 N \log\left(   \frac{3 B}{\lambda p} \right) -
		2 p^2 N \log \left( t_0 \right) - 2 p^2 \log \left( 1-h \right) \frac{(N+1)N}{2} \\
        &  &
        \nonumber
        + 2 p^2 \frac{t_0 \left[ 1-(1-h)^N \right] }{h}    \\
        & = &
        \nonumber
        O(N^2)
	\end{eqnarray}
    where equality \eqref{equ: proof of orderofcomplexity 1} is derived by plugging in that $\widetilde\epsilon_k = \frac{\lambda p }{3 B} \left[ \log(1 + t_k) \right]^2$.
    To be more exactly, there is a hidden constant related to the big O notation in $O\left( \log \left( \frac{1}{\widetilde \epsilon_k}\right)\right)$ in equality \eqref{equ: proof of orderofcomplexity 1}, however, as mention in the proof of Theorem \ref{theo: number of inner-iterations}, this hidden constant can be bounded universally.
    So in equality \eqref{equ: proof of orderofcomplexity 1}, we omit this hidden constant.
    Inequality \eqref{equ: proof of orderofcomplexity 2} is derived due to the inequality that $\log(1+x) \geq \frac{x}{1+x}$ for $x \geq 0 $.
    Inequality \eqref{equ: proof of orderofcomplexity 3} is derived due to the inequality that $\log(1+x) \leq x$ for $x \geq 0$.
\end{proof}

\subsection{Proof of a Proposition}
\label{proof: discussion -- prediction error}
The proof of Proposition  \ref{prop: discusssion -- prediction error} is as follows.

\begin{proof}
Because $\widehat\beta$ is the minimizer of $\frac{1}{2} \left\| y - X \beta \right\|_2^2 + \lambda \left\| \beta \right\|_1$, we can get its first-order condition as:
\begin{equation}
\label{equ: proof -- discussion -- prediction error -- first order of betahat}
  \frac{1}{n} \left( X'X \widehat\beta + X'y \right)
  +
  \lambda \text{sign}\left( \widehat\beta \right)
  = 0
\end{equation}
And because $\widetilde\beta$ is the minimizer of $\frac{1}{2} \left\| y - X \beta \right\|_2^2 + \lambda f_t\left( \beta \right)$, we can get its first-order condition as:
\begin{equation}
\label{equ: proof -- discussion -- prediction error -- first order of betatilde}
  \frac{1}{n} \left( X'X \widetilde\beta + X'y \right)
  +
  \lambda \nabla f_t\left( \widetilde\beta \right)
  = 0,
\end{equation}
where $\nabla f_t\left( \widetilde\beta \right)$ is the gradient of $f_t\left( \widetilde\beta \right)$.
By subtracting \eqref{equ: proof -- discussion -- prediction error -- first order of betahat} from \eqref{equ: proof -- discussion -- prediction error -- first order of betatilde}, we have
$$
  \frac{1}{n} X'X \left( \widetilde\beta - \widehat\beta \right)
  +
  \lambda \left[ \nabla f_t\left( \beta \right) - \text{sign} \left( \widehat\beta \right) \right]
  = 0.
$$
By left multiplying $\left( \widetilde\beta - \widehat\beta \right)'$ on both sides of the above equation, we have
$$
  \frac{1}{n}
  \left( \widetilde\beta - \widehat\beta \right)' X'X \left( \widetilde\beta - \widehat\beta \right)
  +
  \lambda
  \left( \widetilde\beta - \widehat\beta \right)' \left[ \nabla f_t\left( \widetilde\beta \right) - \text{sign} \left( \widehat\beta \right) \right]
  = 0.
$$
The above is equivalent to
\begin{eqnarray*}
  \frac{1}{n}
  \left( \widetilde\beta - \widehat\beta \right)' X'X \left( \widetilde\beta - \widehat\beta \right) &=&
  -\lambda
  \left( \widetilde\beta - \widehat\beta \right)' \left[ \nabla f_t\left( \widetilde\beta \right) - \text{sign} \left( \widehat\beta \right) \right] \\
  &=&
  - \lambda
  \left( \widetilde\beta - \widehat\beta \right)' \nabla f_t\left( \widetilde\beta \right)
  + \lambda
  \left( \widetilde\beta - \widehat\beta \right)' \text{sign} \left( \widehat\beta \right)\\
  &=&
  - \lambda
  \left( \widetilde\beta - \widehat\beta \right)' \nabla f_t\left( \widetilde\beta \right)
  + \lambda
  \widetilde\beta' \text{sign} \left( \widehat\beta \right)
  - \lambda
  \widehat\beta' \text{sign} \left( \widehat\beta \right)\\
  &=&
  - \lambda
  \left( \widetilde\beta - \widehat\beta \right)' \nabla f_t\left( \widetilde\beta \right)
  + \lambda
  \widetilde\beta' \text{sign} \left( \widehat\beta \right)
  - \lambda
  \left\| \widehat\beta \right\|_1.
\end{eqnarray*}
Because $f_t\left( \beta \right)$ is a convex function, we have
$$
  \frac{1}{n}
  \left( \widetilde\beta - \widehat\beta \right)' X'X \left( \widetilde\beta - \widehat\beta \right)
  \leq
  -\lambda \left[ f_t\left( \widehat\beta \right) - f_t\left(\widetilde\beta\right) \right]
  + \lambda
  \widetilde\beta' \text{sign} \left( \widehat\beta \right)
  - \lambda
  \left\| \widehat\beta \right\|_1.
$$
So we have
$$
  \frac{1}{n}
  \left\| X \left(\widetilde\beta - \widehat\beta \right) \right\|_2^2
  \leq
  -\lambda \left[ f_t\left( \widehat\beta \right) - f_t\left(\widetilde\beta\right) \right]
  + \lambda
  \left\| \widetilde\beta \right\|_1
  - \lambda
  \left\| \widehat\beta \right\|_1.
$$
When $t \rightarrow 0$, we have $f_t(\beta)$ very close to $\left\| \beta \right\|_1$, so we have
$
  \frac{1}{n}
  \left\| X \left(\widetilde\beta - \widehat\beta \right) \right\|_2^2
  \to 0.
$
\end{proof}

\subsection{Proof of a Proposition}
\label{proof: discussion -- estimation error}
The proof of Proposition  \ref{prop: discussion -- estimation error} is as follows.

\begin{proof}
From Proposition \ref{prop: discusssion -- prediction error}, we know that
$$
 \left\| X \left( \widetilde \beta - \widehat \beta \right) \right\|_2^2 \to 0
$$
when $t \to 0$, where
$
  \widehat \beta
  =
  \arg\min_\beta
  \frac{1}{2n} \left\| y - X\beta \right\|_2^2 + \lambda \left\| \beta \right\|_1,
$
and
$
  \widetilde \beta
  =
  \arg\min_\beta
  \frac{1}{2n} \left\| y - X\beta \right\|_2^2 + \lambda f_t(\beta).
$
The above can be written as
\begin{equation}
\label{equ: proof -- discussion -- estimation error decomposition}
  X_S \left( \widetilde \beta_S - \widehat \beta_S \right)
  +
  X_{S^c} \widetilde \beta_{S^c}
  =
  \delta,
\end{equation}
where $S$ is the support set of $\widehat \beta$ and
$
  \left\| \delta \right\|_2^2 \approx 0
$.
By left multiplying
$
  \left( X_S' X_S \right)^{-1} X_{S}'
$
on both sides of \eqref{equ: proof -- discussion -- estimation error decomposition}, we have
\begin{equation}
\label{equ: proof -- discussion -- estimation error decomposition -- support set}
  \left( \widetilde \beta_S - \widehat \beta_S \right)
  +
  \left( X_S' X_S \right)^{-1} X_{S}' X_{S^c} \widetilde \beta_{S^c}
  =
  \left( X_S' X_S \right)^{-1} X_{S}' \delta,
\end{equation}
By left multiplying
$
  X_{S^c}^\dagger
$
on both sides of \eqref{equ: proof -- discussion -- estimation error decomposition}, we have
\begin{equation}
\label{equ: proof -- discussion -- estimation error decomposition -- support set complement}
  X_{S^c}^\dagger X_S \left( \widetilde \beta_S - \widehat \beta_S \right)
  +
  X_{S^c}^\dagger X_{S^c} \widetilde \beta_{S^c}
  =
  X_{S^c}^\dagger \delta,
\end{equation}
where
$
  X_{S^c}^\dagger
$
is the \textit{pseudo-inverse} of matrix $X_{S^c}$.
The mathematical meaning of pseudo-inverse is that, suppose
$
  X_{S^c} = U \Sigma V
$, which is the singular value decomposition (SVD) of $X_{S^c}$.
Then
$
  X_{S^c}^\dagger = V' \Sigma^\dagger U'.
$
For the rectangular diagonal matrix $\Sigma$, we get $\Sigma^\dagger$ by taking the reciprocal of each non-zero elements on the diagonal, leaving the zeros in place, and then transposing the matrix.

By reorganizing \eqref{equ: proof -- discussion -- estimation error decomposition -- support set} and \eqref{equ: proof -- discussion -- estimation error decomposition -- support set complement} into block matrix, we have
\begin{equation*}
\underbrace{
  \left(
  \begin{array}{cc}
    I & \left(X_S' X_S \right)^{-1} X_S' X_{S^c} \\
    X_{S^c}^\dagger X_S & X_{S^c}^\dagger X_{S^c}
  \end{array}
  \right)
  }_{M}
  \left(
  \begin{array}{c}
    \widetilde \beta_S - \widehat\beta_S \\
    \widetilde \beta_{S^c}
  \end{array}
  \right)
  =
  \left(
  \begin{array}{c}
    \left( X_S' X_S \right)^{-1} X_S' \delta \\
    X_{S^c}^\dagger \delta
  \end{array}
  \right).
\end{equation*}
Through this system of equations, we can solve
$
  \left\|
  \left(
  \begin{array}{c}
    \widetilde \beta_S - \widehat\beta_S \\
    \widetilde \beta_{S^c}
  \end{array}
  \right)
  \right\|_2^2
$
as
$$
  \left\|
  \left(
  \begin{array}{c}
    \widetilde \beta_S - \widehat\beta_S \\
    \widetilde \beta_{S^c}
  \end{array}
  \right)
  \right\|_2^2
  =
  \left\| \widetilde \beta_S - \widehat\beta_S \right\|_2^2
  +
  \left\| \widetilde \beta_{S^c} \right\|_2^2
  =
  \left\|
  M^{-1}
  \left(
  \begin{array}{c}
    \left( X_S' X_S\right)^{-1} X_S' \delta \\
    X_{S^c}^\dagger \delta
  \end{array}
  \right)
  \right\|_2^2.
$$
Because for a matrix $A$ and vector $x$, we have
$
  \left\| A x \right\|_2^2
  \leq
  \left\| A \right\|_F^2 \left\| x \right\|_2^2,
$
we can bound
$
  \left\| \widetilde \beta_S - \widehat\beta_S \right\|_2^2
  +
  \left\| \widetilde \beta_{S^c} \right\|_2^2
$
as
\begin{eqnarray*}
\left\| \widetilde \beta_S - \widehat\beta_S \right\|_2^2
  +
  \left\| \widetilde \beta_{S^c} \right\|_2^2
  &\leq&
  \left\| M^{-1} \right\|_F^2
  \left\|
  \left(
  \begin{array}{c}
    \left( X_S' X_S\right)^{-1} X_S' \delta \\
    X_{S^c}^\dagger \delta
  \end{array}
  \right)
  \right\|_2^2 \\
&=&
  \left\| M^{-1} \right\|_F^2
  \left(
    \left\| \left( X_S' X_S\right)^{-1} X_S' \delta \right\|_2^2
    +
    \left\| X_{S^c}^\dagger \delta  \right\|_2^2
  \right).
\end{eqnarray*}
Because
$
  \left\| M^{-1} \right\|_F  \leq \sqrt{\text{rank}(M^{-1})} \left\| M^{-1} \right\|_2,
$
we can further bound
$
  \left\| \widetilde \beta_S - \widehat\beta_S \right\|_2^2
  +
  \left\| \widetilde \beta_{S^c} \right\|_2^2
$
as
\begin{eqnarray}
&&  \left\| \widetilde \beta_S - \widehat\beta_S \right\|_2^2
  +
  \left\| \widetilde \beta_{S^c} \right\|_2^2 \nonumber \\
  &\leq& \nonumber
  \text{rank}(M^{-1}) \left\| M^{-1} \right\|_2^2
  \left(
    \left\| \left( X_S' X_S\right)^{-1} X_S' \delta \right\|_2^2
    +
    \left\| X_{S^c}^\dagger \delta  \right\|_2^2
  \right)\\
  &=& \label{equ: proof -- discussion -- estimation error -- mineigen M}
  \text{rank}(M^{-1}) \left[ \frac{1}{\lambda_{\min}(M)} \right]^2
  \left(
    \left\| \left( X_S' X_S\right)^{-1} X_S' \delta \right\|_2^2
    +
    \left\| X_{S^c}^\dagger \delta  \right\|_2^2
  \right).
\end{eqnarray}

For $\left\| \left( X_S' X_S\right)^{-1} X_S' \delta \right\|_2^2$ in \eqref{equ: proof -- discussion -- estimation error -- mineigen M}, we have
\begin{eqnarray*}
  \left\| \underbrace{ \left( X_S' X_S \right)^{-1} X_S'}_{Q} \delta \right\|_2^2
  & = &
  \left\| Q \delta \right\|_2^2 \\
  & = &
  \sum_{i=1}^{|S|} (q_i' \delta)^2\\
  & \leq &
  \sum_{i=1}^{|S|} \left\| q_i \right\|_2^2 \left\| \delta \right\|_2^2 \\
  & = &
  \left\| Q \right\|_F^2 \left\| \delta \right\|_2^2,
\end{eqnarray*}
where $q_i'$ denotes the $i$th row in matrix $Q$, and $Q$ denotes $\left( X_S' X_S \right)^{-1} X_S'$.
Because $\left\| Q \right\|_F^2$ is bounded and $\left\| \delta \right\|_2^2 \to 0$, we have
$
  \left\| \left( X_S' X_S\right)^{-1} X_S' \delta \right\|_2^2 \to 0.
$

For $\left\| X_{S^c}^\dagger \delta  \right\|_2^2$ in \eqref{equ: proof -- discussion -- estimation error -- mineigen M}, following the similar logic, we have
\begin{eqnarray*}
  \left\| X_{S^c}^\dagger \delta \right\|_2^2
  & \leq &
  \left\| X_{S^c}^\dagger \right\|_F^2 \left\| \delta \right\|_2^2,
\end{eqnarray*}
Because $\left\| X_{S^c}^\dagger \right\|_F^2$ is bounded and $\left\| \delta \right\|_2^2 \to 0$, we have
$
  \left\| X_{S^c}^\dagger \delta \right\|_2^2.
$

For $\lambda_{\min}(M)$ in \eqref{equ: proof -- discussion -- estimation error -- mineigen M}, let's start with a general eigenvalue of matrix $M$, and we denote the eigenvalue of $M$ as $\lambda(M)$.
If we prove that all the eigenvalue of matrix $M$ is strictly larger than $0$, than $\frac{1}{\lambda_{\min}}(M)$ can be bounded.
This is equivalent to prove that $M - \lambda(M)I$ is positive semidefinite for any eigenvalue $\lambda(M)$.

If we denote $M^* = \frac{M + M'}{2}$, then we notice that
$
  \lambda(M) = \lambda(M^*).
$
We will verify that $M^* - \lambda(M)I$ is positive semidefinite under the conditions of Proposition \ref{prop: discussion -- estimation error}.
To verify it, we know that for any $\alpha, \beta$, we have
\begin{eqnarray}
&&  \left(
  \begin{array}{cc}
    \alpha' & \beta'
  \end{array}
  \right)
  M^*
  \left(
  \begin{array}{c}
    \alpha \\
    \beta
  \end{array}
  \right) \nonumber \\
&=& \nonumber
  \left(
  \begin{array}{cc}
    \alpha' & \beta'
  \end{array}
  \right)
  \left(
  \begin{array}{cc}
    (1-\lambda)I & \frac{A+B'}{2} \\
    \frac{A'+B}{2} & \frac{1}{2} X_{S^c}^\dagger X_{S^c} + \frac{1}{2} \left( X_{S^c}^\dagger X_{S^c} \right)' - \lambda I
  \end{array}
  \right)
  \left(
  \begin{array}{c}
    \alpha \\
    \beta
  \end{array}
  \right)\\
  &=& \label{equ: proof -- discussion -- mineigenvalue M -- psd -- interaction term}
  (1-\lambda)\left\| \alpha \right\|_2^2
  +
  \beta'
  \left[
  \frac{1}{2} X_{S^c}^\dagger X_{S^c} + \frac{1}{2} \left( X_{S^c}^\dagger X_{S^c} \right)' - \lambda I
  \right]
  \beta
  +
  \alpha' (A+B') \beta,
\end{eqnarray}
where $A = \left(X_S' X_S \right)^{-1} X_S' X_{S^c}$, $B = X_{S^c}^\dagger X_S$.
For the last term in \eqref{equ: proof -- discussion -- mineigenvalue M -- psd -- interaction term}, we can apply SVD to $A+B'$, i.e., $A+B' = U_1 \Sigma_1 V_1$, then we have
\begin{eqnarray*}
  |\alpha' (A+B') \beta|
  &=&
  \alpha' U_1 \Sigma_1 V_1 \beta \\
  &\leq&
  \sigma_{\max}(\Sigma_1) \langle \alpha' U_1, V_1 \beta \rangle \\
  &\leq&
  \sigma_{\max}(\Sigma_1) \left\| \alpha' U_1 \right\|_2 \left\| V_1 \beta \right\|_2  \\
  &\leq&
  \sigma_{\max}(\Sigma_1) \left\| \alpha' \right\|_2 \left\| \beta \right\|_2\\
  &\leq&
  \frac{1}{2} \sigma_{\max}(\Sigma_1) \left( \left\| \alpha' \right\|_2^2  + \left\| \beta \right\|_2^2 \right),
\end{eqnarray*}
where $\sigma_{\max}(\Sigma_1)$ is the maximal absolute value in the diagonal entry of $\Sigma_1$.

By plugging the above result into \eqref{equ: proof -- discussion -- mineigenvalue M -- psd -- interaction term}, we have
\begin{eqnarray}
  \left(
  \begin{array}{cc}
    \alpha' & \beta'
  \end{array}
  \right)
  M^*
  \left(
  \begin{array}{c}
    \alpha \\
    \beta
  \end{array}
  \right)
  &\geq& \nonumber
  (1-\lambda)\left\| \alpha \right\|_2^2
  +
  \beta'
  \left[
  \frac{1}{2} X_{S^c}^\dagger X_{S^c} + \frac{1}{2} \left( X_{S^c}^\dagger X_{S^c} \right)' - \lambda I
  \right]
  \beta \\
  && \nonumber
  -
  |\alpha' (A+B') \beta|\\
  &\geq& \nonumber
  (1-\lambda)\left\| \alpha \right\|_2^2
  +
  \beta'
  \left[
  \frac{1}{2} X_{S^c}^\dagger X_{S^c} + \frac{1}{2} \left( X_{S^c}^\dagger X_{S^c} \right)' - \lambda I
  \right]
  \beta \\
  && \nonumber
  -
  \frac{1}{2} \sigma_{\max}(\Sigma_1) \left( \left\| \alpha' \right\|_2^2  + \left\| \beta \right\|_2^2 \right) \\
  & = &\label{equ: proof -- discussion -- mineigenvalue M -- psd -- quad bound}
  \left(1-\lambda - \frac{1}{2} \sigma_{\max}(\Sigma_1) \right) \left\| \alpha \right\|_2^2
  +\\
  && \nonumber
  \beta'
  \left[
  \frac{1}{2} X_{S^c}^\dagger X_{S^c} + \frac{1}{2} \left( X_{S^c}^\dagger X_{S^c} \right)'
  -
  \left( \lambda + \frac{1}{2} \sigma_{\max}(\Sigma_1) \right) I
  \right]
  \beta,
\end{eqnarray}
where $\sigma_{\max} \left( \Sigma_1 \right)$ is the maximal absolute diagonal value of matrix $\Sigma_1$.
Because we have $\sigma\left(\Sigma_1 \right) < 2$, so the first term in \eqref{equ: proof -- discussion -- mineigenvalue M -- psd -- quad bound} is greater than $0$.
Besides, because the minimal singular value of
$
  \frac{1}{2} X_{S^c}^\dagger X_{S^c} + \frac{1}{2} \left( X_{S^c}^\dagger X_{S^c} \right)'
$
is larger than
$
  \frac{1}{2} \sigma_{\max}(\Sigma_1)
$,
i.e.,
$
  \frac{1}{2} X_{S^c}^\dagger X_{S^c} + \frac{1}{2} \left( X_{S^c}^\dagger X_{S^c} \right)'
  =
  U_2 \Sigma_2 V_2
$
and
$
  2 \sigma_{\min}\left(\Sigma_2 \right) >  \sigma_{\max} \left(\Sigma_1 \right)
$,
the second term in \eqref{equ: proof -- discussion -- mineigenvalue M -- psd -- quad bound} is also greater than $0$.
Thus, we prove that $M^*$ is a positive semidefinite matrix, whose eigenvalue would be strictly larger than $0$.
According, $M$, which shares the same eigenvalue with $M^*$ also has eigenvalues strictly larger than $0$.
So we have $\frac{1}{\lambda_{\min}(M)}$ bounded.

In conclusion, because $\lambda_{\min} \left( M \right)$ is bounded, $\left\| \left( X_S' X_S\right)^{-1} X_S' \delta \right\|_2^2 \to 0$, and $\left\| X_{S^c}^\dagger \delta  \right\|_2^2 \to 0$, we have
$$
  \left\| \widetilde \beta - \widehat \beta \right\|_2^2
  =
  \left\| \widetilde \beta_S - \widehat\beta_S \right\|_2^2
  +
  \left\| \widetilde \beta_{S^c} \right\|_2^2
  \to 0.
$$

\end{proof}







\end{document}